\documentclass[11pt,twoside]{article}

\usepackage{fullpage}

\input{command-template}

\begin{document}

	\begin{center}
		
		{\bf{\LARGE{Stochastic first-order methods for average-reward Markov decision processes
		}}}

		\vspace*{.2in}
		
		{\large{
				\begin{tabular}{ccc}
					Tianjiao Li$^\star$ & Feiyang Wu$^\dagger$ & Guanghui Lan$^\star$
				\end{tabular}
		}}
		\vspace*{.2in}
		
		\begin{tabular}{c}
			$^\star$School of Industrial and Systems Engineering \\
			$^\dagger$School of Computer Science \\
			Georgia Institute of Technology
		\end{tabular}
		
		\vspace*{.2in}

		
		\vspace*{.2in}
		
		\begin{abstract}
		We study average-reward Markov decision processes (AMDPs) and develop novel first-order methods with strong theoretical guarantees for both policy optimization and policy evaluation. Compared with intensive research efforts in finite sample analysis of policy gradient methods for discounted MDPs, existing studies on policy gradient methods for AMDPs mostly focus on regret bounds under restrictive assumptions, and they often lack guarantees on the overall sample complexities. Towards this end, we develop an average-reward stochastic policy mirror descent (SPMD) method for solving AMDPs with and without regularizers and provide convergence guarantees in terms of the long-term average reward. 
For policy evaluation, existing on-policy methods suffer from sub-optimal convergence rates as well as failure in handling insufficiently random policies due to the lack of exploration in the action space. To remedy these issues, we develop a variance-reduced temporal difference (VRTD) method with linear function approximation for randomized policies along with optimal convergence guarantees, and design an exploratory VRTD method that resolves the exploration issue and provides comparable convergence guarantees. 
By combining the policy evaluation and policy optimization parts, we establish sample complexity results for solving AMDPs under both generative and Markovian noise models.
\revision{}{It is worth noting that when linear function approximation is utilized, our algorithm only needs to update in the low-dimensional parameter space and thus can handle MDPs with large state and action spaces.}
		\end{abstract}
	\end{center}
	
\section{Introduction} \label{sec:intro}
Reinforcement learning (RL) problems are generally formulated as Markov decision processes (MDPs). At each time step, the agent observes the current state and subsequently takes an action, which leads to an instantaneous reward or cost as well as the transition to the next state according to an unknown transition kernel. The eventual goal of the agent is to learn a policy to optimize the reward accrued (or cost paid) over time. MDPs consist of two prominent classes, including the discounted MDPs (DMDPs), which introduce a discount factor to the measure of accumulated reward/cost, and the average-reward MDPs (AMDPs), which measure the success of the system in terms of its steady-state performance. Although DMDP has been studied extensively, it usually leads to poor long-term performance when a system operates for an extended time period. The average-reward criterion used in AMDPs can alleviate this issue, hence has been widely used in many different applications ranging from engineering to natural and social sciences; see, e.g., \citet{kober2013reinforcement,xu2014reinforcement} for surveys of applications.

Since the transition kernel is unknown, one important objective in RL studies is to understand its sample complexity,
which characterizes how many samples are required to obtain a nearly-optimal policy up to a given accuracy.
Finite sample analyses are often
carried out under different observation models.
The two most well-known observation models considered by the literature are the so-called \emph{generative model} (``i.i.d.'' model) and the 
\emph{Markovian model} (``single-trajectory" model).
\revision{}{
In the {\em generative} model, we have access to an oracle that returns the next state when queried at any state-action pair, typically seen in agent training with simulators.
On the other hand, the {\em Markovian} setup does not permit system restarts at arbitrary state-action pairs; once initiated, the trajectory must be followed sequentially. This latter situation often occurs following the deployment of an agent's policy. 
}
To facilitate the analysis, these observation models are also augmented with various assumptions on the behavior of the underlying MDPs (e.g., mixing and ergodicity). 
In comparison with significant progress recently made on
the finite sample analysis of DMDPs,
theoretical studies on the finite sample analysis of AMDPs, especially under the Markovian noise model, are rather limited. 

Under the Markovian noise model, we distinguish two types of RL methods, namely on-policy learning \citep{konda1999actor} and off-policy learning \citep{degris2012off}. \revision{}{In on-policy learning, the samples can be collected based on the currently generated policy, while for off-policy learning the samples are collected based on a behavior policy. Both settings have their areas of applications in practice. The off-policy learning encourages the agent to explore the action space by using a sufficiently random, but possibly highly non-optimal policy \citep{wan2021learning,zhang2021average}, while on-policy learning enables continuous policy improvement in-situ with instantaneous rewards generated by nearly-optimal policies.  
One significant challenge associated with on-policy learning exists
in the exploration for the action space}. \revision{}{In particular, as the policy converges towards the optimal policy (either a deterministic or randomized optimal policy), non-optimal actions are rarely explored given their diminishing probabilities in the policy at hand, thus making the policy evaluation task increasingly difficult.} 
\revision{}{Consequently, existing on-policy learning methods require restrictive assumptions that each policy generated by the algorithm is random enough, meaning that the policy will take each action with a probability bounded away from zero.}  \revision{}{This assumption is unrealistic since in general the optimal policy does not possess this structure, so the goal of policy optimization will hardly be achieved. Up to our knowledge,} in spite of intensive research effort of on-policy learning \citep{tsitsiklis1999average,yu2009convergence,zhang2021finite,mou2022optimal}, one seemly unresolved problem in RL is whether one can design sampling-efficient on-policy learning algorithms for insufficiently random policies \revision{and use them for policy optimization}{to remove the aforementioned problematic assumption}  (see Remark 1 of \citet{lan2021policy}). 

It is well-known that one can formulate AMDPs 
as a linear programming (LP) problem (see, e.g., \citet{puterman2014markov}).
As a result, one can utilize the rich theory in LP and specialize some advanced algorithms for solving AMDPs. In particular, to deal with the unknown transition kernels, current research has been focused on utilizing general stochastic primal-dual saddle point optimization methods (see Chapter 4 of \citet{lan2020first}), and some important progresses along this research direction have been made in \citet{wang2017primal, jin2020efficiently}.
In spite of the elegance of the stochastic LP approach, there exist a few significant limitations associated with this approach.
Firstly, due to the high dimensionality of the state space, the size of LP can be huge. One common practice is to use function approximation with linear kernel or neural networks for the value functions. However, the LP approach cannot deal with function approximations. 
Secondly, in MDP and RL, one usually needs to incorporate certain nonlinear objectives and/or constraints. One prominent example is a regularization given by Kullback–Leibler (KL) divergence, which can be used to model the difference between the desirable policy and a reference policy. The LP approach will fail to handle
these nonlinear components.
Thirdly, in the current stochastic LP approaches for RL, one needs to
have access to a generative model to simulate the probability of moving to the next state for any state-action pair.
As a consequence, this type of approach cannot be applied to the Markovian noise setting.


Some classic dynamic programming methods, especially value iteration, have been adapted for AMDPs with unknown transition kernels
\citep{abounadi2001learning,gosavi2004reinforcement,wan2021learning,zhang2021finite}. These methods are often called value-based methods or Q-learning in the RL literature. Even though value-based methods are extensively studied in the DMDP literature and optimal finite sample guarantees are established \citep{wainwright2019variance, li2020sample}, the only known finite sample analysis of Q-learning for AMDPs up to our knowledge is \citet{zhang2021finite}. However, \citet{zhang2021finite} only focused on the generative model, and the overall sample complexity is significantly worse than LP methods \citep{wang2017primal, jin2020efficiently}. \revision{}{Moreover, the value-based approach can fail in handling linear function approximation \citep{boyan1994generalization, baird1995residual, tsitsiklis1996feature} or policy regularization. In addition, Q-learning is also an off-policy method and cannot be applied to in-situ improvements.}
It is also noteworthy that a recent work of  \citet{jin2021towards} established a reduction method that solves an AMDP by solving an associated DMDP. Using this reduction technique, while the dependence on the problem parameters, e.g., mixing time, is improved, the dependence on the accuracy measure is sub-optimal, i.e., $\order(1/\epsilon^3)$. \revision{}{After we released the initial version of this work, a few more recent works \citep{wang2022near, zhang2023sharper, wang2023optimal, zurek2023span} further improved the reduction argument, and showed that the optimal sample complexities can be achieved in \citet{wang2023optimal, zurek2023span}. However, these optimal results require solving the associated DMDP with model-based approaches, thus failing to handle the policy regularization and function approximation.}

In this paper, we focus on the study of a new class of
first-order methods for solving AMDPs that involves
both \emph{policy evaluation} and \emph{policy improvement} steps.
The framework is also known as the \emph{actor-critic} method in the RL literature, which generalizes the classic policy iteration method. Specifically, the actor aims to perform the policy update while the critic is dedicated to evaluating the value associated with the currently generated policy, which constitutes first-order information of the nonlinear policy optimization problem. 
In comparison with the LP method and value-based method, these policy-based first-order methods offer remarkable flexibility for function approximation as well as policy regularization. Meanwhile, the model-free framework of policy-based methods naturally accommodates both generative and Markovian noise settings.  
Recently, there has been considerable interest in the development and analysis of first-order methods for DMDPs. In particular, sublinear convergence guarantees of policy gradient method and its variants are discussed in \citet{agarwal2021theory} and \citet{shani2020adaptive}. 
\citet{cen2021fast} showed that
natural policy gradient (NPG) converges linearly for entropy-regularized DMDPs.
\citet{lan2021policy} proposed the policy mirror descent (PMD) method which achieved linear convergence for DMDPs, and $\widetilde{\cal O}(1/\epsilon)$ (resp., $\widetilde{\cal O}(1/\epsilon^2)$) sample complexity for regularized (resp., unregularized) DMDPs. Further studies for these types of methods can be found in \citet{khodadadian2021linear,zhan2021policy,lan2022block,xiao2022convergence,li2022homotopic}.  
Nevertheless, theoretical understanding of policy-based methods for AMDPs remains very limited. It is noteworthy that \citet{abbasi2019politex, lazic2021improved} proposed a policy-based method named POLITEX for AMDPs, but their analysis requires restrictive assumptions on the transition kernel and they only established regret bounds without sample complexities (see early related work \citet{bartlett2012regal,auer2008near,ouyang2017learning}).



\subsection{Main contributions}
In this work, we make three distinct contributions.
\begin{itemize}
\item\textbf{Policy mirror descent for AMDPs (Actor)}: Motivated by the stochastic policy mirror descent (SPMD) algorithm for solving DMDPs \citep{lan2021policy}, we develop the average-reward SPMD method. \revision{In the convergence analysis, we handle the expected error and bias error from the critic step separately, leading to sharpened convergence results.}{We consider the linear function approximation in the SPMD updates and reveal an update rule in the low-dimensional space when equipping SPMD with KL-divergence.
Specifically, our algorithm only needs to update in the low-dimensional parameter space and thus can handle MDPs with large state and action spaces.
} 
\revision{For unregularized AMDPs, the SPMD method achieves  $\order(1/\sqrt{T})$ convergence rate. }{In the convergence analysis, we handle the expected error and bias error from the critic step separately, leading to sharpened convergence results. The SPMD method achieves  $\order(1/\sqrt{T})$ convergence rate for unregularized AMDPs, and an improved $\order(1/\omega T)$ convergence rate (where $\omega$ stands for the strongly convex modulus of the regularizer) for strongly convex regularized AMDPs.}
	\item\textbf{Policy evaluation for AMDPs (Critic)}: We first propose a simple multiple trajectory method for policy evaluation in the generative model, 
 \revision{
 which achieves $\order(\tmix \log(1/\epsilon))$ sample complexity for $\ell_\infty$-bound on the bias of the estimators, as well as $\order(\tmix^3/\epsilon)$ sample complexity for the expected squared $\ell_\infty$-error of the estimators.}{along with sample complexities for both bias and expected squared error of the estimators.}  For the on-policy evaluation under Markovian noise, we develop an average-reward variant of the variance-reduced temporal difference (VRTD) algorithm \citep{khamaru2020temporal,li2021accelerated} with linear function approximation, which achieves $\order(\tmix^3 \log(1/\epsilon))$ sample complexity for the bias of the estimators, as well as an instance-dependent sample complexity for expected error of the estimators. The latter complexity improved the one in \citet{zhang2021finite} by at least a factor of $\order(\tmix^2)$. Meanwhile, we develop a novel exploratory variance-reduced temporal difference (EVRTD) that \revision{remedies}{resolves} the exploration issue in the action space when the policy is not random enough. The idea of the EVRTD algorithm naturally extends to DMDPs. 
 \revision{We show that the sample complexity of EVRTD is not much worse than VRTD for both AMDP and DMDP.}{}
	\item\textbf{Sample complexity}:
	We establish the overall sample complexity under both the generative model and Markovian noise model. Under the mixing unichain setting with a generative model, by implementing the SPMD algorithm with multiple trajectory method as the critic, we achieve $\widetilde{\order}\big(\frac{\tmix^3 |\calS||\calA|}{\epsilon^2}\big)$ overall sample complexity.  As for the ergodic setting with the Markovian observation model, we implement SPMD with VRTD/EVRTD as the critic, which achieves
	$\widetilde{\order}(\epsilon^{-2})$ sample complexity. For regularized AMDPs, this bound is improved to $\widetilde{\order}(\tfrac{1}{\omega^2 \epsilon})$ for obtaining an $\epsilon$-optimal policy in terms of distance to the optimal solution.
	To the best of our knowledge, these sample complexities are new in the literature on solving average-reward RL problems with policy-based (first-order) methods.
\end{itemize}

\subsection{Notation} 
For a positive integer n, we define $[n] := \{1, 2, . . . , n\}$. We let $\mathbf{1}_n$ denote the all-ones vector with dimension $n$, and we will write $\ones$ when the dimension is clear in the context. We let $e_j$ denote the $j$-th standard basis vector in $\bbr^D$. Let $\mathbb{I}_S:X\rightarrow \{0,1\}$ denote the indicator function of the subset $S \subseteq X$. 
Given a vector $x\in \bbr^m$, denote its 
$i$-th entry by $x_{(i)}$. In situations in which there is no ambiguity, we also use $x_i$ to denote the $i$-th coordinate of a vector $x$. Let $\|x\|_1:=\tsum_{i=1}^m |x_{(i)}|$, $\|x\|_2:=\sqrt{\tsum_{t=1}^m x_{(i)}^2}$ and $\|x\|_\infty:=\max_{i\in [m]}|x_{(i)}|$ denote the $\ell_1$, $\ell_2$ and $\ell_\infty$-norms respectively. For a random variable $X\in\bbr$, let $\|X\|_{\psi_1}:=\inf \{t>0: \bbe [\exp(|X|/t)]\leq 2\}$ and $\|X\|_{\psi_2}:=\inf \{t>0: \bbe [\exp(X^2/t^2)]\leq 2\}$ denote the sub-exponential and sub-gaussian norm respectively. Given a matrix $A$, denote its $(i,j)$-th entry by $A_{i,j}$. Let $\|A\|_2$ denote the spectral norm of matrix $A$. Let $\|A\|_\infty:=\max_{\|x\|_\infty =1}\|Ax\|_\infty$ denote the $\ell_\infty$-operator norm of $A$. We let $\lambda_{\min} (A)$ denote the smallest eigenvalue of a square matrix $A$. 
For a symmetric positive (semi)definite matrix $A$, define  $\langle x, y \rangle_A := x^\top A y$ and the associated (semi)norm $\|x\|_A := \sqrt{x^\top A x}$. We refer to $\|x\|_A$ as the $\ell_A$-norm of $x$. \revision{}{We use $\spannorm{x}:= \max_{i} x_i - \min_{i} x_i$ to define the span semi-norm of a vector $x$.} We let $\Delta_{n}$ denote a simplex with dimension $n$. 

\section{Background and problem setting}
In this section, we introduce some preliminaries for (regularized) AMDPs and the concrete observation models that we study. 
\subsection{Average-reward Markov decision processes}
An AMDP is described by a tuple $\mathcal{M}:=(\calS,\calA,\mathsf{P},c)$, where $\mathcal{S}$ denotes a finite state space, $\calA$ denotes a finite action space, $\mathsf{P}$  is the transition kernel, and $c$ is the cost function. At each time step, the agent takes an action $a\in\calA$ at the current state $s\in\calS$, then the system moves to some state $s'\in \mathcal{S}$ with probability $\mathsf{P}(s' |s,a)$, while the agent pays the instantaneous cost $c(s, a)$ (or receives 
the $-c(s,a)$ instantaneous reward). The goal of the agent is to determine a policy which minimizes the long-term cost. A randomized stationary policy of an MDP is a mapping $\pi: \calS\rightarrow \Delta_{|\calA|}$ that maps a state $s\in\calS$ to a fixed distribution over actions. We denote $P_\pi \in \bbr^{|\calS|\times|\calS|}$ and $P^\pi \in \bbr^{(|\calS|\times|\calA|)\times(|\calS|\times|\calA|)}$ as the state transition matrix and state-action transition matrix induced by policy $\pi$, respectively, where $P_\pi(s,s') = \sum_{a\in \calA} \pi(a|s)\mathsf{P}(s'|s,a)$ and $P^\pi((s,a),(s',a')) = \mathsf{P}(s'|s,a)\pi(a'|s')$. Under a given policy $\pi$, the (regularized) long-run average cost/reward for state $s\in \calS$ is defined as  
\begin{align}\label{def_avg_rwd}
	\rho^\pi(s) : =\lim _{T \rightarrow \infty} \tfrac{1}{T}\bbe_\pi\big[\tsum_{t=0}^{T-1} (c(s_t,a_t)+h^\pi(s_t))\big|s_0=s\big].
\end{align}
Here, $h^\pi$ is a closed convex function with respect to the policy $\pi$, i.e., there exists $\omega\geq 0$ such that 
\begin{align}\label{convex_regularizer}
	h^\pi(s) - [h^{\pi'}(s) + \langle \nabla h^{\pi'} (s,\cdot), \pi(\cdot|s)-\pi'(\cdot|s)\rangle]\geq \omega D_{\pi'}^\pi(s),
\end{align}
where $\nabla h^{\pi'}$ denotes the subgradient of $h$ at $\pi'$ and $D_{\pi'}^\pi(s)$ is the Bregman distance between $\pi$ and $\pi'$. In this paper, we fix $D_{\pi'}^\pi(s)$ to be the Kullback-Leibler (KL) divergence, i.e., $D_{\pi'}^\pi(s)=\sum_{a\in\calA}\pi(a|s) \log\frac{\pi(a|s)}{\pi'(a|s)}$.
It should be noted that our algorithmic framework allows us to use other distance generating functions, e.g., $\|\cdot\|_p^2$ for some $p>1$. Clearly, if $h^\pi=0$, then Eq.~\eqref{def_avg_rwd} reduces to the classical unregularized average cost function, while if $h^\pi(s)= \omega D_{\pi_0}^\pi(s)$ for $\omega>0$ then Eq.~\eqref{def_avg_rwd} defines the average cost of the entropy-regularized MDP.

In this work, we consider the \emph{unichain} setting, where the induced Markov chain consists of a single recurrent class plus a possibly empty set of transient states for any deterministic stationary policy. We further restrict our attention to \emph{mixing} AMDPs, which satisfy the following assumption. 

\begin{assumption}\label{assump_mixing}
	An AMDP instance is mixing if for any feasible policy $\pi$, there exists a stationary distribution $\nu^\pi$ such that for any distribution $q \in \Delta_{|\mathcal{S}|}$, the induced Markov chain has mixing time bounded by $\tmix < \infty$, where $\tmix$ is defined as 
	\begin{align*}
		\tmix := \max_{\pi}\big[\arg \min_{t\geq 1} \big\{\max_{q\in \Delta_{|\mathcal{S}|}} ||((P_\pi)^t)^\top q - \nu^\pi||_1\leq \tfrac{1}{2}\big\}\big].
	\end{align*}
\end{assumption}
\noindent \textcolor{black}{This is a widely used regularity condition for AMDPs (see, e.g., \citet{wang2017primal,jin2020efficiently}), which can be ensured by adding the \emph{aperiodic} assumption to the unichain setting.} As a result, for any feasible policy $\pi$, the average reward function does not depend on the initial state  (see Section 8 of \cite{puterman2014markov}), and we can define a scalar which satisfies for all $s\in \calS$
\begin{align*}
	\rho^\pi(s) = \avgrwd := \langle \nu^\pi, c^\pi + h^\pi \rangle, 
\end{align*}
where $\nu^\pi$ is the stationary distribution of the states induced by the policy $\pi$, and $c^\pi$ is the expected cost function induced by policy $\pi$, i.e., $c^\pi(s) = \sum_{a\in \calA} c(s,a) \pi(a|s)$. Given that one can view $\rho^\pi$ as a function of $\pi$, we will use the notation $\rho^\pi$ and $\rho(\pi)$ interchangeably. Since the average-reward function only captures the ``steady-state'' behavior of the underlying policy, the literature utilizes the following differential functions to capture the ``transient'' behavior of the underlying policy. 
Specifically, the \emph{basic differential value function} (also called bias function in \cite{puterman2014markov})  is defined as
\begin{align*}
	\bar {V}^{\pi}(s):=\bbe_\pi\left[\tsum_{t=0}^{\infty}\left(c\left(s_{t}, a_{t}\right)+h^\pi(s_t)-\avgrwd\right) \big| s_{0}=s\right],
\end{align*}
and the \emph{basic differential action-value function (or basic differential Q-function)}  is defined as
\begin{align*}
	\bar {Q}^{\pi}(s, a):=\bbe_\pi\left[\tsum_{t=0}^{\infty}\left(c\left(s_{t}, a_{t}\right)+h^\pi(s_t)-\avgrwd\right) \big| s_{0}=s, a_{0}=a\right].
\end{align*}
\textcolor{black}{Moreover, we define the sets of \emph{differential value functions} and \emph{differential action-value functions (differential Q-functions)} as the solution sets of the following Bellman equations, respectively,
	\begin{align}
		V &= c^\pi + h^\pi - \rho^\pi \ones_{|\calS|} + P_\pi V\label{bellman_0_0}\\
		Q &= c + \widetilde h^\pi - \rho^\pi \ones_{|\calS|\times|\calA|} + P^\pi 	Q, \label{bellman_0}
	\end{align}
	where $\widetilde h^\pi (s,a) = h^\pi (s)$. Under Assumption~\ref{assump_mixing}, the solution of Eq.~\eqref{bellman_0_0} (resp., Eq. \eqref{bellman_0}) is unique up to an additive constant, which takes the form of $\{\bar V^\pi+b\ones_{|\calS|}|b\in\bbr\}$ (resp., $\{\bar Q^\pi+b\ones_{|\calS|\times |\calA|}|b\in\bbr\}$).}

Finally, our goal in AMDP is to find an optimal policy $\pi^*$ that minimizes the long-run average cost. Consequently,
the concerned \emph{policy optimization} problem can be formulated as:
\begin{equation}\label{def_problem}
	\begin{aligned}
		\min _{\pi} \quad \rho(\pi)\quad \text{s.t.} \quad \pi(\cdot|s) \in \Delta_{|\mathcal{A}|},~\forall s\in \mathcal{S}.
	\end{aligned}
\end{equation}

\subsection{Observation models}

It is typically assumed in RL that there is sample access to the stochastic transition kernel and cost function. As mentioned in Section~\ref{sec:intro}, this study differentiates between two fundamental observation models in RL: the generative model (also known as the ``i.i.d." model) and the Markovian noise model (also known as the ``single trajectory" model).

For the generative model, we utilize an oracle that, when queried with a state-action pair $(s,a)$, provides the subsequent state $s'$ and the immediate cost $c(s,a)$. If the cost function is noiseless, this scenario implicitly implies a complete understanding of the cost function.

In the Markovian noise model, we assume that all observed samples stem from a single trajectory of a Markov chain induced by the policy $\pi$ under evaluation. Specifically, we are presented with a sequence of state-action pairs $\{(s_0,a_0), (s_1,a_1),...\}$ that are generated according to the policy $\pi$ and the transition kernel $\mathsf{P}$. 
Since solving problem \eqref{def_problem} requires sample access to all the states in $\calS$, we introduce the following natural assumption for the Markovian noise model. Clearly, this assumption is stronger than Assumption~\ref{assump_mixing} since it indicates that the stationary distribution of any feasible policy is strictly positive.
\begin{assumption}\label{assump_rho_0}
	For any feasible policy $\pi$, the Markov chian induced by policy $\pi$ is irreducible and aperiodic.
\end{assumption}
\noindent 

Given the stochastic nature of the RL setting, a key problem for us is \emph{policy evaluation}, which involves estimating the long-run average reward/cost and the differential value function or differential action-value function for a specific policy under various observation models. By utilizing policy evaluation techniques, we are able to develop effective first-order algorithms to address the policy optimization problem~\eqref{def_problem}.

The rest of this paper is organized as follows. 
In Section~\ref{sec:policy_optimization}, we develop the stochastic policy mirror descent (SPMD) method for solving AMDPs.
In Section~\ref{sec:policy_evaluation}, we propose efficient methods to solve the policy evaluation problem under both types of observation models and establish the overall sample complexity of the SPMD method. In Section~\ref{sec:numerical}, we provide numerical experiments to corroborate our theoretical guarantees. 

\section{Actor: Stochastic policy mirror descent for AMDPs}\label{sec:policy_optimization}
Our goal in this section is to develop a computationally efficient first-order method to solve the AMDP problem~\eqref{def_problem} in the stochastic setting. To start with, we first explore the first-order information of the average reward/cost function and establish a variational inequality (VI) formulation of the problem.

The following simple lemma provides the gradient of the average-reward function \revision{}{and the proof follows from routine differentiation.}
\begin{lemma}[Gradient of average reward/cost]\label{lemma:gradient}
	For any $s\in\calS$ and $a \in \calA$,
	\begin{align*}
		\tfrac{\partial \rho(\pi)}{\partial \pi(a|s)} = \nu^\pi(s)\left( \bar {Q}^{\pi}(s,a) + \nabla h^{\pi}(s,a)\right).
	\end{align*}
	where $\nabla h^{\pi}(s,\cdot)$ denotes the gradient of $h^\pi(s)$ w.r.t. policy $\pi$.
\end{lemma}
In view of lemma~\ref{lemma:gradient}, the optimality condition of \eqref{def_problem} suggests us to solve the following VI 
\begin{align}\label{VI_formulation}
	\bbe_{s\in \nu^*} \big[ \langle \bar  Q^{\pi^*}, \pi(\cdot|s) - \pi^*(\cdot|s) \rangle + h^{\pi}(s) - h^{\pi^*}(s) \big]\geq 0.
\end{align}
This VI can be solved efficiently since it satisfies certain generalized monotonicity condition \citep{FacPang03,dang2015convergence,kotsalis2020simple1}, thanks to the following performance difference lemma (generalized from \cite{kakade2002approximately, zhang2021policy}). 
\begin{lemma}[Performance difference lemma]\label{lemma:perf-diff}
	Given any two policies $\pi$ and $\pi'$, we have
	\begin{equation}\label{perf-diff}
		\rho(\pi') - \rho(\pi) 
		= \mathbb{E}_{s\sim \nu^{\pi'}} \big[\langle \bar {Q}^{\pi}(s,\cdot), \pi'(\cdot|s) - \pi(\cdot|s)\rangle + h^{\pi'}(s)-h^{\pi}(s)\big].
	\end{equation}
\end{lemma}
As a consequence of Lemma~\ref{lemma:perf-diff}, we arrive at
\begin{align*}
	\rho(\pi^{*}) - \rho(\pi) = \mathbb{E}_{s\sim \nu^*} \big[\left\langle \bar {Q}^{\pi}\left(s, \cdot\right), \pi^{*}\left(\cdot | s \right)-\pi\left(\cdot | s\right)\right\rangle + h^{\pi^*}(s)-h^{\pi}(s)\big].
\end{align*}
Since $\rho(\pi)- \rho(\pi^*)\geq 0$ for any feasible policy $\pi$, we conclude that
\begin{equation}\label{eq_strong_montone}
	\mathbb{E}_{s\sim \nu^*} \big[\langle \bar {Q}^{\pi}\left(s, \cdot\right), \pi\left(\cdot | s\right)-\pi^*\left(\cdot | s\right)\rangle + h^{\pi}(s)-h^{\pi^*}(s)\big]
 = \rho(\pi) - \rho(\pi^*) \geq 0,
\end{equation}
which implies the VI in Ineq. \eqref{VI_formulation} satisfies the generalized monotonicity condition. 
We will design efficient stochastic optimization algorithms and provide theoretical guarantees for solving this VI in the following several subsections.

\subsection{Stochastic policy mirror descent for AMDPs}
In order to reduce the optimality gap
$\rho(\pi)-\rho(\pi^*)$, 
we intend to minimize the left-hand side of the identity in \eqnok{eq_strong_montone}, which inspires us to introdue the following iterative prox-mapping step:
\begin{equation}
	\pi_{k+1}(\cdot | s)=\underset{p(\cdot | s) \in \Delta_{|\mathcal{A}|}}{\arg \min }\left\{\lambda_{k}\left[\left\langle \bar {Q}^{\pi_{k}}(s, \cdot), p(\cdot | s)\right\rangle+h^{p}(s)\right]+D_{\pi_{k}}^{p}(s)\right\}, \forall s \in \mathcal{S},
\end{equation}
Here, the proximal term $D_{\pi_{k}}^{p}$ is added
so that $\pi_{k+1}$ does not move too far away from $\pi_k$.
Note however that in the RL setting the exact basic differential Q-function $\bar  Q^{\pi_k}$ is unavailable, thus we replace it with its stochastic estimator calculated with samples $\zeta_k$, denoted as $\stochQ{k}$, which leads to the following actor-critic type stochastic policy mirror descent (SPMD) algorithm.
\begin{algorithm}[h]\caption{Stochastic policy mirror descent (SPMD) for AMDPs}\label{alg:SPMD}
	\begin{algorithmic}[1]
		\STATE{\textbf{Input}: initial points $\pi_{0}(a|s)=1/|\calA|, ~\forall a \in \calA, s\in \calS$ and stepsize parameters $\lambda_k$; }
		\FOR{$k=0,1,\cdots, K$}
		\STATE{\textbf{Critic step}: Implement a policy evaluation algorithm to evaluate $\bar Q^{\pi_k}$ \revision{}{with samples $\zeta_k$},
			\begin{equation}\label{critic_step}
				\stochQ{k} = \text{critic}(\pi_k, \zeta_k). 
		\end{equation}}
		\STATE{
			\textbf{Actor step}: \revision{Implement}{Perform} the stochastic policy mirror descent update,
			\begin{equation}\label{SPMD_update}
				\pi_{k+1}(\cdot | s)=\underset{p(\cdot | s) \in \Delta_{|\mathcal{A}|}}{\arg \min }\left\{\lambda_{k}\left[\langle \stochQ{k}(s, \cdot), p(\cdot | s)\rangle+h^{p}(s)\right]+D_{\pi_{k}}^{p}(s)\right\}, \forall s \in \mathcal{S}.
			\end{equation}
		}
		\ENDFOR
	\end{algorithmic}
\end{algorithm}

\revision{}{To ensure the convergence of the SPMD method, we make the following assumptions on the evaluation error. In Section~\ref{sec:policy_evaluation}, we will propose concrete policy evaluation approaches with the desired accuracy guarantees.
\begin{assumption}\label{assump_critic_error}
Define $\widetilde{\mathcal{Q}}^{\pi_{k}}:=\mathbb{E}_{\zeta_{k}}\left[\stochQ{k}\right]$. There exist some $\varsigma, \sigma, \kappa \geq 0$ such that for any $k \geq 0$
\begin{align}
\spannorm{\widetilde{\mathcal{Q}}^{\pi_{k}}-\bar Q^{\pi_{k}}}  &\leq \varsigma,\label{eq-stochastic-assumption-1}\\
	\label{eq-stochastic-assumption-2}
\mathbb{E}_{\zeta_{k}}\big[\spannorm{ \stochQ{k} -\bar Q^{\pi_{k}}}^2\big]  &\leq \sigma^{2}, \\\qquad\mathbb{E}_{\zeta_{k}}\big[\spannorm{ \stochQ{k} }^2\big] &\leq \kappa^2.
\end{align}
\end{assumption}
}

Note that we \revision{utilize these error measures}{define the error measures in the span norm} because the update rule of Algorithm~\ref{alg:SPMD} is invariant with \revision{}{additive all-one vector} $\ones$ on $\stochQ{k}$. 
\revision{As shown in Section~\ref{sec:policy_evaluation}}{We will show in Section~\ref{sec:policy_evaluation} that} the bias term $\varsigma$ can be reduced much faster than the expected error term $\sigma^2$, which enables us to sharpen the analysis to obtain more efficient convergence guarantees \revision{}{ than one would obtain without distinguishing the bias and variance terms}. For convenience, we denote the $\sigma$-field $\calF_s:=\sigma(\zeta_0,...,\zeta_s)$ and
$
	\delta_{k}:=\stochQ{k}-\bar {Q}^{\pi_{k}}.
$

\subsection{SPMD with linear function approximation} \label{subsec_linear_function_appox}
\revision{}{
In many reinforcement learning applications, the state and action spaces (especially the state space) can be very large, making the original SPMD method, as well as the policy evaluation algorithms, computationally inefficient and memory-intensive. 
To address this issue, we discuss the incorporation of a linear function approximation scheme for the basic differential Q-functions $\bar Q^{\pi_k}$ within the SPMD method in this subsection.

In particular, one can choose $$\mathbb{S}:=\text{span}\{\psi_1, ..., \psi_d\},$$ for $d$ linearly independent basis vectors $\psi_1, ..., \psi_d \in \bbr^{|\calS| |\calA|}$. For each state action pair $s\in \calS,~ a \in \calA$, we let $\psi(s,a): = [\psi_1(s,a),\psi_2(s,a),...,\psi_d(s,a)]^\top$ denote its feature vector. For simplicity, we define $\Psi := [\psi_1, ..., \psi_d] \in \bbr^{(|\calS||\calA|)\times d}$. It should be noted that we fix the feature matrix $\Psi$ for all SPMD iterations, and use $\theta_k \in \bbr^d$ to denote the low-dimensional estimator for evaluating policy $\pi_k$, i.e., $\stochQ{k} = \Psi \theta_k$. 

    Due to the linear function approximation, we will introduce an approximation error when $\bar Q^{\pi_k}\notin\mathbb{S}$. Consequently, the bias term $\varsigma$ defined in \eqref{eq-stochastic-assumption-1} cannot be controlled towards $0$. However, in the following sections, we will show that the bias will not be accumulated while running the SPMD method; instead, it can converge to the optimal solution up to the upper bound  $\varsigma$ for the bias.
    By contrast, some alternative algorithms, such as the stochastic Q-learning methods, could diverge with simple function approximation due to the absence of contraction properties for the Bellman optimality equation with function approximation, see \cite{baird1995residual,tsitsiklis1996feature}.

Moreover, we do not need to enumerate all state-action pairs for the policy updating
step \eqref{SPMD_update} in the SPMD method.
By utilizing the optimality condition of the prox-mapping updates \eqref{SPMD_update} of SPMD under entropy regularization, we obtain the following SPMD update rule for the parameters $\widetilde \theta_k$ in the lower-dimensional space.

\begin{lemma}
Assume $h^\pi(s) := \omega\cdot \tsum_{a \in \calA} \pi(a|s) \log \pi(a|s)$ for some $\omega \geq 0$, and $\pi_0(a|s) = \tfrac{1}{|\calA|}$ for any $s \in \calA, a \in \calA$. Consider the linear parametrization $\stochQ{k} = \Psi \theta_k$. Then  step~\eqref{SPMD_update} in Algorithm~\ref{alg:SPMD} can be equivalently written in $\mathbb{R}^d$ as 
\begin{align}
&\widetilde \theta_0 = \mathbf{0},\nn\\
&\widetilde \theta_{k+1} = -\tfrac{\lambda_k}{1+\lambda_k \omega} \theta_k + \tfrac{1}{1+\lambda_k \omega} \widetilde \theta_k, \quad k \geq 0. \label{d_SPMD}
\end{align}
\end{lemma}
\proof{Proof.}
By the definition of $D_{\pi}^{\pi'}$ and $h^\pi$, we can equivalently written Eq.~\eqref{SPMD_update} as
\begin{align*}
				\pi_{k+1}(\cdot | s)&=\underset{p(\cdot | s) \in \Delta_{|\mathcal{A}|}}{\arg \min }\left\{\lambda_{k}\langle \stochQ{k}(s, \cdot) - \omega \log\pi_{k}(\cdot|s), p(\cdot | s)\rangle+ (1+\lambda_k \omega )\langle \log p(\cdot|s), p(\cdot|s)\rangle\right\}, \forall s \in \mathcal{S}\\
    &=\underset{p(\cdot | s) \in \Delta_{|\mathcal{A}|}}{\arg \min }\left\{\lambda_{k}\langle (\Psi \theta_k)(s,\cdot) - \omega \log\pi_{k}(\cdot|s), p(\cdot | s)\rangle+ (1+\lambda_k\omega )\langle \log p(\cdot|s), p(\cdot|s)\rangle\right\}, \forall s \in \mathcal{S}.
			\end{align*}
   Then by checking the Karush-Kuhn-Tucker conditions, we obtain that 
   \begin{align*}
   \pi_{k+1}(\cdot | s) \propto \exp \left( -\tfrac{\lambda_k}{1+\lambda_k \omega}(\Psi \theta_k)(s, \cdot) + \tfrac{1}{1+\lambda_k \omega}\log \pi_k(\cdot|s)\right).
   \end{align*}
   Then by using the above equation recursively and noticing that $\pi_0(a|s) = \tfrac{1}{|\calA|}$, it is easy to see that 
   $$
   \pi_{k+1}(\cdot |s) \propto \exp\big((\Psi\widetilde \theta_{k+1})(s,\cdot) \big),
   $$
   where $\widetilde \theta_{k+1}$ is defined in Eqs.~\eqref{d_SPMD}.
\endproof


From Eq.~\eqref{d_SPMD}, the SPMD updates can be fully implemented in the low-dimensional subspace $\bbr^d$ if the linear parametrization of the differential-Q-function $\bar Q^{\pi_k}$, namely $\theta_k$, is given in each SPMD iteration. 

To implement the SPMD method, we need to combine the above update rule with a policy evaluation subroutine that can handle linear function approximation. Note that in order to evaluate the policy $\pi_k$ under the on-policy setting, we need to generate the Markov trajectory following policy $\pi_k$. Nevertheless, there is no need to calculate the whole policy $\pi_k \in \bbr^{|\calS||\calA|}$ and save it in memory. Instead, we only need to calculate the policy $\pi(\cdot|s)$ following 
\begin{align}
\pi_{k+1}(\cdot |s) \propto \exp\big((\Psi\widetilde \theta_{k+1})(s,\cdot) \big), \label{d_SPMD_2}
\end{align}
once reaching a new state $s$, which only requires a computational complexity of $\order(|\calA|\times d)$. This structure makes the algorithm implementable when the state space is super large or even infinite. In Section~\ref{sec_VRTD}, We will propose some efficient on-policy evaluation approaches with linear function approximation.
}

\subsection{SPMD for unregularized AMDPs}\label{subsec_unreg} \revision{}{In this subsection, we start proving the convergence guarantees for SPMD.} 
\revision{}{The following lemma characterizes the convergence behavior for each iteration of Algorithm~\ref{alg:SPMD}. Recall the definition of the convexity modulus $\omega\geq 0$ in Ineq.~\eqref{convex_regularizer}.
}
\revision{}{\begin{lemma}\label{lemma_SPMD_convergence} Assume Assumption~\ref{assump_critic_error} holds. If $h^\pi(s)$ is $L$-Lipschitz, we have for any $k\geq 0$,
\begin{align}\label{SPMD_convergence}
\bbe\left[\lambda_k\big( \rho(\pi_k) - \rho(\pi^*) + (1+\lambda_k \omega) D(\pi_{k+1}, \pi^*) \big) \right] \leq \bbe[D(\pi_k, \pi^*)] + \lambda_k\varsigma + \lambda_k^2(\tfrac{\kappa^2}{4} + L^2),
\end{align}
where $D({\pi},{\pi'}):=\bbe_{s\sim\nu^*}[D_{\pi}^{\pi'}(s)]$, and $\kappa$, $\varsigma$ are defined in Assumption~\ref{assump_critic_error}.
\end{lemma}
}
\revision{}{
\proof{Proof.}
First, by the optimality condition of \eqref{SPMD_update}, we have the following inequality (three-point lemma, see, e.g., Lemma 3.5 of \cite{lan2020first}). For any $p(\cdot|s) \in \Delta_{|\calA|}$,
\begin{align}\label{three_point_ineq}
		&\lambda_{k}\left[\langle \stochQ{k}(s, \cdot), \pi_{k+1}(\cdot | s)-p(\cdot | s)\rangle+h^{\pi_{k+1}}(s)-h^{p}(s)\right]+D_{\pi_{k}}^{\pi_{k+1}}(s) \nn\\
		&\leq D_{\pi_{k}}^{p}(s)-\left(1+\lambda_{k} \omega\right) D_{\pi_{k+1}}^{p}(s).
	\end{align}
 Set $p = \pi^*$. Notice that we have
 \begin{align*}
&\lambda_{k}\left(\langle \stochQ{k}(s, \cdot), \pi_{k+1}(\cdot | s)-\pi^*(\cdot | s)\rangle+h^{\pi_{k+1}}(s)-h^{\pi^*}(s)\right)+D_{\pi_{k}}^{\pi_{k+1}}(s)\\
& = \lambda_{k}\left(\langle \stochQ{k}(s, \cdot), \pi_{k}(\cdot | s)-\pi^*(\cdot | s)\rangle+h^{\pi_{k}}(s)-h^{\pi^*}(s)\right)\\
&\quad + \left[\lambda_{k}\left(\langle \stochQ{k}(s, \cdot), \pi_{k+1}(\cdot | s)-\pi_{k}(\cdot | s)\rangle+h^{\pi_{k+1}}(s)-h^{\pi_k}(s)\right) +D_{\pi_{k}}^{\pi_{k+1}}(s)\right]\\
& \overset{(i)}\geq \lambda_{k}\left(\langle \stochQ{k}(s, \cdot), \pi_{k}(\cdot | s)-\pi^*(\cdot | s)\rangle+h^{\pi_{k}}(s)-h^{\pi^*}(s)\right)\\
&\quad +\left[-\tfrac{\lambda_k}{2} \spannorm{\stochQ{k}(s, \cdot)}\|\pi_{k+1}(\cdot|s) - \pi_k(\cdot|s)\|_1 - \lambda_k L \|\pi_{k+1}(\cdot|s) - \pi_k(\cdot|s)\|_1 + \tfrac{1}{2}\|\pi_{k+1}(\cdot|s) - \pi_k(\cdot|s)\|_1^2\right]\\
& \overset{(ii)}\geq \lambda_{k}\left(\langle \stochQ{k}(s, \cdot), \pi_{k}(\cdot | s)-\pi^*(\cdot | s)\rangle+h^{\pi_{k}}(s)-h^{\pi^*}(s)\right) - \tfrac{\lambda_k^2}{4} \spannorm{\stochQ{k}(s, \cdot)}^2 - \lambda_k^2 L^2,
 \end{align*}
 where step (i) follows from $\langle \ones, \pi_{k+1}(\cdot|s)-\pi_{k}(\cdot|s) \rangle=0$, Cauchy-Schwarz inequality, the $L$-Lipschitz of $h^\pi$ and $D_{\pi_{k}}^{\pi_{k+1}}(s)\geq \tfrac{1}{2}\|\pi_{k+1}(\cdot|s)-\pi_{k}(\cdot|s)\|_1^2$, and step (ii) follows from Young's inequality. Combining the previous two inequalities with $p = \pi^*$, we obtain
\begin{align*}
&\lambda_{k}\left(\langle \stochQ{k}(s, \cdot), \pi_{k}(\cdot | s)-\pi^*(\cdot | s)\rangle+h^{\pi_{k}}(s)-h^{\pi^*}(s)\right)\\
&\leq D_{\pi_k}^{\pi^*}(s) - \left(1+\lambda_{k} \omega\right) D_{\pi_{k+1}}^{\pi^*}(s) + \tfrac{\lambda_k^2}{4} \spannorm{\stochQ{k}(s, \cdot)}^2 + \lambda_k^2 L^2,
\end{align*}
 By taking expectation with respect to $\calF_{k-1}$ and using Assumption~\ref{assump_critic_error}, we obtain that
 \begin{align*}
&\lambda_{k}\left(\langle  \bar  Q^{\pi_k}(s, \cdot), \pi_{k}(\cdot | s)-\pi^*(\cdot | s)\rangle+h^{\pi_{k}}(s)-h^{\pi^*}(s)\right)\\
&\leq D_{\pi_k}^{\pi^*}(s) - \left(1+\lambda_{k} \omega\right) D_{\pi_{k+1}}^{\pi^*}(s) + \tfrac{\lambda_k}{2}\spannorm{\widetilde{\mathcal{Q}}^{\pi_{k}}-\bar Q^{\pi_{k}}}\|\pi_k(\cdot|s) - \pi^*(\cdot|s)\|_1 + \tfrac{\lambda_k^2\bbe[\spannorm{\stochQ{k}(s, \cdot)}^2|\calF_{k-1}] }{4} + \lambda_k^2 L^2\\
&\leq D_{\pi_k}^{\pi^*}(s) - \left(1+\lambda_{k} \omega\right) D_{\pi_{k+1}}^{\pi^*}(s) + \lambda_k\varsigma + \tfrac{\lambda_k^2 \kappa^2}{4}+ \lambda_k^2 L^2.
\end{align*}
By taking expectation with respect to $\nu^*$ and full expectation with respect to $\calF_{-1}$, and using the performance difference lemma (Lemma~\ref{lemma:perf-diff}), we obtain
\begin{align*}
\bbe\left[\lambda_k\big( \rho(\pi_k) - \rho(\pi^*) + (1+\lambda_k \omega) D(\pi_{k+1}, \pi^*) \big) \right] \leq \bbe[D(\pi_k, \pi^*)] + \lambda_k\varsigma + \lambda_k^2(\tfrac{\kappa^2}{4} + L^2),
\end{align*}
which completes the proof.
\endproof
}

\revision{}{
The following theorem establishes the convergence guarantees of the SPMD algorithm for unregularized AMDPs, i.e., $\omega = 0$ and $L = 0$.
\begin{theorem}\label{thm_unreg_1}
    Suppose that $\omega = 0$, $L = 0$, and Assumption~\ref{assump_critic_error} holds. Set $\lambda_k = \lambda=\frac{2\sqrt{\log|\calA|}}{\kappa\sqrt{K}}$ in Algorithm~\ref{alg:SPMD}. Then we have
    \begin{align}\label{theorem_unreg_1}
		\bbe\big[\tsum_{k=0}^{K-1} \big(\rho(\pi_k) - \rho(\pi^*)\big)\big] \leq\kappa\sqrt{K\log|\calA|} + K\varsigma.
	\end{align}
	Assume that $r$ is uniformly chosen from $\{0,1,...,K-1\}$, we have
	\begin{align}\label{theorem_unreg_2}
		\bbe[\rho(\pi_r) -\rho(\pi^*)] \leq \tfrac{\kappa\sqrt{\log|\calA|}}{\sqrt{K}} + \varsigma.
	\end{align}
\end{theorem}
}
\revision{}{
\proof{Proof.}
By taking the telescope sum of $k=0,.., K-1$ for Ineq.~\eqref{SPMD_convergence}, we have
\begin{align*}
&\bbe\big[\tsum_{k=0}^{K-1} \big(\rho(\pi_k) - \rho(\pi^*)\big)\big] + \tfrac{1}{\lambda}\bbe[D(\pi_K, \pi^*)] \leq \tfrac{1}{\lambda}D(\pi_0, \pi^*) + K  \varsigma + \tfrac{K \lambda \kappa^2}{4}.
\end{align*}
Invoking that $r$ is uniformly chosen from $\{0,..,K-1\}$ gives us
	\begin{align*}
		\bbe\left[ \rho(\pi_r) - \rho(\pi^*)\right] \leq \tfrac{D(\pi_0, \pi^*)}{\lambda K} + \tfrac{\lambda \kappa^2}{4} + \varsigma.
	\end{align*}
By applying the upper bound $D(\pi_0, \pi^*)\leq \log|\calA|$ and the choice of $\lambda$, we complete the proof of Ineqs. \eqref{theorem_unreg_1} and \eqref{theorem_unreg_2}.	
\endproof
Clearly, the SPMD method achieves the $\mathcal{O}(1/\sqrt{K})$ rate of convergence for unregularized AMDPs. In the next subsection, we will show that this rate can be improved to $\mathcal{O}(1/K)$ for strongly-convex regularized AMDPs.
}

\subsection{SPMD for strongly-convex regularized AMDPs}\label{subsec_reg} \revision{}{In this subsection, we consider using the SPMD method to solve strongly convex regularized MDPs, i.e., $\omega > 0$. 
We first consider the case when $h^\pi$ is $L$-Lipschitz and propose the following theorem.
\begin{theorem}\label{thm_reg_1}
Suppose $\omega >0$, $h^\pi$ is $L$-Lipschitz, and Assumption~\ref{assump_critic_error} holds. Set $\lambda_k = \lambda = \tfrac{2}{\omega (k+1)}$, we have
\begin{align*}
\tfrac{1}{\sum_{k=0}^{K-1}(k+2)} \cdot \tsum_{k=0}^{K-1}(k+2)\bbe[\rho(\pi_k) - \rho(\pi^*)] + \omega \bbe[D(\pi_K, \pi^*)] \leq \tfrac{2\omega}{K(K+3)}D(\pi_0, \pi^*) + \varsigma + \tfrac{2\kappa^2 + 8 L^2}{\omega (K+3)}.
\end{align*}
Consequently, assume that $r$ is chosen from $\{0,1,...,K-1\}$ with distribution $\mathbb{P}(r=i) = \tfrac{2(i+2)}{K(K+3)}$, we have
	\begin{align}\label{theorem_reg_2}
		\bbe[\rho(\pi_r) -\rho(\pi^*)] \leq \tfrac{2\omega}{K(K+3)}D(\pi_0, \pi^*) + \varsigma + \tfrac{2\kappa^2 + 8 L^2}{\omega (K+3)},
	\end{align}
 where $\kappa$ and $\varsigma$ are defined in Assumption~\ref{assump_critic_error}.
\end{theorem}}
\revision{}{
\proof{Proof.}
Multiplying $(k+1)(k+2)\omega/2$ on both sides of Ineq.~\eqref{SPMD_convergence} in Lemma~\ref{lemma_SPMD_convergence}, and utilizing the choice of $\lambda$, we have
\begin{align*}
\bbe\left[k+2\big( \rho(\pi_k) - \rho(\pi^*) + \tfrac{(k+2)(k+3){\omega}}{2} D(\pi_{k+1}, \pi^*) \big) \right] \leq \tfrac{(k+1)(k+2){\omega}}{2}\bbe[D(\pi_k, \pi^*)] + (k+2)\varsigma + \tfrac{2(k+2)}{\omega(k+1)}(\tfrac{\kappa^2}{4} + L^2).
\end{align*}
By taking telescope sum of $k=0, ..., K-1$ of the above inequality, we obtain
\begin{align*}
    \bbe\left[\tsum_{k=0}^{K-1} (k+2)(\rho(\pi_k) - \rho(\pi^*)) + \tfrac{K(K+1)\omega }{2}D(\pi_K, \pi^*)\right] \leq \omega D(\pi_0, \pi^*) + \tfrac{K(K+3)\varsigma}{2} + \tfrac{4 K}{\omega}(\tfrac{\kappa^2}{4} + L^2),
\end{align*}
which completes the proof.
\endproof
}

\revision{}{Next, we remove the assumption that $h^\pi$ is Lipschitz continuous, but require the ergodicity assumption (Assumption~\ref{assump_rho_0}). As a consequence, there exists a $\Gamma \in (0,1)$, such that for all feasible policy $\pi$,
\begin{align*}
	\min_{s\in\calS} \nu^\pi(s) \geq 1- \Gamma.
\end{align*}
Under this assumption, the convergence guarantees of the SPMD method for solving AMDPs are mostly similar to the SPMD method for DMDPs presented in \cite{lan2021policy}. To demonstrate this claim, we first establish two supporting lemmas whose analogs are Lemma 12 and 13 in  \cite{lan2021policy}. }
\begin{lemma} \label{lemma:convexity}
	For any $s\in \mathcal{S}$, we have
	\begin{align*}
		\tfrac{1}{1-\Gamma}\big(\rho(\pi_{k+1}) - \rho(\pi_k)\big) \leq \langle \stochQ{k}(s, \cdot), &\pi_{k+1}(\cdot \mid s)-\pi_{k}(\cdot \mid s)\rangle+ h^{\pi_{k+1}}(s)-h^{\pi_{k}}(s)\\
		&+\tfrac{1}{\lambda_{k}} D_{\pi_{k}}^{\pi_{k+1}}(s)
		+\tfrac{\lambda_{k}\spannorm{\delta_{k}}^2}{8(1-\Gamma)}
	\end{align*}
\end{lemma}
\proof{Proof.}
	Using Lemma~\ref{lemma:perf-diff}, we obtain
	\begin{align}\label{conv_step_1}
		\rho(\pi_{k+1}) - \rho(\pi_{k}) 
		&= \mathbb{E}_{s'\sim \nu^{\pi_{k+1}}}\big[\langle \stochQ{k}(s', \cdot), \pi_{k+1}(\cdot | s')-\pi_{k}(\cdot | s')\rangle + h^{\pi_{k+1}}(s')-h^{\pi_{k}}(s')\nn\\
		&\quad-\langle\delta_{k}(s',\cdot), \pi_{k+1}(\cdot \mid s^{\prime})-\pi_{k}(\cdot \mid s^{\prime})\rangle\big]\nn\\
		&\overset{(i)}\leq \mathbb{E}_{s'\sim \nu^{\pi_{k+1}}}\big[\langle \stochQ{k}(s', \cdot), \pi_{k+1}(\cdot | s')-\pi_{k}(\cdot | s')\rangle + h^{\pi_{k+1}}(s')-h^{\pi_{k}}(s')\nn\\
		&\quad+\tfrac{1}{2 \lambda_{k}}\left\|\pi_{k+1}(\cdot \mid s^{\prime})-\pi_{k}(\cdot \mid s^{\prime})\right\|_{1}^{2}+\tfrac{\lambda_{k}\spannorm{\delta_{k}}^2}{8}]\nn\\ 
		&\overset{(ii)}\leq \mathbb{E}_{s'\sim \nu^{\pi_{k+1}}}\big[\langle \stochQ{k}(s', \cdot), \pi_{k+1}(\cdot | s')-\pi_{k}(\cdot | s')\rangle + h^{\pi_{k+1}}(s')-h^{\pi_{k}}(s').\nn\\
		&\quad+\tfrac{1}{\lambda_{k}} D_{\pi_{k}}^{\pi_{k+1}}(s^{\prime})+\tfrac{\lambda_{k}\spannorm{\delta_{k}}^2}{8}\big], 
	\end{align}
	where step (i) follows from Young's inequality and the fact that $\langle \ones, \pi_{k+1}(\cdot|s')-\pi_{k}(\cdot|s') \rangle=0$, step (ii) follows from the strong convexity of $D_{\pi_{k}}^{\pi_{k+1}}$ with respect to $\ell_1$-norm. Recalling the three-point lemma (Ineq.~\eqref{three_point_ineq}) and taking $p = \pi_k$, we obtain that for all $s'\in\calS$,
	\begin{align*}
		&\langle \stochQ{k}(s^{\prime}, \cdot), \pi_{k+1}(\cdot \mid s^{\prime})-\pi_{k}(\cdot \mid s^{\prime})\rangle+h^{\pi_{k+1}}(s^{\prime})-h^{\pi_{k}}(s^{\prime})+\tfrac{1}{\lambda_{k}} D_{\pi_{k}}^{\pi_{k+1}}(s^{\prime}) \\
		&\leq-\tfrac{1}{\lambda_{k}}\big[(1+\lambda_{k} \omega) D_{\pi_{k+1}}^{\pi_{k}}(s^{\prime})\big] \leq 0,
	\end{align*}
	which implies that 
	\begin{align*}
		&\mathbb{E}_{s^{\prime} \sim \nu^{\pi_{k+1}}}\big[\langle \stochQ{k}(s^{\prime}, \cdot), \pi_{k+1}(\cdot \mid s^{\prime})-\pi_{k}(\cdot \mid s^{\prime})\rangle+h^{\pi_{k+1}}(s^{\prime})-h^{\pi_{k}}(s^{\prime})+\tfrac{1}{\lambda_{k}} D_{\pi_{k+1}}^{\pi_{k}}(s^{\prime})\big] \\
		&{\leq \nu^{\pi_{k+1}}(s)\big[\langle\stochQ{k}(s, \cdot), \pi_{k+1}(\cdot \mid s)-\pi_{k}(\cdot \mid s)\rangle+h^{\pi_{k+1}}(s)-h^{\pi_{k}}(s)+\tfrac{1}{\lambda_{k}} D_{\pi_{k}}^{\pi_{k+1}}(s)\big]} \\
		&{\leq (1-\Gamma) \big[\langle\stochQ{k}(s, \cdot), \pi_{k+1}(\cdot \mid s)-\pi_{k}(\cdot \mid s)\rangle+h^{\pi_{k+1}}(s)-h^{\pi_{k}}(s)+\tfrac{1}{\lambda_{k}} D_{\pi_{k}}^{\pi_{k+1}}(s)\big],}
	\end{align*}
	and the desired result follows from combining the above inequality with Ineq. \eqref{conv_step_1}.
\endproof

\begin{lemma}\label{lemma:iterative}
	For any $k\geq 0$, we have
	\begin{align*}
		&\mathbb{E}\big[\tfrac{1}{1-\Gamma}\big(\rho(\pi_{k+1})-\rho(\pi^*)\big) +  \big(\tfrac{1}{\lambda_{k}}+\omega\big) D({\pi_{k+1}},{\pi^{*}}) \big] \\
		&\leq \mathbb{E}\big[ \tfrac{\Gamma}{1-\Gamma}\big(\rho(\pi_k) - \rho(\pi^*)\big) +  \tfrac{1}{\lambda_{k}} D({\pi_{k}},{\pi^{*}}) \big]+ \varsigma+\tfrac{\lambda_{k} \sigma^{2}}{8(1-\Gamma)},
	\end{align*}
	where the expectation is taken with respect to $\mathcal{F}_{-1}$, $D({\pi},{\pi'}):=\bbe_{s\sim\nu^*}[D_{\pi}^{\pi'}(s)]$, and $\sigma$, $\varsigma$ are defined in Assumption~\ref{assump_critic_error}.
\end{lemma}
\proof{Proof.}
	By applying Ineq.~\eqref{three_point_ineq} with $p=\pi^*$ we obtain
	\begin{align*}
		&\lambda_{k}\big[\langle \stochQ{k}(s, \cdot), \pi_{k+1}(\cdot \mid s)-\pi^{*}(\cdot \mid s)\rangle+h^{\pi_{k+1}}(s)-h^{\pi^{*}}(s)\big]+D_{\pi_{k}}^{\pi_{k+1}}(s) \\
		&\leq D_{\pi_{k}}^{\pi^{*}}(s)-\left(1+\lambda_{k} \omega\right) D_{\pi_{k+1}}^{\pi^{*}}(s),
	\end{align*}
	Combining the above inequality with Lemma~\ref{lemma:convexity} implies that 
	\begin{align*}
		&\langle \stochQ{k}(s, \cdot), \pi_{k}(\cdot \mid s)-\pi^{*}(\cdot \mid s)\rangle+h^{\pi_{k}}(s)-h^{\pi^{*}}(s)+\tfrac{1}{1-\Gamma}(\rho(\pi_{k+1})-\rho(\pi_{k})) \\
		&\leq \tfrac{1}{\lambda_{k}} D_{\pi_{k}}^{\pi^{*}}(s)-\left(\tfrac{1}{\lambda_{k}}+\omega\right) D_{\pi_{k+1}}^{\pi^{*}}(s)+\tfrac{\lambda_{k}\spannorm{\delta_{k}}^2}{8(1-\Gamma)}.
	\end{align*}
	Taking full expectation with $\calF_{-1}$ and $\nu^*$, and using Lemma~\ref{lemma:perf-diff}, we arrive at
	\begin{align*}
		\mathbb{E}\big[(\rho(\pi_{k})-\rho(\pi^*)) + \tfrac{1}{1-\Gamma}(\rho(\pi_{k+1})-\rho(\pi_{k})) \big] \leq \mathbb{E}\big[\tfrac{1}{\lambda_{k}} D_{\pi_{k}}^{\pi^{*}}(s)-\big(\tfrac{1}{\lambda_{k}}+\omega\big) D_{\pi_{k+1}}^{\pi^{*}}(s)\big]+ \varsigma+\tfrac{\lambda_{k} \sigma^{2}}{8(1-\Gamma)}.
	\end{align*}
	Rearranging terms yields the desired result.
\endproof

\vgap

	In view of Lemma~\ref{lemma:iterative} and its analog (Lemma 13 in \cite{lan2021policy}), it can be easily proved that Theorem 5 and 6 in \cite{lan2021policy} still holds by replacing $f(\cdot)$ with $\rho(\cdot)/(1-\Gamma)$ and $\gamma$ with $\Gamma$. Similarly, the convergence guarantees for deterministic PMD in \cite{lan2021policy} can also be extended to AMDPs.
Nevertheless, all the convergence guarantees in \cite{lan2021policy} failed to decompose the convergence rate into a deterministic error part with linear convergence rate and a stochastic error part with $\mathcal{O}(1/K)$ convergence rate, which motivates us to derive the following theorem.
\begin{theorem}\label{thm_regularized_SPMD}
	Suppose $\omega >0$, and Assumption~\ref{assump_rho_0} and Assumption~\ref{assump_critic_error} hold. If $K$ is fixed and for any $0\leq k\leq K$, set 
	\begin{align}\label{eta_cond}
		\lambda_k = \lambda = \min\{\tfrac{1-\Gamma}{\omega\Gamma}, \tfrac{q \log K}{\omega K}\}, \quad \text{with} \quad q = 2\big(1+ \log(\tfrac{8\Gamma(1-\Gamma)\omega^2D(\pi_0,\pi^*)+ 8\Gamma(1-\Gamma)\omega (\rho(\pi_0)-\rho(\pi^*))}{\sigma^2}) \big),
	\end{align}
	then
	\begin{align*}
		\bbe\left[ D({\pi_{K+1}},{\pi^{*}}) \right] \leq\tfrac{\Gamma^K}{\omega}\big(\rho(\pi_0) - \rho(\pi^*)\big) +  \Gamma^K D({\pi_{0}},{\pi^{*}}) +\tfrac{\sigma^2}{4\omega^2\Gamma(1-\Gamma)K}(1+q\log K) +  \tfrac{2\varsigma}{\omega\Gamma}.
	\end{align*}
\end{theorem}
\proof{Proof.}
	By Lemma~\ref{lemma:iterative} and $\lambda_k=\lambda$, we have that with $\Theta_k>0$
	\begin{align*}
		&\Theta_k\mathbb{E}\big[\tfrac{1}{1-\Gamma}\left(\rho(\pi_{k+1})-\rho(\pi^*)\right) +  \left(\tfrac{1}{\lambda}+\omega\right) D({\pi_{k+1}},{\pi^{*}}) \big] \\
		&\leq \Theta_k\mathbb{E}\big[ \tfrac{\Gamma}{1-\Gamma}\big(\rho(\pi_k) - \rho(\pi^*)\big) +  \tfrac{1}{\lambda} D({\pi_{k}},{\pi^{*}}) \big]+\Theta_k \varsigma+\tfrac{\Theta_k\lambda \sigma^{2}}{8(1-\Gamma)}.
	\end{align*}
	Set $\Theta_k = \Theta^k$ and $\Theta \leq \min\{\tfrac{1}{\Gamma}, 1+\omega\lambda\}$, then by taking telescope sum from $k=0$ to $K$, we obtain
	\begin{align*}
		&\Theta^K \cdot\mathbb{E}\big[\tfrac{1}{1-\Gamma}\left(\rho(\pi_{K+1})-\rho(\pi^*)\right) +  \left(\tfrac{1}{\lambda}+\omega\right) D({\pi_{K+1}},{\pi^{*}}) \big] \\
		&\leq  \tfrac{\Gamma}{1-\Gamma}\big(\rho(\pi_0) - \rho(\pi^*)\big) +  \tfrac{1}{\lambda} D({\pi_{0}},{\pi^{*}}) +\tfrac{ \Theta^K\varsigma}{1-\Theta^{-1}}+\tfrac{\Theta^K\lambda \sigma^{2}}{8(1-\Gamma)(1-\Theta^{-1})}.
	\end{align*}
	From $\Theta \leq \min\{\tfrac{1}{\Gamma}, 1+\omega\lambda\}$, we have $(1-\Theta^{-1})^{-1}\geq \max\{\tfrac{1}{1-\Gamma}, \tfrac{1+\omega\lambda}{\omega\lambda}\}$. Using this relationship and mutiplying $\Theta^{-K}\lambda$ on both sides of the above inequality, we obtain
	\begin{align}\label{unify_step_1}
		&\lambda\mathbb{E}\big[\tfrac{1}{1-\Gamma}\left(\rho(\pi_{K+1})-\rho(\pi^*)\right)\big] +  \left(1+\omega \lambda\right)\bbe\left[ D({\pi_{K+1}},{\pi^{*}}) \right] \nn\\
		&\leq \Theta^{-K} \tfrac{ \lambda\Gamma}{1-\Gamma}\big(\rho(\pi_0) - \rho(\pi^*)\big) +  \Theta^{-K} D({\pi_{0}},{\pi^{*}})+\tfrac{ \lambda\varsigma}{1-\Gamma} +\tfrac{(1+\omega\lambda)\varsigma}{\omega}+ \tfrac{\lambda^2\sigma^2}{8(1-\Gamma)^2} + \tfrac{(1+\omega\lambda)\lambda\sigma^2}{8(1-\Gamma)\omega}.
	\end{align}
	Take $\lambda = \min\{\tfrac{1-\Gamma}{\omega\Gamma},\tfrac{q\log K}{\omega K}\}$ and $\Theta = \min \{\tfrac{1}{\Gamma}, 1+\tfrac{q\log K }{K}\}$, then we have
	\begin{align}\label{unify_step_2}
		\Theta^{-K} \leq \Gamma^K + (1+\tfrac{q\log K }{K})^{-K}\leq \Gamma^K + \exp(-\tfrac{1}{2}q\log K) \leq \Gamma^K + \tfrac{1}{K^{q/2}}.
	\end{align}
	Invoking the definition that 
	$
		q = 2\big(1+ \log(\tfrac{8\Gamma(1-\Gamma)\omega^2D(\pi_0,\pi^*)+ 8\Gamma(1-\Gamma)\omega (\rho(\pi_0)-\rho(\pi^*))}{\sigma^2}) \big), 
$
	we have
	\begin{align}\label{unify_step_3}
		\tfrac{\lambda \Gamma}{(1-\Gamma)K^{q/2}} \leq \tfrac{\sigma^2}{8\Gamma(1-\Gamma)\omega^2 K}\quad\text{and}\quad \tfrac{D(\pi_0,\pi^*)}{K^{q/2}} \leq \tfrac{\sigma^2}{8\Gamma(1-\Gamma)\omega^2 K}.
	\end{align}
	Combining Ineqs.~\eqref{unify_step_1}, \eqref{unify_step_2} and \eqref{unify_step_3} yields
	\begin{align*}
		&\lambda\mathbb{E}\big[\tfrac{1}{1-\Gamma}\left(\rho(\pi_{K+1})-\rho(\pi^*)\right)\big] +  \bbe\big[ D({\pi_{K+1}},{\pi^{*}}) \big] \nn\\
		&\leq \tfrac{\Gamma^K}{\omega}\big(\rho(\pi_0) - \rho(\pi^*)\big) +  \Gamma^K D({\pi_{0}},{\pi^{*}}) + \tfrac{\lambda \Gamma}{(1-\Gamma)K^{q/2}}+\tfrac{D(\pi_0,\pi^*)}{K^{q/2}} +\tfrac{2 \varsigma}{\omega\Gamma}+ \tfrac{\lambda\sigma^2}{4\omega\Gamma(1-\Gamma)}\\
		&\leq\tfrac{\Gamma^K}{\omega}\big(\rho(\pi_0) - \rho(\pi^*)\big) +  \Gamma^K D({\pi_{0}},{\pi^{*}}) +\tfrac{\sigma^2}{4\omega^2\Gamma(1-\Gamma)K}(1+q\log K) +  \tfrac{2 \varsigma}{\omega\Gamma},
	\end{align*}
	which completes the proof.
\endproof
Theorem~\ref{thm_regularized_SPMD} establishes the convergence guarantee for the distance between the output policy and the optimal policy. Meanwhile, we are allowed to output the policy generated by the last iteration instead of randomly outputting a policy as in Theorem~\ref{thm_unreg_1} and Theorem~\ref{thm_reg_1}.

\revision{}{
In the next section, we will propose some novel policy evaluation subroutines under both generative and Markovian noise settings. Then by combining the ``actor'' and ``critic'' parts, we establish sample complexities for different observation models under different assumptions.
}

\section{Critic: Policy evaluation for differential Q-functions}\label{sec:policy_evaluation}
\revision{To serve as a fundamental building block of policy optimization, we focus on estimating the differential Q-function since it is closely related to the gradient information of the average reward/cost function. Moreover, the results and analyses extend to the estimation of the differential value function naturally. 
It should be noted that once the policy is fixed, $h^\pi$ is a fixed function which together with $c(s,a)$ can be viewed as the overall cost function of the policy of interest. Therefore, for sack of simplicity,  in this section we use $c(s,a)$ to represent $c(s,a)+h^\pi(s)$. 
And we let $|c(s,a)|\leq \bar  c$, for all $s\in\calS$ and $a\in\calA$.}
{In this section, we discuss the estimation of the differential Q-function of AMDPs. It should be noted that once the policy is fixed, $h^\pi$ is a fixed function which together with $c(s,a)$ can be viewed as the overall cost function of the policy of interest. Therefore, for the sake of simplicity,  in this section, we use $c(s,a)$ to represent $c(s,a)+h^\pi(s)$. 
We also assume $|c(s,a)|\leq \bar  c$ for all $s\in\calS$ and $a\in\calA$. Moreover, the results and analyses naturally extend to the simpler task of estimating the differential value function $\bar V^\pi$. }

\revision{}{
In Section~\ref{sec:generative}, we propose a direct multiple trajectory method (Algorithm~\ref{alg:Mult}) and establish its convergence guarantees under the simpler generative model. In Section~\ref{sec_VRTD}, we consider the more challenging Markovian noise setting. Here, we would like to highlight that the main difficulty of estimating the differential Q-function in the Markovian setting stems from the possible lack of exploration. To be more precise, when $\pi(a|s)$ for some state-action pairs are very close to $0$, estimating the corresponding $\bar Q^\pi(s,a)$ is very challenging, due to the lack of corresponding samples. Therefore, in Section~\ref{sec_random_policy}, we start with policy evaluation of sufficiently random policies and propose the VRTD method (Algorithm~\ref{alg:VRTD}) along with its convergence guarantees. Then, in Section~\ref{sec:insufficient_random}, we design an EVRTD method on top of VRTD to remedy the issue of exploration for insufficiently random policies.}


\subsection{Multiple trajectory method for generative model}\label{sec:generative}

Under a generative model, we assume that the AMDP is mixing (Assumption~\ref{assump_mixing}). We propose the following multiple trajectory method (Algorithm~\ref{alg:Mult}), which estimates both the average reward/cost $\rho^\pi$ and the basic differential Q-function $\bar  Q^\pi(s,a)$ for all the state-action pairs $(s,a)\in \mathcal{S} \times \mathcal{A}$. 
\begin{algorithm}[h] \caption{Multiple-trajectory method for generative model}  
	\label{alg:Mult}   
	\begin{algorithmic}[1]
		\STATE{\textbf{Input}: $T, T', N, N' \in \mathbb{Z}_+$ and a feasible policy $\pi$.}
		\FOR{$ i = 1, \ldots, N$}
		\STATE{ Generate a trajectory with length $T+1$ following policy $\pi$ starting from an arbitrary state $s'\in \calS$, denoted by
			$
			\{(s_0 = s', a_0),(s_1, a_1),...,(s_T,a_T)\}
			$. Let $\rho^i = c(s_T, a_T)$.}
		\ENDFOR
		\STATE{Calculate $\widehat \rho^{\pi} = \frac{1}{N}\sum_{i=1}^N \rho^i$.}
		\FOR{$ j = 1, \ldots, N'$}
		\FOR{$(s,a)\in \calS\times \calA$}
		\STATE{Generate a trajectory with length $T'+1$ following policy $\pi$ starting from $(s,a)$, denoted by
			$
			\{(s_0^j = s, a_0^j = a),(s_1^j, a_1^j),...,(s_{T'}^j,a_{T'}^j)\}. 
			$ Let $\bar  Q^j(s,a) = \sum_{t=0}^{T'} \big(c(s_t^j,a_t^j)- \widehat \rho^{\pi}\big)$}
		\ENDFOR
		\ENDFOR
		\STATE{
			\textbf{Output}:
			$\widehat \rho^{\pi}$ and $\widehat  Q^{\pi} (s,a) = \frac{1}{N'} \sum_{j=1}^{N'}\bar  Q^j(s,a) $.
		}
	\end{algorithmic}
\end{algorithm}

\revision{In view of Algorithm~\ref{alg:Mult}, we generate trajectories starting from all state-action pairs to estimate the corresponding basic differential Q-function, which is allowed in the generative model. One should also note that we generate independent trajectories to estimate the average reward/cost and differential Q-function to avoid additional bias.}{}
In Proposition~\ref{lemma_bias_bound_generative}, we provide $\ell_\infty$-convergence guarantees on the bias of the estimators $\widehat \rho^{\pi}$ and $\widehat  Q^{\pi}$ generated by the multiple-trajectory method, i.e., $|\bbe[\widehat \rho^{\pi}] - \rho^\pi|$ and $\|\bbe[\widehat  Q^{\pi}] - \bar  Q^\pi\|_\infty$, where the expectation is taken with respect to all the samples used in Algorithm~\ref{alg:Mult}. \revision{}{The proof of Proposition~\ref{lemma_bias_bound_generative} is postponed to Appendix \ref{proof_lemma_bias_bound_generative}.}


\begin{proposition}\label{lemma_bias_bound_generative}
	Fix a feasible policy $\pi$. Assume that the epoch length parameters $T, T' > \tmix +1$. Then the estimators $\widehat\rho^{\pi}$ and $\widehat  Q^{\pi}$ generated by Algorithm~\ref{alg:Mult} satisfy
	\begin{align}\label{bias_bound_generative_0}
		|\bbe[\widehat\rho^{\pi}] - \rho^\pi|\leq \bar  c\cdot (\tfrac{1}{2})^{\lfloor T/\tmix \rfloor},
	\end{align}
	and
	\begin{align}
		\|\bbe[\widehat  Q^{\pi}] - \bar  Q^\pi\|_\infty\leq \bar  c(T'+1) \cdot (\tfrac{1}{2})^{\lfloor T/\tmix \rfloor} + 2\bar  c\cdot\tmix \cdot (\tfrac{1}{2})^{\lfloor T'/\tmix \rfloor},\label{bias_bound_generative_1}
	\end{align}
 where $\bar c := \max_{s\in \calS, a \in \calA} |c(s,a)|.$
\end{proposition}    

In view of Proposition~\ref{lemma_bias_bound_generative}, the number of samples (state-action pairs) required by the multiple trajectory method to find a solution $\widehat Q^\pi \in \bbr^{|\mathcal{S}\times \mathcal{A}|}$ such that $\|\bbe[\widehat Q^\pi] - \bar  Q^\pi \|_\infty\leq \epsilon$ is bounded on the order
$\mathcal{O}\{|\calS||\calA|\tmix\cdot \log({\tmix}/{\epsilon})\}$. In Proposition~\ref{lemma_variance_bound_generative}, we establish the convergence guarantee of Algorithm~\ref{alg:Mult} in terms of the expected squared $\ell_\infty$-error\tli{, and the proof is postponed to Appendix~\ref{proof_lemma_variance_bound_generative}.}
\begin{proposition}\label{lemma_variance_bound_generative}
	Fix a feasible policy $\pi$. Assume that the parameters $T, T' > \tmix +1$ and $N\geq N'$. Then the estimators $\widehat  \rho^{\pi}$ and $\widehat  Q^{\pi}$ generated by Algorithm~\ref{alg:Mult} satisfy that
	\begin{align*}
		\bbe|\widehat \rho^{\pi} - \rho^\pi|^2\leq \tfrac{4\bar c^2}{N} + \bar c^2\cdot (\tfrac{1}{2})^{2\lfloor T/\tmix \rfloor},
	\end{align*}
	and 
	\begin{align}\label{bias_bound_generative}
		\bbe[\|\widehat  Q^{\pi} - \bar  Q^\pi\|_\infty^2]\leq \tfrac{\upsilon\cdot   \bar c^2 (T'+1)^2}{N'}(\log(|\calS||\calA|)+1) + \bar c^2\big(3(T'+1)^2+12\tmix^2)\cdot (\tfrac{1}{2})^{2\lfloor T'/\tmix \rfloor},
	\end{align}
	where $\upsilon>0$ denotes a universal constant. 
\end{proposition}

Clearly, for a fixed policy $\pi$, the sample complexity to find an $\epsilon$-accurate estimator $\widehat Q^\pi \in \bbr^{|\mathcal{S}|\times |\mathcal{A}|}$ such that $\bbe[\|\widehat Q^\pi - \bar  Q^\pi \|_\infty^2]\leq \epsilon$ requires the parameters $N$ and $N'$ to be chosen in the order of $\widetilde{\order}(\tmix^2/\epsilon)$. The total number of samples used is $|\calS||\calA|N'T' + NT$, which is of the order \revision{$\widetilde{\order}(|\calS||\calA|\tmix^2/\epsilon)$}{$\widetilde{\order}(|\calS||\calA|\tmix^3/\epsilon)$}.

\subsubsection{Sample complexity of SPMD under generative model}
\revision{}{Below we discuss the sample complexity of using SPMD (Algorithm \ref{alg:SPMD}) and Multiple-trajectory method (Algorithm~\ref{alg:Mult}) to find an $\epsilon$-optimal solution, i.e., a solution $\widehat \pi$, such that $\bbe\left[ \rho(\widehat\pi) - \rho(\pi^*)\right] \leq \epsilon$, in the tabular setting under the generative model. 

We start with the unregularized AMDP,  see  Theorem~\ref{thm_unreg_1}.
Let $T_k, T_k', N_k, N_k'$ denote the parameters $T, T', N, N'$ in Algorithm~\ref{alg:Mult} for estimating $\stochQ{k}$ respectively.
\begin{corollary}[Sample complexity of SPMD for unregularized AMDPs]\label{coro_comp_generative}
	Assume Assumption~\ref{assump_mixing} holds and we are given a generative model. Consider using SPMD (Algorithm \ref{alg:SPMD}) and the Multiple-trajectory method (Algorithm~\ref{alg:Mult}) to solve the unregularized AMDP. 
 Consider the SPMD stepsizes in Theorem~\ref{thm_unreg_1}.
 If $T_k'\geq T' = \tmix\cdot \log_2\big(\tfrac{4\tmix}{\epsilon}\big)$, $T_k \geq T = \tmix\cdot \log_2\left(\frac{2(T'+1)}{\epsilon}\right)+1$ and $N_k=N=N_k'=N'=1$, then an $\epsilon$-optimal solution can be found in at most $\order(\tmix^2 \log|\calA|/\epsilon^2)$ SPMD iterations and the total number of state-action transition samples can be bounded by
	$$
	\widetilde\order\left( \tfrac{|\calS||\calA| \tmix^3}{\epsilon^2} \right).
	$$
\end{corollary}
\proof{Proof.}
	Notice that we can upper bound $\kappa$ by $\kappa\leq \bar c \cdot T' = \widetilde{\order}(\tmix)$, where $c(s,a)\leq \bar c$. Then the total number of iterations $K$ is bounded by $\widetilde{\order}(\tmix^2/\epsilon^2)$. Moreover, for each SPMD iteration, we require $N_kT_k+|\calS||\calA|N'_kT'_k = \widetilde \order(|\calS||\calA|\tmix)$ samples for policy evaluation. Therefore, the total number of samples can be bounded by
	$
	\widetilde\order\big( \tfrac{|\calS||\calA| \tmix^3}{\epsilon^2} \big),
	$
	as desired.
\endproof
To the best of our knowledge, the result in Corollary~\ref{coro_comp_generative} is the first sample complexity result for using policy gradient/actor-critic methods to solve AMDPs. This complexity result matches the lower bound $\Omega(\tfrac{\tmix|\calS||\calA|}{\epsilon^2})$ proved in \cite{jin2021towards} in terms of the dependence on $\epsilon$ as well as on $|\calS||\calA|$ up to logarithmic factors. 
The dependence on $\tmix$ appears to be worse than the state-of-the-art reduction-based approach \citep{wang2023optimal} where \tli{a nearly-optimal sample complexity $\widetilde{\order}(\tfrac{\tmix|\calS||\calA|}{\epsilon^2})$ is achieved. In \citet{wang2023optimal}, the AMDP problem is reduced to solving a DMDP with a discount factor $\gamma = 1 - \tfrac{\epsilon}{9\tmix}$ up to $\tfrac{\epsilon}{3(1-\gamma)}$-accuracy, and this DMDP is solved by a ``model-based'' approach with complexity $\widetilde{\mathcal{O}}(\tfrac{\tmix |\calS||\calA|}{(1-\gamma)^2\epsilon^2})$.} \tli{However, this performance gap between ``model-based'' approach and ``model-free'' approach also appears in the DMDP literature. Up to our knowledge, the state-of-the-art model-free methods for DMDPs only achieve a sample complexity of $\widetilde{\mathcal{O}}(\tfrac{|\calS||\calA|}{(1-\gamma)^5\epsilon^2})$; see, e.g., \cite{lan2021policy}.\footnote{\textcolor{black}{Notice that value-based methods like \cite{wainwright2019variance} achieve $\widetilde{\mathcal{O}}(\tfrac{|\calS||\calA|}{(1-\gamma)^3\epsilon^2})$ complexity in approximating the optimal $Q$ function. However, to further generate a greedy policy with an $\epsilon$-optimal value function, an additional price of $(1-\gamma)^{-2}$ needs to be paid and the final complexity is $\widetilde{\mathcal{O}}(\tfrac{|\calS||\calA|}{(1-\gamma)^5\epsilon^2})$.}} Then by combining the same reduction argument with the model-free DMDP solver, one can only obtain a sample complexity of $\mathcal{O}(\tfrac{\tmix^3|\calS||\calA|}{\epsilon^5})$ for solving AMDPs, which is strictly worse than the sample complexity proposed in Corollary~\ref{coro_comp_generative} achieved by our SPMD method. 
}
In spite of slightly worse sample complexity in the tabular setting, major advantages of policy gradient-type methods over model-based approaches 
exist in their flexibility in accommodating function approximation and on-policy Markovian noise model, which will be the main focus of the next section.


Next, we establish the sample complexity for AMDPs with strongly convex and Lipschitz continuous regularizers. 
\begin{corollary}[Sample complexity of SPMD for strongly convex regularized AMDPs]\label{coro_comp_generative_reg}
	Suppose Assumption~\ref{assump_mixing} holds and we are given a generative model. Consider using SPMD (Algorithm \ref{alg:SPMD}) with stepsizes in Theorem~\ref{thm_reg_1} and Multiple-trajectory method (Algorithm~\ref{alg:Mult}) to solve the mixing AMDP with a regularizer $h^\pi$ is $L$-Lipschitz continuous and $\omega$-strongly convex.
 If $T_k'\geq T' = \tmix\cdot \log_2\big(\tfrac{4\tmix}{\epsilon}\big)$, $T_k \geq T = \tmix\cdot \log_2\left(\frac{2(T'+1)}{\epsilon}\right)+1$ and $N_k=N=N_k'=N'=1$, then an $\epsilon$-optimal solution can be found in at most $\order(\sqrt{\tfrac{\omega\log|\calA|}{\epsilon}}+\tfrac{\tmix^2 + L^2 }{\omega \epsilon})$ SPMD iterations and the total number of state-action transition samples can be bounded by
	$$
	\widetilde\order\big(\tfrac{|\calS||\calA| \tmix \sqrt{\omega}}{\sqrt{\epsilon}}+\tfrac{(\tmix^3 + \tmix L^2)|\calS||\calA| }{\omega \epsilon} \big).
	$$
\end{corollary}
The proof of the above corollary follows from a similar argument to the proof of Corollary~\ref{coro_comp_generative}. 
}

\subsection{Policy evaluation in Markovian noise model}\label{sec_VRTD}
In the Markovian noise model, we assume that Assumption~\ref{assump_rho_0} (ergodicity assumption) holds to ensure sample access to all the states.
As a consequence, we have the following lemma which characterizes the mixing time property under Assumption~\ref{assump_rho_0} (See \cite{bertsekas1995neuro}). 
\begin{lemma}\label{lemma_assump_rho}
	There exist constants $\skipcon>0$ and $\mixrate\in(0,1)$ such that for any policy $\pi$, we have $\nu^\pi > 0$ and ~
	$
		\max_{s\in S} \|(P_\pi)^t_{(s,\cdot)}-\nu^\pi\|_{1}\leq \skipcon \cdot \mixrate^t, ~ \text{ for all } \; t\in \mathbb{Z}_+.
	$
\end{lemma}

Assumption~\ref{assump_rho_0} and Lemma~\ref{lemma_assump_rho} guarantee that the mixing time defined in Assumption~\ref{assump_mixing} is bounded as $\tmix \leq \order\left(\tfrac{\log(\skipcon)}{\log(1/\mixrate)}\right).$
Under this assumption, we aim to solve the following Bellman equation for a fixed policy $\pi$:
\begin{align}\label{bellman_eq}
	Q = P^\pi Q + c - \rho^\pi \mathbf{1}, \quad Q \in \bbr^{|\calS||\calA|}.
\end{align}

\revision{}{
It is worth noting \emph{the issue of lacking exploration in the action space} when solving the above Bellman equation. We need access to samples of all state-action pairs in $\mathcal{S} \times \mathcal{A}$ to compute an accurate solution of \eqref{bellman_eq}. Clearly, under the ergodicity assumption (Assumption~\ref{assump_rho_0}), we have enough explorations over the state space for all feasible policies. However, when the policy $\pi$ is not a sufficiently random policy (common in cases where the optimal policy is a greedy policy), then the action space may not be sufficiently explored. More precisely, when there exists $s\in \calS, a \in \calA$ such that $\pi(a|s)$ is very close or equal to $0$, we do not have enough samples for this state-action pair, posing challenges in learning the corresponding differential Q-function. 
}

\revision{}{To address the aforementioned issue, we will first focus on the simple case where the policy is sufficiently random (Section~\ref{sec_random_policy}). Then in Section~\ref{sec:insufficient_random}, we propose a novel exploration framework to remedy the lack of exploration issue encountered when the policy is not sufficiently random. Notice that our new exploration framework guarantees the same optimal sample complexity in $\epsilon$ for both sufficiently random and insufficiently random policies, which is achieved for the first time in both AMDP and DMDP literature.}

\subsubsection{Evaluation of sufficiently random policy}\label{sec_random_policy}
\revision{
When analyzing the sampling scheme of the Markovian noise model, we define a natural weighting diagonal matrix
$$
		D^\pi := \diag\big([D^\pi_1; ...; D^\pi_{|\calS|}]\big)\in \bbr^{(|\calS|\times|\calA|)\times(|\calS|\times|\calA|)}, 
	 $$
	where $D^\pi_s := \nu^\pi(s) \cdot \diag(\pi(\cdot|s)) \in \bbr^{|\calA|\times|\calA|}$.
Clearly, the diagonal elements of $D^\pi$ are the steady-state distribution of all state-action pairs induced by the policy $\pi$. For algorithmic design and analysis, \revision{we desire the matrix $D^\pi$ to be positive definite}{a desired $D^\pi$ should be positive definite}. However, the matrix $D^\pi$ can be ill-conditioned or even not positive definite when there exists $s\in\calS, a\in\calA$ such that $\pi(a|s)$ is close or equal to $0$. Towards this end, we first restrict our attention to sufficiently random policy $\pi$, and defer the discussion of insufficiently random policy $\pi$ to Section~\ref{sec:insufficient_random}. 
Specifically, we assume there exists some $\underline \pi>0$ such that $\pi(s|a) \geq \underline \pi$ for all $s\in \calS, a \in \calA$, thus $D^\pi\succ 0$.
For brevity of notation, we let $P := P^\pi, D:=D^\pi, \nu := \nu^\pi$ and $\rho^* := \rho^\pi$. 
For large state and action spaces, it is common to seek a low-dimensional approximation to the differential Q-function by implementing linear function approximation. In particular, one chooses $\mathbb{S}:=\text{span}\{\psi_1, ..., \psi_d\}$ for $d$ linearly independent basis vectors $\psi_1, ..., \psi_d \in \bbr^{|\calS|\times |\calA|}$. For each state action pair $s\in \calS, a \in \calA$, we let $\psi(s,a): = [\psi_1(s,a),\psi_2(s,a),...,\psi_d(s,a)]^\top$ denote its feature vector. Additionally, we assume that $\|\psi(s,a)\|_2\leq 1$ for each state action pair, which can be ensured through feature normalization. Letting $\Pi_{\mathbb{S}, D}$ denote the projection onto the subspace $\mathbb{S}$ with respect to the $\|\cdot\|_D$-norm, then one can instead solve the following projected Bellman equation:
\begin{align*}
	\bar  Q = \Pi_{\calS, D} (P \bar  Q +c -\rho^* \mathbf{1} ).
\end{align*}
}{
In this subsection, we consider the simple case such that
\begin{align}\label{assump_random_policy}
There~ exists~ a ~\underline \pi>0~ such ~that~ \pi(s|a) \geq \underline \pi ~for ~all~ s\in \calS, a \in \calA.
\end{align}
When analyzing the sampling scheme of the Markovian noise model, we define a natural weighting diagonal matrix
$$
		D^\pi := \diag\big([D^\pi_1; ...; D^\pi_{|\calS|}]\big)\in \bbr^{(|\calS||\calA|)\times(|\calS||\calA|)}, 
	 $$
	where $D^\pi_s := \nu^\pi(s) \cdot \diag(\pi(\cdot|s)) \in \bbr^{|\calA|\times|\calA|}$.
 Clearly, the diagonal elements of $D^\pi$ are the steady-state distribution of all state-action pairs in the Markov chain induced by the policy $\pi$, and we have $D^\pi \succ 0$ due to Assumption~\ref{assump_rho_0} and the condition $\pi(s|a) \geq \underline \pi$. For brevity of notation, in the remaining part of Section~\ref{sec_random_policy}, we let $P := P^\pi, D:=D^\pi, \nu := \nu^\pi$ and $\rho^* := \rho^\pi$. 
 
 To deal with large state and action spaces, we recall the linear function approximation introduced in Section~\ref{subsec_linear_function_appox}. Specifically, we choose $\mathbb{S}:=\text{span}\{\psi_1, ..., \psi_d\}$ for $d$ linearly independent basis vectors $\psi_1, ..., \psi_d \in \bbr^{|\calS|\times |\calA|}$. For each state action pair $s\in \calS, a \in \calA$, we let $\psi(s,a): = [\psi_1(s,a),\psi_2(s,a),...,\psi_d(s,a)]^\top$ denote its feature vector. Additionally, we assume that $\|\psi(s,a)\|_2\leq 1$ for each state-action pair, which can be ensured through feature normalization. Let $\Pi_{\mathbb{S}, D}$ denote the projection onto the subspace $\mathbb{S}$ with respect to the $\|\cdot\|_D$-norm. Then instead of solving Eq.~\eqref{bellman_eq}, we consider solving the following projected Bellman equation:
\begin{align*}
	\bar  Q = \Pi_{\calS, D} (P \bar  Q +c -\rho^* \mathbf{1} ).
\end{align*}}
This is equivalent to  find $\theta^* \in \mathbb{R}^d$ such that
\begin{align}\label{projected_bellman}
	\Psi^\top D \Psi \theta^* = \Psi^\top D P \Psi \theta^* + \Psi^\top D (c - \rho^* \mathbf{1}),
\end{align}
where $\Psi := [\psi_1, ..., \psi_d]$.
Next, we define a set of vectors in $\bbr^{|\calS|\times |\calA|}$ that have identical value on the support of $\pi$, i.e.,
\begin{align*}
	\mathcal{I} := \{x\in \bbr^{|\calS|\times |\calA|}|x(s,a) = 1~\text{if}~\pi(a|s)>0\}.
\end{align*}
For the sake of simplicity, we assume that $\mathcal{I} \cap\mathbb{S} = \phi$, which is equivalent to $\ones \notin \mathbb{S}$ in this subsection since $\pi$ is strictly positive. This condition can be easily verified by choosing the features appropriately. Then by Theorem 1 of \cite{tsitsiklis1999average},  Eq. \eqref{projected_bellman} has a unique solution. Furthermore, we define
\begin{align}\label{def_Qsol}
	\Qsol = \bar  Q^\pi+\widehat b\ones\quad\text{where}\quad\widehat b:=\arg\min_{b\in\bbr}\{\|\bar  Q^\pi+b\ones - \Psi \theta^*\|_D\}.
\end{align}
\textcolor{black}{In words, $\Qsol$ is the projection of $\Psi \theta^*$ onto the set of the differential Q-functions defined in Eq.~\eqref{bellman_0}.}

It is convenient in the analysis to have access to an orthonormal basis spanning the projected space $\mathbb{S}$. Define the matrix $B\in \R^{d\times d}$ as $B_{i,j}:= \langle \psi_i, \psi_j\rangle_D $ for each $i,j\index{[d]}$. \textcolor{black}{Note that $B\succ 0$ due to the linear independence of $\psi_i$'s and $D \succ 0$.}
Let
$\Phi:= [\phi_1 ,  \phi_2  , ..., \phi_d ] = \Psi B^{-\frac{1}{2}}.$
By construction, the vectors $\phi_1, \ldots, \phi_d$ satisfy $\langle \phi_i, \phi_j \rangle_{D} = \mathbb{I}(i = j)$. We further denote
%
\begin{align}\label{def_mu_M}
	\mu:=\lambda_{\min}(B)\quad \text{and} \quad	\tli{M:=\Phi^\top D P \Phi}.
\end{align}
Notice that $M$ can be viewed as a $d$-dimensional matrix that describes the action of $P$ on the projected space $\mathbb{S}$. 

To solve the projected Bellman equation~\eqref{projected_bellman_mod}, we define the deterministic operator and the corresponding stochastic operator calculated from sample $\xi=\{(s,a), (s',a'), c(s,a)\}$ as
\begin{align}
	&g(\theta, \rho) = \Psi^\top  D(\Psi \theta - P \Psi \theta - c + \rho \bfone),\label{deterministic_opt}\\
	&\widetilde g(\theta, \rho, \xi)= \left(\langle \psi(s,a)-\psi(s',a'), \theta \rangle - c(s,a) + \rho \right)\psi(s,a). \label{stochastic_opt}
\end{align}
The following lemma characterizes the strong monotonicity property of the operator $g(\theta, \rho)$.
\begin{lemma}\label{monotone_constant}
	Under Assumption~\ref{assump_rho_0}, assume $\pi(a|s)>0$ for all $s\in \calS, a \in \calA$,  then
	\begin{align}\label{def_monotonecon}
		\monotonecon := \inf_{x \in \mathbb{S}, \|x\|_D=1} x^\top D (I-P)x > 0,
	\end{align}
	As a consequence, for any $\theta, \theta' \in \bbr^d$ and $\rho\in \bbr$,
	\begin{align}
		\langle g(\theta, \rho) - g(\theta', \rho), \theta - \theta' \rangle \geq (\monotonecon)\|\Psi \theta - \Psi \theta'\|_D^2.
	\end{align}
\end{lemma}
\proof{Proof.}
	For any $x\in \mathbb{S}, \|x\|_D=1$, let the expectation be taken on $s \sim \nu, a \sim \pi(\cdot|s), s' \sim \mathsf{P}(\cdot|s,a), a' \sim\pi(\cdot|s')$, then
	\begin{align}\label{proof_monotone_step_1}
		x^\top D (I-P)x = \bbe[x(s,a)^2] - \bbe[x(s,a)x(s',a')]\overset{(i)}= \tfrac{1}{2} \bbe[\big(x(s,a) - x(s',a')\big)^2]\overset{(ii)}>0,
	\end{align}
	where step (i) follows from the fact that $(s,a)$ and $(s',a')$ have the same marginal distribution, and step (ii) follows from the fact that $\mathbf{1}\notin \mathbb{S}$ and the Markov chain is irreducible. Meanwhile, given that this feasible region is compact and the function of interest is continuous, by the extreme value theorem, we obtain that $\monotonecon>0$, which completes the proof.
\endproof

\vgap

\revision{}{
 It should be noted that the term $({\monotonecon})^{-1}$, which replaces the role of the effective horizon $(1-\gamma)^{-1}$ in the discounted MDPs, is strongly related to the mixing time of the induced Markov chain. It can be shown that if the linear subspace $\mathbb{S}$ is properly chosen, i.e., $\langle \phi_i, \ones \rangle_D$ is not extremely close to 1 for all $i\in[d]$, then $({\monotonecon})^{-1} = \order(\tmix)$ (see Corollary 1.14 of \cite{montenegro2006mathematical}).   
}

\begin{algorithm}[h] \caption{Variance-reduced Temporal Difference (VRTD) Method for AMDPs}  
	\label{alg:VRTD}   
	\begin{algorithmic}[1]
		\STATE{\textbf{Input}: $  \widehat \theta_0 \in \R^d$, $\eta > 0$, $\tau, \tau' \in \mathbb{Z}_+$,  $\{N_k\}_{k=1}^{K}\subset \mathbb{Z}_+$ and $\{N_k'\}_{k=1}^{K} \subset \mathbb{Z}_+$.}
		\FOR{$ k = 1, \ldots, K$}
		\STATE{Set $\theta_1 = \widetilde \theta  = \widehat \theta_{k-1}$.}
            \STATE{Collect $N_k'$ state-action samples $\{(\widetilde s_i,\widetilde a_i)\}_{i=1}^{N_k'}$ from the single Markov trajectory induced by \quad\quad\qquad policy $\pi$, where each sample is the last one of $\tau'+1$ successive state-action samples.\\ Calculate $\widetilde \rho = \frac{1}{N_k'}\sum_{i=1}^{N_k'} c(\widetilde s_i, \widetilde a_i)$.}
		\STATE{Collect $N_k$ samples $\xi_i^k(\tau')=\{(s_i,a_i), (s_i',a_i'), c(s_i,a_i)\}$ from the single Markov trajectory  induced by policy $\pi$, where each sample is the last one of $\tau'+1$ successive samples.\\ Calculate $\widehat g(\widetilde\theta, \widetilde \rho) = \tfrac{1}{N_k} \tsum_{i=1}^{N_k} \widetilde g(\widetilde \theta, \widetilde \rho, \xi_i^k(\tau')).$}
		\FOR{$t=1, \ldots, T$}
		\STATE{ Collect a sample $\xi_{t}(\tau)=\{(s_t,a_t), (s_t',a_t'), c(s_t,a_t)\}$ from the single Markov trajectory, which is the last one of $\tau+1$ successive samples.}
		\STATE{
			\beq \label{TD_iid_step}
			\theta_{t+1} = \theta_t - \eta \left(\widetilde g(\theta_t,\widetilde \rho, \xi_t(\tau))-\widetilde g(\widetilde \theta, \widetilde \rho,\xi_t(\tau)) + \widehat g(\widetilde \theta,\widetilde \rho)\right).
			\eeq}
		\ENDFOR
		\STATE{Output of the epoch: \beq \label{eq:VRTD-ave-output}
			\widehat \theta_k = \tfrac{\sum_{t=1}^T \theta_t}{T},\quad\text{and}\quad \widehat \rho_k = \widetilde \rho.\eeq}
		\ENDFOR
	\end{algorithmic}
\end{algorithm}

We are now prepared to introduce the basic ideas of the Variance-reduced Temporal Difference (VRTD) Method   (Algorithm~\ref{alg:VRTD}) for solving \eqref{projected_bellman_mod}. First, it incorporates the idea of the variance-reduced TD learning for solving DMDPs from \cite{li2021accelerated}. This algorithm runs in epochs and utilizes the current best estimator $\widetilde \theta$ to recenter the updates inside each epoch, allowing us to capture the ``correct'' stochastic error of the problem. 
Second, the data generated by the Markovian noise model is highly correlated, which induces challenges in the algorithmic analysis. To overcome the difficulties, we incorporate the idea of sample skipping from \cite{kotsalis2020simple2, li2021accelerated} to reduce the bias induced by correlated data. The spirit of this sample skipping idea can be traced back to the classical work by \cite{yu1994rates}.
Third, in contrast to the DMDP environment, policy evaluation within AMDPs entails two objectives—$\rho^*$ for average reward/cost and $\theta^*$ for the parameterized differential Q-function. It is worth noting that estimating $\rho^*$ demands significantly fewer samples compared to estimating $\theta^*$. This distinction arises from the nature of the tasks. Estimating $\rho^*$ involves assessing a policy's steady-state behavior, whereas estimating $\theta^*$ necessitates acquiring a deeper understanding of the transition kernel.
\revision{}{In view of this observation, unlike the existing TD approaches that updates both objectives simultaneously \citep{tsitsiklis1999average, zhang2021finite}, our VRTD algorithm first collects samples to estimate $\rho^*$ at the beginning of each epoch and then performs (variance-reduced) stochastic approximation to update $\theta$. Consequently, we are able to significantly improve the sample complexity of the TD algorithms introduced by \cite{tsitsiklis1999average} and analyzed by \cite{zhang2021finite}. 
}

The following theorem characterizes the convergence of VRTD in terms of expected squared $\ell_D$-norm error, which is a natural convergence criterion (see, e.g., \cite{bhandari2018finite,mou2020optimal,li2021accelerated} for DMDPs) for policy evaluation in the RL literature. The proof of this result can be found in  Appendix~\ref{proof_thm_VRTD}.

\begin{theorem}\label{thm_VRTD}
	Fix the total number of epochs $K$ and positive integers $N$ and $N'$. Assume for each epoch $k\in [K]$, the parameters $\eta$, $T$, $N_k$ and $N'_k$ satisfy 
	\beq\label{stepsize_1}
	\eta\leq \tfrac{\monotonecon}{845},~~T\geq \tfrac{64}{\mu(\monotonecon)\eta},~~N_k\geq\max\big\{\tfrac{1160}{\mu(\monotonecon)^2}, (\tfrac{3}{4})^{K-k} N\big\},~~\text{and}~~N_k'\geq(\tfrac{3}{4})^{K-k} N'.
	\eeq
	Set the output of epoch $k$ to be $\widehat \theta_k:=\frac{\sum_{t=1}^T \theta_t}{T  }$. If $\tau$ and $\tau'$ are chosen to satisfy 
	\begin{align}\label{tau_tau_pp}
		\mixrate^{\tau}\leq \tfrac{\sqrt{\mu}\eta}{\mixcon}\quad \text{and} \quad \mixrate^{\tau'} \leq \min\{\tfrac{\min_{s\in \calS} \nu(s)}{\skipcon}, \tfrac{1}{4 \Cmax+1}\},
	\end{align}
	then we have 
	$
	\bbe[(\widehat \rho_K - \rhosol)^2] \leq \tfrac{5\cbound^2}{N'},
	$
	and
	\begin{align}\label{conclusion_thm_4}
		\bbe\|\Psi \widehat \theta_K - \Psi \thetasol\|_D^2 \leq \tfrac{1}{2^K}\|\Psi \theta^0 - \Psi \thetasol\|^2_D  + \tfrac{3W_1}{N}+\tfrac{108\cbound^2}{N'(\monotonecon)^2\mu}+\tfrac{3\cbound^2}{4N'},
	\end{align}
where $W_1:= 22\trace\{(I_d-M)^{-1}\bar  \Sigma(I_d-M)^{-\top}\} +  \tfrac{4}{\mu(\monotonecon)^2}\bbe[\|(\Qsol - \Psi\theta^*) \|_D^2]$ and 
	$\bar \Sigma := \cov\big[B^{-\frac{1}{2}}
	\big(\langle\psi(s,a) -  \psi(s',a'),\theta^*\rangle - c(s, a)\big) \psi(s,a)\big],$
	for $s\sim \nu, a\sim \pi(\cdot|s), s'\sim \mathsf{P}(\cdot|s,a), a'\sim\pi(\cdot|s').$
\end{theorem}



Noting that the error caused by inexact feature approximation is unavoidable, we set the stopping criterion of VRTD as finding a solution $\widehat \theta\in \bbr^d$ to achieve the accuracy $\bbe\|\Psi \widehat \theta - \Psi \theta^*\|_D^2 \leq \order(1)\|\Psi \theta^*- \Qsol\|_D^2+\epsilon$. From the statement of Theorem~\ref{thm_VRTD}, we conclude that the number of required epochs is bounded by $\order\{\log(\|\Psi \theta^0 - \Psi \thetasol\|^2_D/\epsilon)\}$ and the total number of samples is $\sum_{k=1}^K \big(\tau T + \tau'(N_k+N_k')\big)$, which is of the order 
\begin{align*}
	\order\big\{ \tfrac{\tmix}{(\monotonecon)^2\mu}\log(\tfrac{1}{\epsilon})  + \tfrac{\tmix\cbound^2}{(\monotonecon)^2\mu\epsilon} + \underbrace{\tfrac{\tmix\cdot\trace\{(I_d-M)^{-1}\bar{\Sigma}(I_d-M)^{-\top}\} }{\epsilon}}_{\text{dominant stochastic error}} \big\}.
\end{align*}

We would like to provide some comments on the dominant term in the complexity discussed above. This term represents an instance-dependent characterization of the stochastic error associated with solving the projected Bellman equation \eqref{projected_bellman} while having access to samples from the transition kernel. A direct extension of Proposition 1 of \cite{li2021accelerated} demonstrates that this term matches the asymptotic instance-dependent lower bound on the stochastic error for solving Eq. \eqref{projected_bellman} with i.i.d. samples up to an additional mixing time factor. For fair comparison to the existing literature, we apply a direct upper bound $\order\big\{\frac{\|\thetasol\|_2^2}{(\monotonecon)^2\mu}\big\}$
on $\trace\{(I_d-M)^{-1}\bar{\Sigma}(I_d-M)^{-\top}\}$, as well as the upper bound $\order(\tmix)$ on $(\monotonecon)^{-1}$, thus arriving at the sample complexity
$
	\order\big\{ \tfrac{\tmix^3}{\mu}\log(\tfrac{1}{\epsilon})  + \tfrac{\tmix^3(\cbound^2+\|\theta^*\|_2^2)}{\mu\epsilon} \big\}.
$

To the best of our knowledge, the only work that had established finite sample complexity for solving the AMDP policy evaluation problem under Markovian noise is \citet{zhang2021finite}. They analyzed the TD($\lambda$) method proposed in \cite{tsitsiklis1999average}, where the algorithm simultaneously updates estimators of both objectives, i.e., average reward/cost $\rho^\pi$ and differential value function $\bar V^\pi$. Compared to their guarantees, our sample complexity outperforms by a factor of $\order((\monotonecon)^{-2})$ or equivalently $\order(\tmix^2)$. In addition, their  convergence analysis was conducted in $\ell_2$-norm, which is not a natural metric for the underlying problem, thus leading to worse dependences on other problem parameters, e.g., dimension of the transition kernel. It is noteworthy that recently \citet{mou2022optimal} proposed a variance-reduced stochastic approximation approach that solves the AMDP policy evaluation problem in span semi-norm under the generative model. However, their results do not directly extend to the Markovian noise setting. Meanwhile, it is hard to obtain convergence guarantees in terms of the ``bias'' by using their metric and analysis techniques. 

\vgap

Next, we establish a bias bound in Theorem~\ref{thm_VRTD_bias},
whose proof is deferred to Appendix~\ref{proof_thm_VRTD_bias}.
\begin{theorem}\label{thm_VRTD_bias}
	Fix the total number of epochs $K$ and positive integers $N$ and $N'$. Assume for each epoch $k\in [K]$, the parameters $\eta$, $N$, $N'$ and $T$ satisfy conditions~\eqref{stepsize_1} and 
	\beq\label{stepsize_thm_2}
	N_k= N\geq\tfrac{1160}{\mu(\monotonecon)^2},~~N_k'= N'\geq \tfrac{1}{\mu(\monotonecon)^2} .
	\eeq
	Set the output of epoch $k$ to be $\widehat \theta_k:=\frac{\sum_{t=1}^T \theta_t}{T}$. Then we have
	\begin{align}\label{result_VRTD_bias}
		\|\Psi \bbe[\widehat \theta_K] - \Psi \theta^*\|_D^2 \leq \tfrac{1}{2^K}\|\Psi \theta^0 - \Psi \thetasol\|_D^2 +\big( b_1 \Cmax^2 \mixrate^{2\tau'} + \tfrac{b_2\Cmax\mixrate^\tau}{(\monotonecon)\sqrt{\mu}}\big) \big(\|\Psi \theta^0 - \Psi \thetasol\|_D^2 +  \tfrac{W_1}{N} + \cbound^2\big),
	\end{align}
	where $b_1$ and $b_2$ are universal constants.
\end{theorem}

According to Theorem~\ref{thm_VRTD_bias}, in order to find a solution $\widehat \theta\in \bbr^d$ such that $\|\Psi  \bbe[\widehat\theta] - \Psi \theta^*\|_D^2 \leq \order(1)\|\Psi \theta^*- \Qsol\|_D^2+\epsilon$, we need to set the parameters $\tau$ and $\tau'$ to $\tau, \tau' \geq \order\{\tmix\log(1/\epsilon)\}$, and as a result, the total sample complexity can be bounded by
\begin{align}\label{complexity_bias_VRD+TD}
	\order\big\{ \tfrac{\tmix}{(\monotonecon)^2\mu}\log^2(\tfrac{1}{\epsilon})\big\}, \quad \text{or equivalently}\quad\order\big\{ \tfrac{\tmix^3}{\mu}\log^2(\tfrac{1}{\epsilon})\big\}.
\end{align}
It should be noted that this result naturally extends to policy evaluation of DMDPs under Markovian noise. 
To the best of our knowledge, this geometric convergence rate on the bias of policy evaluation is new to both AMDP and DMDP literature. \revision{}{We will utilize this new complexity bound in Section~\ref{sec_sample_complexity_Mark} to establish sample complexity bounds for SPMD in the Markovian model.}

\subsubsection{Evaluation of insufficiently random policies}\label{sec:insufficient_random}
In this section, we remove the assumption on sufficiently random policies, i.e., \eqref{assump_random_policy}. Precisely, we aim to evaluate a policy $\pi$ where $\pi(a|s)$ can be $0$ or arbitrarily small for any state-action pair. 

As mentioned before, a possible lack of exploration in this setting poses significant challenges for on-policy evaluation. A classical method to remedy this issue is through off-policy learning, where the agent chooses actions according to a behavior policy that encourages exploration. However, in this manner, the agent is not able to collect rewards according to the policy of interest, thus losing the advantages of on-policy reinforcement learning. On the other hand, one may also argue that we can directly perturb the policy to encourage more exploration. However, this is still an off-policy method, and the differential Q-function of the perturbed policy can be very different from that of the original policy.

In this section, we seek an on-policy method to remedy this exploration issue. More specifically, we propose an exploratory variance-reduced temporal difference (EVRTD) algorithm that alternates between the policy of interest $\pi$ and a perturbed policy, denoted as $\widetilde \pi$. 
The EVRTD algorithm for insufficiently random policy $\pi$ has a similar structure to Algorithm~\ref{alg:VRTD}, with differences only lying in the sample collection steps. In particular, the EVRTD algorithm replaces steps 5 and 7 in Algorithm~\ref{alg:VRTD} with the following ones. Here again, we set $P:=P^\pi$, $\nu:=\nu^\pi$ and $\rho^* := \rho^\pi$ for brevity of notation.

\vspace{0.1in}

\emph{Step $5'$}: Run the Markov trajectory induced by policy $\pi$ for $\tau'$ steps and obtain the current state $s$. Then generate an action $a\sim \widetilde \pi(\cdot|s)$, the next state $s'\sim \mathsf{P}(\cdot|s,a)$, the next action $a' \sim \pi(\cdot|s')$, and collect the tuple $\{(s,a), (s',a'), c(s,a)\}$. Repeat for $N_k$ times, and index the samples as  $\xi_i^k(\tau')=\{(s_i,a_i), (s_i',a_i'), c(s_i,a_i)\}$. 
Calculate $\widehat g(\widetilde\theta, \widetilde \rho) = \tfrac{1}{N_k} \tsum_{i=1}^{N_k} \widetilde g(\widetilde \theta, \widetilde \rho, \xi_i^k(\tau')).$

\emph{Step $7'$}:  Run the Markov trajectory induced by policy $\pi$ for $\tau$ steps and obtain the current state $s_t$. Then generate an action $a_t\sim \widetilde \pi(\cdot|s_t)$, the next state $s_t'\sim \mathsf{P}(\cdot|s_t,a_t)$, the next action $a_t' \sim \pi(\cdot|s_t')$, and collect the tuple $\{(s_t,a_t), (s_t',a_t'), c(s_t,a_t)\}$.

\vspace{0.1in}

Clearly, the EVRTD algorithm is still an on-policy algorithm since, in most time steps, the agent collects rewards/costs following the policy of interest $\pi$. We now explain some intrinsic ideas of the EVRTD algorithm by focusing on one sample tuple  $\{(s,a), (s',a'), c(s,a)\}$. By choosing a large enough $\tau$ or $\tau'$, the distribution of the state $s$ converges to the stationary distribution $\nu$ of the original Markov chain induced by policy $\pi$. Invoking that $a\sim \widetilde \pi(\cdot|s)$, $s'\sim \mathsf{P}(\cdot|s,a)$ and $a'\sim \pi(\cdot|s')$, the projected fixed point equation we are trying to solve is
\begin{align}\label{projected_bellman_mod}
	\Psi^\top \widetilde D \Psi \theta^* = \Psi^\top \widetilde D P \Psi \theta^* + \Psi^\top \widetilde D (c - \rho^* \mathbf{1}),
\end{align}
where $\widetilde D$ is defined as
$
	\widetilde D := \diag\big([\widetilde D_1; ...; \widetilde D_{|\calS|}]\big)\in \bbr^{(|\calS|\times|\calA|)\times(|\calS|\times|\calA|)}, 
$
and $\widetilde D_s := \nu^\pi(s) \cdot \diag(\widetilde \pi(\cdot|s))$. Then Eq.~\eqref{projected_bellman_mod} can be equivalently written as 
\begin{align*}
	\bar  Q = \Pi_{\mathbb{S}, \widetilde D} (P \bar  Q +c -\rho^* \mathbf{1} ).
\end{align*}
It is easy to see that instead of solving the Bellman equation induced by a perturbed policy, we are still solving a projected version of the average-reward Bellman equation~\eqref{bellman_eq} induced by the original policy $\pi$. \tli{This projected Bellman equation completely reduces to \eqref{bellman_eq} when $\exists b \in \bbr,~ \text{s.t.}~ \bar Q^\pi + b\ones \in \mathbb{S}$.} Moreover, we have $\widetilde D \succ 0$ when the underlying Markov chain is ergodic and the policy $\widetilde\pi(a|s) >0$ for all $(s, a)$ pairs. With a slight abuse of notation, in this subsection we define the matrix $B\in \R^{d\times d}$ by letting $B_{i,j}:= \langle \psi_i, \psi_j\rangle_{\widetilde D} $ for each $i,j\index{[d]}$, and let
\begin{align}\label{def_M_EVRTD}
\Phi:= [\phi_1 ,  \phi_2  , ..., \phi_d ] = \Psi B^{-\frac{1}{2}},\quad\mu:=\lambda_{\min}(B), \quad \widetilde M:=\Phi^\top \widetilde D P \Phi,
\end{align}
and 
$\widetilde \Sigma := \cov\big[B^{-\frac{1}{2}}
\big(\langle\psi(s,a) -  \psi(s',a'),\theta^*\rangle - c(s, a)\big) \psi(s,a)\big],$
for $s\sim \nu, a\sim \widetilde \pi(\cdot|s), s'\sim \mathsf{P}(\cdot|s,a) \text{ and } a'\sim\pi(\cdot|s')$.
We modify the definition of the deterministic operator correspondingly
\begin{align}\label{deterministic_opt_perturb}
	g(\theta, \rho) = \Psi^\top \widetilde D(\Psi \theta - P \Psi \theta - c + \rho \bfone).
\end{align}

Now switch our attention to the construction of the perturbed policy $\widetilde \pi$. We first fix a constant $0<\underline \pi<1 $. let $\calA_s:=\{a\in \calA| \pi(a|s) > {\underline \pi/2}\}$ \revision{}{denote the set of actions that does not lack exploration}, and the policy $\widetilde \pi$ is constructed as follows,
	\begin{align}\label{construct_tilde_pi}
		\forall s \in \calS, 
		~~\widetilde \pi(a|s) := \begin{cases} \underline \pi, & \text{if}~a \notin \calA_s\\\pi(a|s)\cdot \tfrac{1- (|\calA|-|\calA_s|)\underline \pi}{\sum_{a'\in \calA_s}\pi(a'|s)}, & \text{if}~a \in \calA_s
		\end{cases}.
	\end{align}
	We have that $\tfrac{1- (|\calA|-|\calA_s|)\underline \pi}{\sum_{a'\in \calA_s}\pi(a'|s)}=\tfrac{1- (|\calA|-|\calA_s|)\underline \pi}{1-\sum_{a'\notin \calA_s}\pi(a'|s)}\leq \tfrac{1- (|\calA|-|\calA_s|)\underline \pi}{1- (|\calA|-|\calA_s|)\underline \pi/2}\leq 1$. On the other hand, when $\underline \pi \leq \tfrac{1}{|\calA|- \min_{s\in\calS}|\calA_s|},$ we have $\tfrac{1- (|\calA|-|\calA_s|)\underline \pi}{\sum_{a'\in \calA_s}\pi(a'|s)}\geq 0$. Therefore, we ensure that $\widetilde \pi(\cdot|s)\in\Delta_{|\calA|}$ for all $s\in \calS$. It should be noted that the EVRTD method reduces to the VRTD method (Algorithm~\ref{alg:VRTD}) when $\calA_s=\calA$ for all $s\in\calS$.
\begin{lemma}\label{monotone_constant_perturb}
	Assume Assumption~\ref{assump_rho_0} holds. Recalling the definition of $\monotonecon$ in \eqref{def_monotonecon},  we have
	\begin{align}\label{lemma_monotone_before_perturb}
		\monotonecon > 0. 
	\end{align}
	Assume $|\calS|\geq 2$. When $\min_{s\in\calS}|\calA_s|\neq |\calA|$, if we take 
	\begin{align}\label{underline_pi}
		\underline \pi \leq  \tfrac{(\monotonecon)\min_{s\in\calS}\nu(s)}{(|\calA|- \min_{s\in\calS}|\calA_s|)(8+(\monotonecon)\min_{s\in\calS}\nu(s))},  
	\end{align}
	and construct the perturbed policy $\widetilde \pi$ as \eqref{construct_tilde_pi}, then
	\begin{align}\label{lemma_monotone_perturb}
		\inf_{x \in \mathbb{S}, \|x\|_{\widetilde D}=1} x^\top \widetilde D (I-P)x \geq \min\big\{\tfrac{27\alpha(\underline \pi)(\monotonecon)}{32}, \tfrac{1-\alpha(\underline \pi)(\monotonecon)}{8}\big\},
	\end{align}
	where $\alpha(\underline \pi) := 1- (|\calA| - \min_{s\in \calS}|\calA_s|)\underline \pi\geq \tfrac{8}{9}$. As a consequence, Eq.~\eqref{projected_bellman_mod} has a unique solution. Moreover, for any $\theta, \theta' \in \bbr^d$ and $\rho\in \bbr$,
	\begin{align}
		\langle g(\theta, \rho) - g(\theta', \rho), \theta - \theta' \rangle \geq \min\big\{\tfrac{27\alpha(\underline \pi)(\monotonecon)}{32}, \tfrac{1-\alpha(\underline \pi)(\monotonecon)}{8}\big\}\cdot\|\Psi \theta - \Psi \theta'\|_{\widetilde D}^2.
	\end{align}
\end{lemma}
\noindent See Appendix~\ref{proof_of_lemma_perturb} for the detailed proof of this lemma.

By Lemma~\ref{monotone_constant_perturb} above,  we conclude that Eq.~\eqref{projected_bellman_mod} has a unique solution, and that the strong monotonicity of the deterministic operator $g$ holds. Consequently, we obtain the convergence guarantees of EVRTD by proving the analogs of Theorem~\ref{thm_VRTD} and \ref{thm_VRTD_bias}. 
\begin{proposition}\label{prop_EVRTD}
 After replacing $\monotonecon$ with $\min\big\{\tfrac{27\alpha(\underline \pi)(\monotonecon)}{32}, \tfrac{1-\alpha(\underline \pi)(\monotonecon)}{8}\big\}$, $\ell_D$-norm with $\ell_{\widetilde D}$-norm, $M$ with $\widetilde M$, $\bar \Sigma$ with $\widetilde \Sigma$, and slightly modifying some universal constants, the conclusions of Theorem~\ref{thm_VRTD} and \ref{thm_VRTD_bias} also hold for the EVRTD algorithm.
\end{proposition}
In conclusion, we develop the EVRTD algorithm, where the agent follows the original policy $\pi$ in most of the time and follows the perturbed policy $\widetilde \pi$ only when collecting samples for stochastic approximation updates. 
The EVRTD algorithm recovers the convergence guarantees of VRTD in a weighted $\ell_2$-metric induced by the positive definite matrix $\widetilde D\succ \tfrac{\alpha(\underline \pi)\cdot\underline \pi}{2} I$. Notice that when our desired guarantee is in $\ell_\infty$-norm or $\ell_2$-norm, we need to pay for additional dependence on $1/\lambda_{\min}(\widetilde D)$.

\begin{remark}
	The design idea of the EVRTD algorithm naturally extends to the DMDP setting, by replacing the $\monotonecon$ with $1-\gamma$, where $\gamma$ is the discount factor. We provide proof of the analog of Lemma~\ref{monotone_constant_perturb} for DMDPs in Appendix~\ref{extension_to_DMDP}.
\end{remark}

\subsubsection{Sample complexities for SPMD in Markovian noise}\label{sec_sample_complexity_Mark} \revision{}{In this subsection, we provide the sample complexities for using the SPMD method with EVRTD as a critic to solve AMDPs in Markovian noise. We consider both unregularized AMDPs and regularized AMDPs.}


\revision{For unregularized AMDPs, notice that Theorem~\ref{thm_unichain_SPMD} still holds given that Assumption~\ref{assump_rho_0} is stronger than Assumption~\ref{assump_mixing}. We first implement the policy mirror descent in Theorem~\ref{thm_unichain_SPMD} with  EVRTD critic solver and establish the following sample complexity.
\begin{corollary}[Sample complexity of unregularized AMDPs]\label{coro_comp_Mark_1}
	Consider the mixing AMDPs which satisfy Assumption~\ref{assump_rho_0}. Suppose the number of iterations $K$ of Algorithm~\ref{alg:SPMD} is given a prior and $\lambda = \tfrac{\sqrt{2\log|\calA|}}{\kappa\sqrt{K}}$. Suppose that each policy $\pi_k$ is evaluated using the EVRTD algorithm with parameter selection in Theorem~\ref{thm_VRTD_bias} for $\log (1/\epsilon)$ epochs. Then a solution $\widehat \pi$, which satisfies that $\bbe[\rho(\widehat \pi)-\rho(\pi^*)]\leq \order(1)\eapprox+ \epsilon$, can be found in at most $\order(\kappa^2 \log|\calA|/\epsilon^2)$ SPMD iterations and the total number of state-action transition samples can be bounded by
	$
	\widetilde\order\big( \tfrac{\tmix^3\kappa^2}{\bar\mu\epsilon^2} \big).
	$
	Moreover, we have $\kappa^2\leq2\max_\pi\|\Psi \theta_\pi^*\|_\infty^2 + \order\big(\max_{\pi} \tfrac{\|\theta_\pi^*\|_2^2}{\lambda_{\min} (\widetilde D^{\pi})}\big)$.
\end{corollary}
\proof{Proof.}
	For each SPMD iteration, to ensure policy evaluation bias to satisfy $\varsigma_k\leq \epsilon+\order(1)\eapprox$, we require $\widetilde{\order}(\frac{\tmix^3}{\bar\mu})$ samples for policy evaluation by Theorem~\ref{thm_VRTD_bias}. Therefore, the total number of samples can be bounded by
	$
	\widetilde\order\big( \tfrac{\tmix^3\kappa^2}{\bar\mu\epsilon^2} \big).
	$
	Next, we proof an upper bound for $\kappa^2$. 
	For $k=0,...,K-1$, by Theorem~\ref{thm_VRTD} we can upper bound $\kappa_k^2$ by
	\begin{align*}
		\kappa_k^2 &\leq 2\|\Psi \theta_{\pi_k}^*\|_\infty^2+ 2\bbe\|\stochQ{k} - \Psi \theta_{\pi_k}^*\|_\infty^2 \leq 2\|\Psi \theta_{\pi_k}^*\|_\infty^2 + 2\bbe[\|\stochQ{k} - \Psi \theta_{\pi_k}^*\|_\infty^2/\lambda_{\min}(\widetilde D^{\pi_k})]\\
		&\leq 2\|\Psi \theta_{\pi_k}^*\|_\infty^2 + \order\big(\tfrac{\|\theta_{\pi_k}^*\|_2^2}{\lambda_{\min} (\widetilde D^{\pi_k})}\big)\leq 2\max_\pi\|\Psi \theta_\pi^*\|_\infty^2 + \order\big(\max_{\pi} \tfrac{\|\theta_\pi^*\|_2^2}{\lambda_{\min} (\widetilde D^{\pi})}\big),
	\end{align*}
	as desired.
\endproof}
{To establish the convergence guarantees, we require some additional notation. For a fixed policy $\pi$, we let $\bar\theta_\pi$ denote the solution of the projected Bellman equation~\eqref{projected_bellman_mod} induced by policy $\pi$.
Since the approximation error introduced by linear function approximation cannot be eliminated, we can only establish the convergence to a neighborhood of the optimal solution whose diameter is the approximation error defined as $\eapprox:=\max_{\pi}\spannorm{\Psi \bar\theta_\pi - \bar Q^{\pi}}$. Moreover, we let $\bar\mu$ be a lower bound of $\mu$ defined in \eqref{def_mu_M} across all feasible policies.

For unregularized AMDPs, we recall the convergence guarantee of SPMD in Theorem~\ref{thm_unreg_1} and establish the following sample complexity.


\begin{corollary}[Sample complexity of unregularized AMDPs]\label{coro_comp_Mark_1}
	Assume the Markovian noise setting and Assumption~\ref{assump_rho_0} holds. Consider the ergodic AMDPs with no regularizer. Suppose that the SPMD stepsize parameters are set according to Theorem~\ref{thm_unreg_1}, and that each policy $\pi_k$ is evaluated by the EVRTD algorithm with parameter selection in Proposition~\ref{prop_EVRTD} for $\mathcal{O}(\log (1/\epsilon))$ epochs. Then a solution $\widehat \pi$ satisfying $\bbe[\rho(\widehat \pi)-\rho(\pi^*)]\leq \order(1)\eapprox+ \epsilon$ can be found in at most $\order(\kappa^2 \log|\calA|/\epsilon^2)$ SPMD iterations and the total number of state-action transition samples can be bounded by
	$$
	\widetilde\order\big( \tfrac{\tmix^3\kappa^2}{\bar\mu\epsilon^2} \big).
	$$
	Moreover, we have \tli{$\kappa^2\leq\order\big( \tmix^2 + \eapprox^2 +  \tfrac{\bar \mu^{-1}\tmix^2 +(\bar \mu^{-1} + 1) \eapprox^2}{\min_{\pi}\{\lambda_{\min} (\widetilde D^{\pi})\}}\big)$.}
\end{corollary}
\proof{Proof.}
	For each SPMD iteration, to ensure policy evaluation bias to satisfy $\varsigma_k\leq \epsilon+\order(1)\eapprox$, we require $\widetilde{\order}(\frac{\tmix^3}{\bar\mu})$ samples for policy evaluation by \tli{Proposition~\ref{prop_EVRTD} and} Theorem~\ref{thm_VRTD_bias}. \tli{Here, due to the linear rate in the complexity \eqref{complexity_bias_VRD+TD}, the price we paid for transferring from $\|\cdot\|_{\widetilde D^{\pi_k}}$ to $\|\cdot\|_{\text{sp}}$ will only be a logarithmic factor of problem parameters. Then, together with Theorem~\ref{thm_unreg_1},} the total number of samples can be bounded by
	$
	\widetilde\order\big( \tfrac{\tmix^3\kappa^2}{\bar\mu\epsilon^2} \big).
	$
	Next, we prove an upper bound for $\kappa^2$. 
	For $k=0,...,K-1$, by \tli{Proposition~\ref{prop_EVRTD} and Ineq.~\eqref{conclusion_thm_4} of Theorem~\ref{thm_VRTD}} we can upper bound $\kappa_k^2$ by
 \tli{
	\begin{align}\label{derive_kappa_bound}
		\kappa_k^2 &\leq 2\spannorm{\Psi \theta_{\pi_k}^*}^2+ 2\bbe\spannorm{\stochQ{k} - \Psi \theta_{\pi_k}^*}^2\nn\\
  &\overset{(i)}\leq 4\spannorm{\bar Q^{\pi_k}}^2 + 4 \spannorm{\bar Q^{\pi_k} - \Psi \theta_{\pi_k}^*}^2 + 2\bbe[\|\stochQ{k} - \Psi \theta_{\pi_k}^*\|_{\widetilde D^{\pi_k}}^2/\lambda_{\min}(\widetilde D^{\pi_k})]\nn\\
		&\overset{(ii)}\leq 4\spannorm{\bar Q^{\pi_k}}^2 + 4 \spannorm{\bar Q^{\pi_k} - \Psi \theta_{\pi_k}^*}^2 + \mathcal{O}\left( \tfrac{\mu (1-\beta)^2 W_1}{\lambda_{\min} (\widetilde D^{\pi_k})}  \right) \nn\\
  &\overset{(iii)}\leq 4\spannorm{\bar Q^{\pi_k}}^2 + 4 \spannorm{\bar Q^{\pi_k} - \Psi \theta_{\pi_k}^*}^2 + \order\big(\tfrac{\|\theta_{\pi_k}^*\|_2^2+\|\bar Q^{\pi_k} - \Psi \theta_{\pi_k}^* \|^2_{\widetilde D^{\pi_k}}} {\lambda_{\min} (\widetilde D^{\pi_k})}\big)\nn\\
  &\overset{(iv)}\leq  \order\big( \tmix^2 + \eapprox^2 +  \tfrac{\bar \mu^{-1}\tmix^2 +(\bar \mu^{-1} + 1) \eapprox^2}{\min_{\pi}\{\lambda_{\min} (\widetilde D^{\pi})\}}\big), 
	\end{align}
	where step (i) follows from triangle inequality and the norm equivalence, step (ii) follows from Ineq.~\eqref{conclusion_thm_4} of Theorem~\ref{thm_VRTD} and the choice $N, N' \geq \tfrac{1}{\mu (1-\beta)^2}, K \geq \log_2(1/\epsilon)$, step (iii) follows from the definition of $W_1$, and step (iv) follows from the fact that $\|\theta\|_2^2 \leq \mu^{-1}\|\Psi \theta\|^2_{\widetilde D_{\pi_k}}\leq \mathcal{O}(\mu^{-1}(\tmix^2+\eapprox^2))$ and taking the maximum over the last step.  }
\endproof}

\revision{The next corollary provides the sample complexity for obtaining a convergence guarantee on the output policy. 
\begin{corollary}[Sample complexity for regularized AMDPs]\label{coro_comp_Mark_2}
	Assume the Markovian noise setting and Assumption~\ref{assump_rho_0} holds. Consider the ergodic AMDPs with an $\omega$-strongly convex regularizer. Suppose the SPMD stepsize parameters are set according to Theorem~\ref{thm_regularized_SPMD}. 
 For each policy $\pi_k$, implement the EVRTD algorithm with parameter selection in Theorem~\ref{thm_VRTD_bias} for $\log (1/\epsilon)$ epochs, then a solution $\widehat \pi$, which satisfies that $\bbe[D(\widehat \pi, \pi^*)]\leq \order(1)\frac{\eapprox}{\omega\Gamma}+ \epsilon$, can be found in at most $\widetilde\order(\frac{\sigma^2}{\omega^2(1-\Gamma)\epsilon})$ SPMD iterations and the total number of state-action transition samples can be bounded by
	$
	\widetilde\order\big( \tfrac{\sigma^2\tmix^3}{\omega^2(1-\Gamma)\bar\mu\epsilon} \big).
	$
	Moreover, we have $\sigma^2 \leq 2\eapprox^2 + \order\big(\max_{\pi} \tfrac{\|\theta_\pi^*\|_2^2}{\lambda_{\min} (\widetilde D^{\pi})}\big)$.
\end{corollary}
\proof{Proof.}
	For each SPMD iteration, to ensure policy evaluation bias to satisfy $\varsigma_k\leq \epsilon+\order(1)\eapprox/(\omega\Gamma)$, we require $\widetilde{\order}(\frac{\tmix^3}{\bar\mu})$ samples for policy evaluation by Theorem~\ref{thm_VRTD_bias}.  Therefore, the total number of samples can be bounded by
	$
	\widetilde\order\big( \frac{\sigma^2\tmix^3}{\omega^2(1-\Gamma)\bar\mu\epsilon} \big).
	$
	Next, we proof an upper bound for $\sigma^2$. 
	For $k=0,...,K-1$, by Theorem~\ref{thm_VRTD} we can upper bound $\sigma_k^2$ by
	\begin{align*}
		\sigma_k^2 &\leq 2\eapprox^2+ 2\bbe\|\stochQ{k}- \Psi \theta_{\pi_k}^*\|_\infty^2 \leq 2\eapprox^2 + 2\bbe[\|\stochQ{k} - \Psi \theta_{\pi_k}^*\|_\infty^2/\lambda_{\min}(\widetilde D^{\pi_k})]\\
		&\leq 2\eapprox^2 + \order\big(\tfrac{\|\theta_{\pi_k}^*\|_2^2}{\lambda_{\min} (\widetilde D^{\pi_k})}\big)\leq 2\eapprox^2 + \order\big(\max_{\pi} \tfrac{\|\theta_\pi^*\|_2^2}{\lambda_{\min} (\widetilde D^{\pi})}\big),
	\end{align*}
	as desired.
\endproof
To the best of our knowledge, both sample complexities in Corollary~\ref{coro_comp_Mark_1} and \ref{coro_comp_Mark_2} appear to be new in the AMDP literature. The former result extends the result of the generative model to the Markovian noise model, while the latter result further exploits the ergodicity structure as well as the impact of the strongly convex regularizer, thus providing an improved $\widetilde{\order}(\epsilon^{-1})$ convergence guarantee in terms of the KL divergence between the output solution and the optimal solution. }
{The next corollary provides the sample complexity for using SPMD and EVRTD to solve the regularized AMDPs under Assumption~\ref{assump_rho_0}. 
\begin{corollary}[Sample complexity for regularized AMDPs]\label{coro_comp_Mark_2}
	Assume the Markovian noise setting and Assumption~\ref{assump_rho_0} holds. Consider the ergodic AMDPs with an $\omega$-strongly convex regularizer. Suppose that the SPMD stepsize parameters are set according to Theorem~\ref{thm_regularized_SPMD}, that each policy $\pi_k$ is evaluated by using the EVRTD algorithm with parameter selection in Theorem~\ref{thm_VRTD_bias} for $\log (1/\epsilon)$ epochs, then a solution $\widehat \pi$ satisfying $\bbe[D(\widehat \pi, \pi^*)]\leq \order(1)\frac{\eapprox}{\omega\Gamma}+ \epsilon$ can be found in at most $\widetilde\order(\frac{\sigma^2}{\omega^2(1-\Gamma)\epsilon})$ SPMD iterations and the total number of state-action transition samples can be bounded by
	$
	\widetilde\order\big( \tfrac{\sigma^2\tmix^3}{\omega^2(1-\Gamma)\bar\mu\epsilon} \big).
	$
	Moreover, \tli{we have $\sigma^2 \leq \order\big(   \tfrac{\bar \mu^{-1}\tmix^2 +(\bar \mu^{-1} + 1) \eapprox^2}{\min_{\pi}\{\lambda_{\min} (\widetilde D^{\pi})\}}\big)$.}
\end{corollary}
\proof{Proof.}
	For each SPMD iteration, to ensure policy evaluation bias to satisfy $\varsigma_k\leq \epsilon+\order(1)\cdot \eapprox/(\omega\Gamma)$, we require $\widetilde{\order}(\frac{\tmix^3}{\bar\mu})$ samples for policy evaluation by \tli{Proposition~\ref{prop_EVRTD} and} Theorem~\ref{thm_VRTD_bias}. Therefore, the total number of samples can be bounded by
	$
	\widetilde\order\big( \frac{\sigma^2\tmix^3}{\omega^2(1-\Gamma)\bar\mu\epsilon} \big)
	$ \tli{according to Theorem~\ref{thm_regularized_SPMD}.}
	Next, we prove an upper bound for $\sigma^2$. 
	For $k=0,...,K-1$, by \tli{Proposition~\ref{prop_EVRTD} and Ineq.~\eqref{conclusion_thm_4} of Theorem~\ref{thm_VRTD}} we can upper bound $\sigma_k^2$ by
 \tli{
	\begin{align*}
		\sigma_k^2 &\leq 2\eapprox^2+ 2\bbe\|\stochQ{k}- \Psi \theta_{\pi_k}^*\|_\infty^2 \overset{(i)}\leq 2\eapprox^2 +\order\big(   \tfrac{\bar \mu^{-1}\tmix^2 +(\bar \mu^{-1} + 1) \eapprox^2}{\min_{\pi}\{\lambda_{\min} (\widetilde D^{\pi})\}}\big) = \order\big(   \tfrac{\bar \mu^{-1}\tmix^2 +(\bar \mu^{-1} + 1) \eapprox^2}{\min_{\pi}\{\lambda_{\min} (\widetilde D^{\pi})\}}\big),
	\end{align*}}
 \tli{where step (i) follows a similar argument to Ineq.~\eqref{derive_kappa_bound}.}
\endproof
To the best of our knowledge, both sample complexities in Corollary~\ref{coro_comp_Mark_1} and \ref{coro_comp_Mark_2} appear to be new in the AMDP literature. The former result extends the result of the generative model to the Markovian noise model, while the latter result further exploits the ergodicity structure as well as the impact of the strongly convex regularizer, thus providing an improved $\widetilde{\order}(\epsilon^{-1})$ convergence guarantee in terms of the KL divergence between the output solution and the optimal solution. }

\section{Numerical experiments}\label{sec:numerical} In this section, we report numerical experiments for our proposed algorithms, i.e., VRTD, EVRTD and SPMD methods. To generate a comprehensive performance profile, we conduct experiments on the popular evaluation benchmark -- OpenAI Gym \citep{brockman2016openai}.

    \revision{}{
	\subsection{Taxi Driving}\label{subsec:numerical_markov_chain}
    We consider Taxi, a grid-world alike MDP with state space $|\calS|=500$ and action space $|\calA|=6$.
    In the $5$ by $5$ grid world, there are four designated pick-up and drop-off locations where the passenger spawns in one of them and wants to go to one of the rest locations. The agent can move up, down, left, or right on the map, and pick up and drop off the passenger. The agent starts at a random position while the passenger starts at a random designated location. The goal of the agent is to pick up the passenger and drop him off at the desired destination. A negative cost ($-20$) is given when the agent picks up the passenger and drops him off at the designated location. The agent gets a cost $(+15)$ whenever illegal actions are performed (e.g., pick up at empty locations). Additionally, every step triggers a $(+1)$ cost to promote shorter routes. 
	
	Using the Taxi instance, we test the performance of the SPMD method with EVRTD as the critic (namely SPMD-EVRTD), and compare its performance against the SPMD method with VRTD as the critic (namely SMPD-VRTD). Figure~\ref{fig-numerical-study-mc} plots the average cost function of the current policies against the number of iterations of the SPMD method. 
	
	\begin{figure}[!htb]
		\minipage{0.32\textwidth}
			  \includegraphics[width=5.5cm]{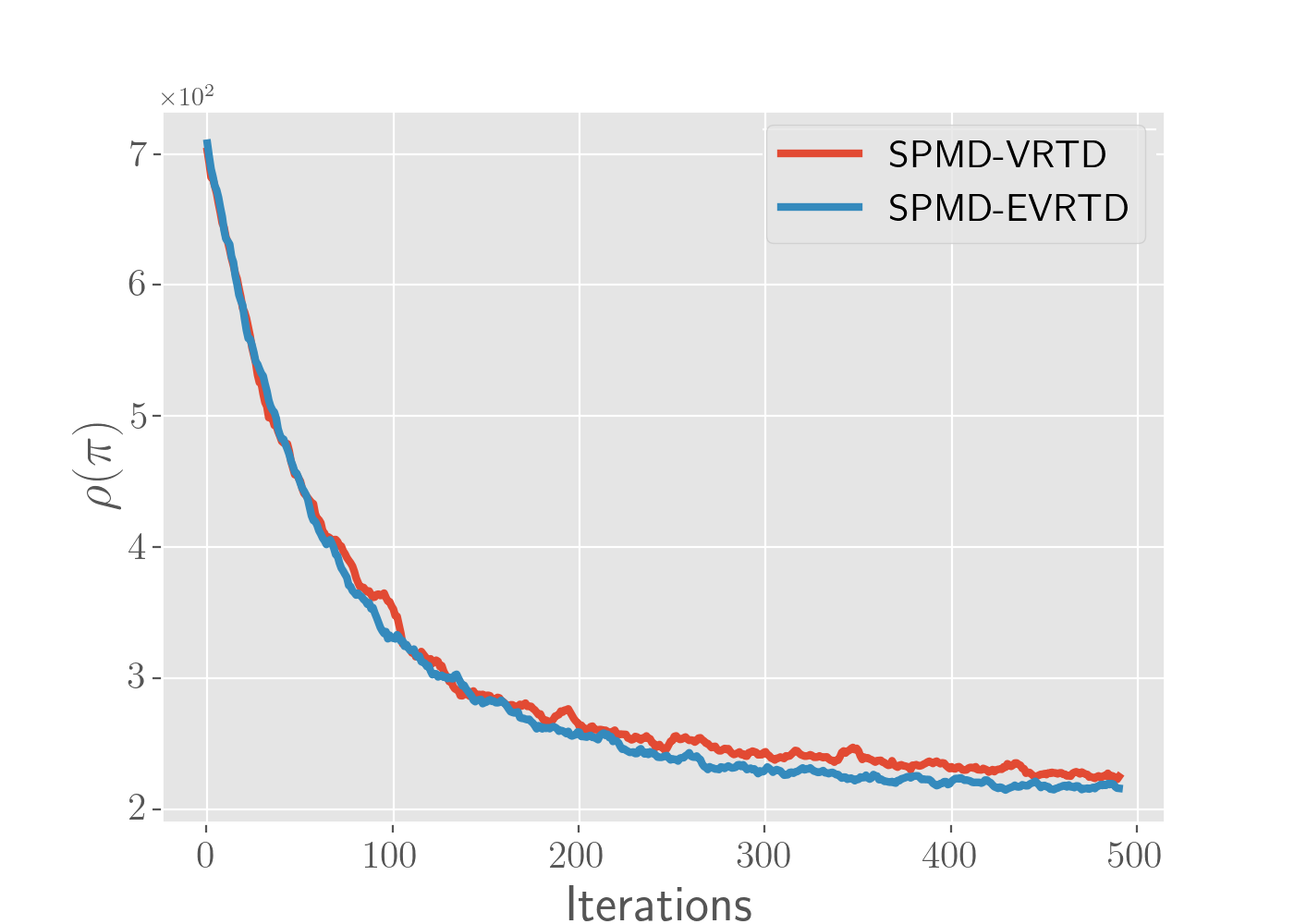}
		\endminipage\hfill
		\minipage{0.32\textwidth}
		  \includegraphics[width=5.5cm]{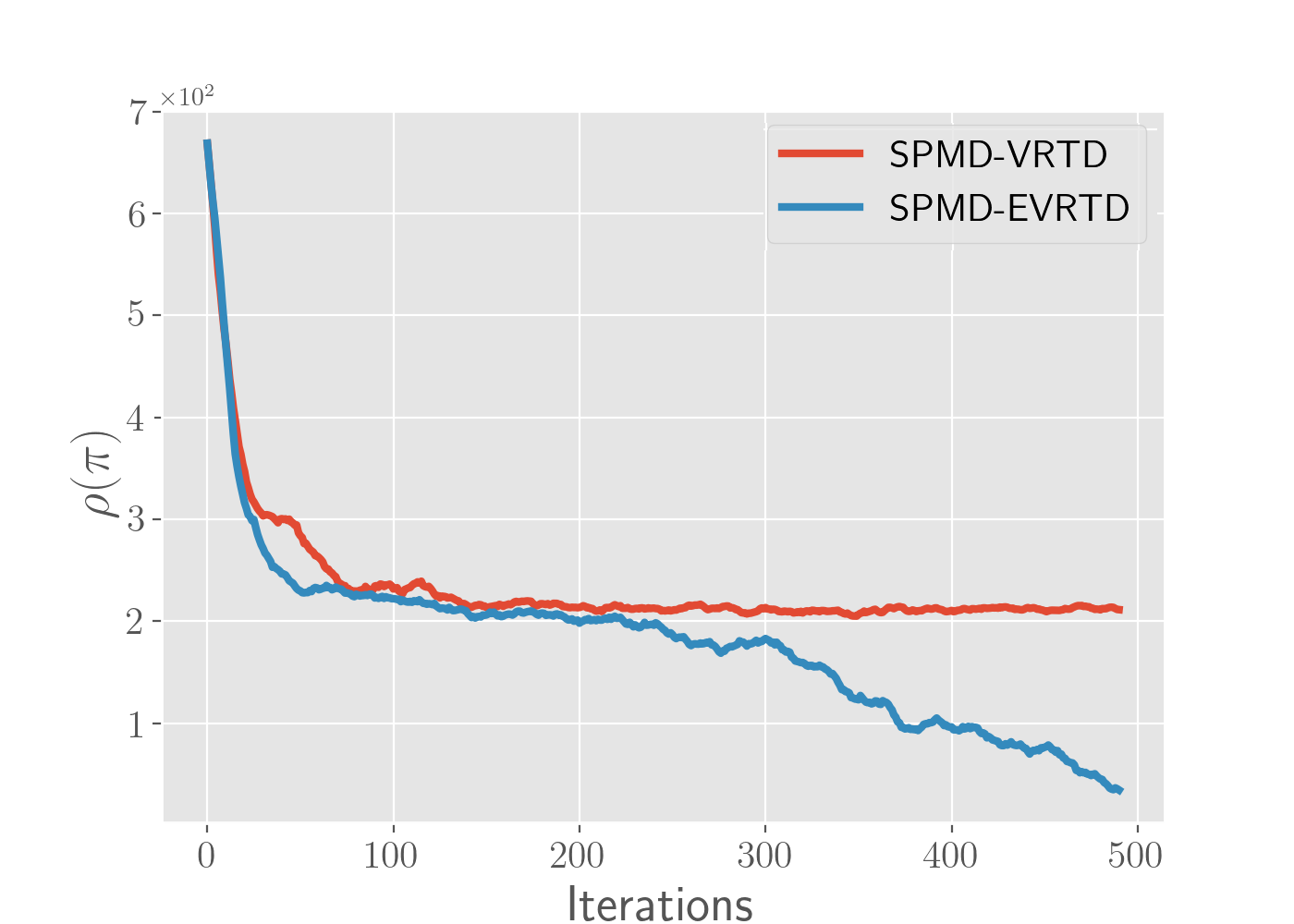}
		\endminipage\hfill
		\minipage{0.32\textwidth}%
		  \includegraphics[width=5.5cm]{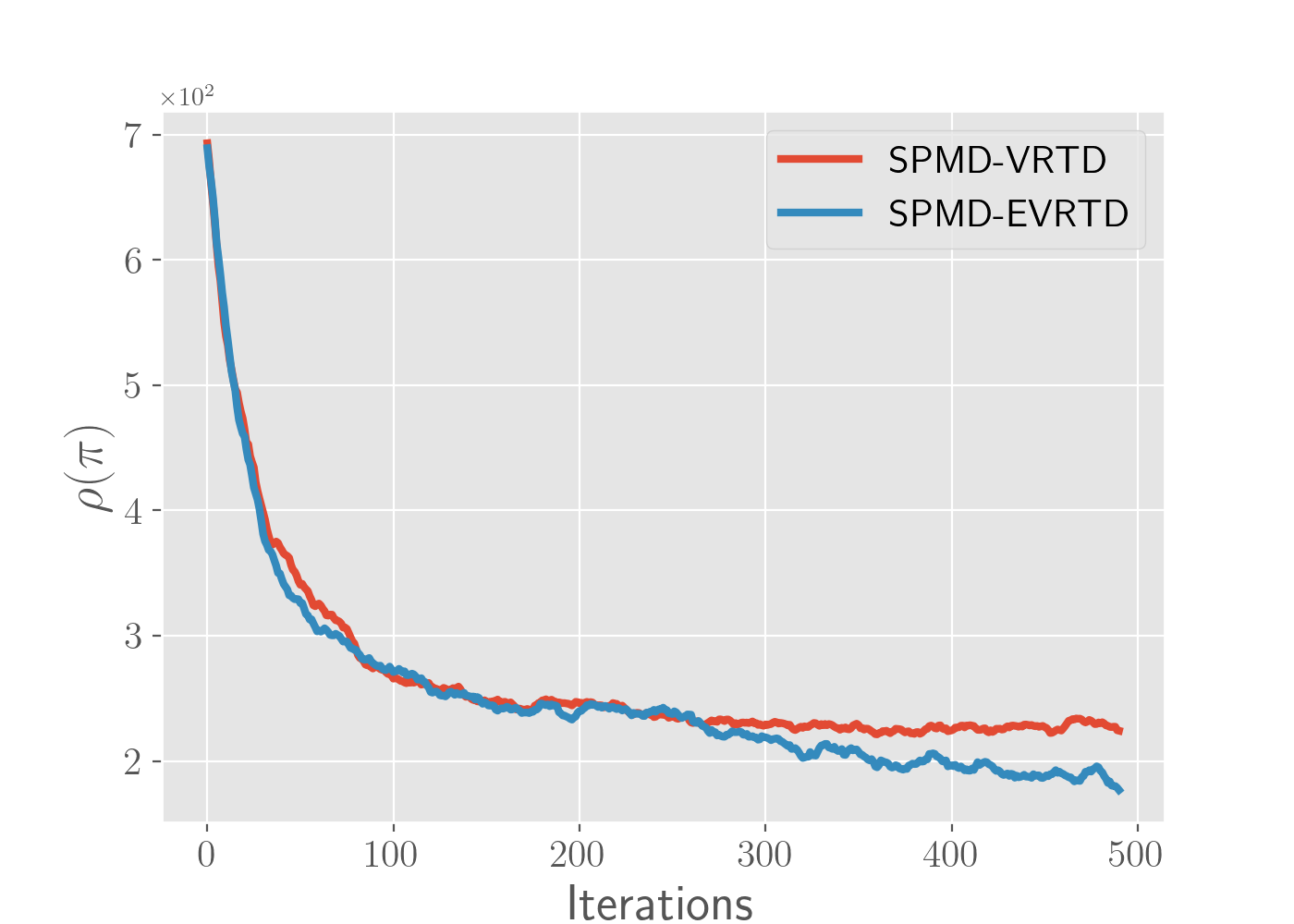}
		\endminipage
		\caption{Comparing the performance of the SPMD method with different critic solvers. From left to right: $h=0$, $h^\pi(s)= \tsum_{a\in{\cal A}} \pi(a|s)\log \pi(a|s)$, and $h^\pi(s)= 3\cdot \tsum_{a\in{\cal A}} \pi(a|s)\log \pi(a|s)$. For the red curves, SPMD takes the VRTD method as the critic solver. For the blue curves, SPMD takes EVRTD as the critic solver. In the iterations marked by crosses, EVRTD implements the perturbation rule \eqref{construct_tilde_pi}; in the other iterations, EVRTD reduces to VRTD. 
		}\label{fig-numerical-study-mc}
		\end{figure}
	
	In the presence of a regularizer, we can see that for the blue curves, the average cost function in the later iterations decreases continuously. In comparison, it is hard for the SPMD method without the EVRTD solver to make progress in the later iterations. When no regularizer is added, VRTD and EVRTD perform similarly with EVRTD achieves slightly better results. 
    As a result, this set of experiments showcases that the VRTD method fails to efficiently evaluate the insufficiently random policies due to lack of exploration, and the EVRTD method can remedy this issue to enable fast and robust convergence of the SPMD method.
	
		
	}
 \subsection{Robotics Control}\label{subsec:numerical_gym}
    For the rest of the experiments conducted on OpenAI Gym, we choose two environments (LunarLander and MountainCar), and compare the performance of SPMD against TRPO \citep{schulman2015trust}, PPO \citep{schulman2017proximal}, and Q-learning. Specifically, TRPO and PPO utilize neural networks for value function approximation as well as policy parametrization, and the networks are updated using stochastic gradient descent. We use state-of-the-art implementations for both TRPO and PPO in \texttt{stable\_baseline3} \citep{stable-baselines3}. Additionally, we incorporate undiscounted Q-learning with linear function approximation, which utilizes least-square updates for linear function estimators described in \cite{bertsekas1996temporal}.
    
    For linear function approximation in SPMD and Q-learning, we implement radial basis function (RBF) kernels to construct the feature space. All the RBF centers are computed prior to learning, using randomly sampled state-action pairs from the environment. The policy-based algorithms (SPMD, PPO, TRPO) perform policy updates after each episode with a specified length, and evaluate its performance by sampling multiple trajectories. 
    Note that even though TRPO and PPO are trained using discounted objectives, their performance is evaluated without discounting, thus leading to fair comparisons with our proposed SPMD method.
    \begin{figure}[!htb]
	\centering
	\minipage{0.4\textwidth}
  	\includegraphics[width=5cm]{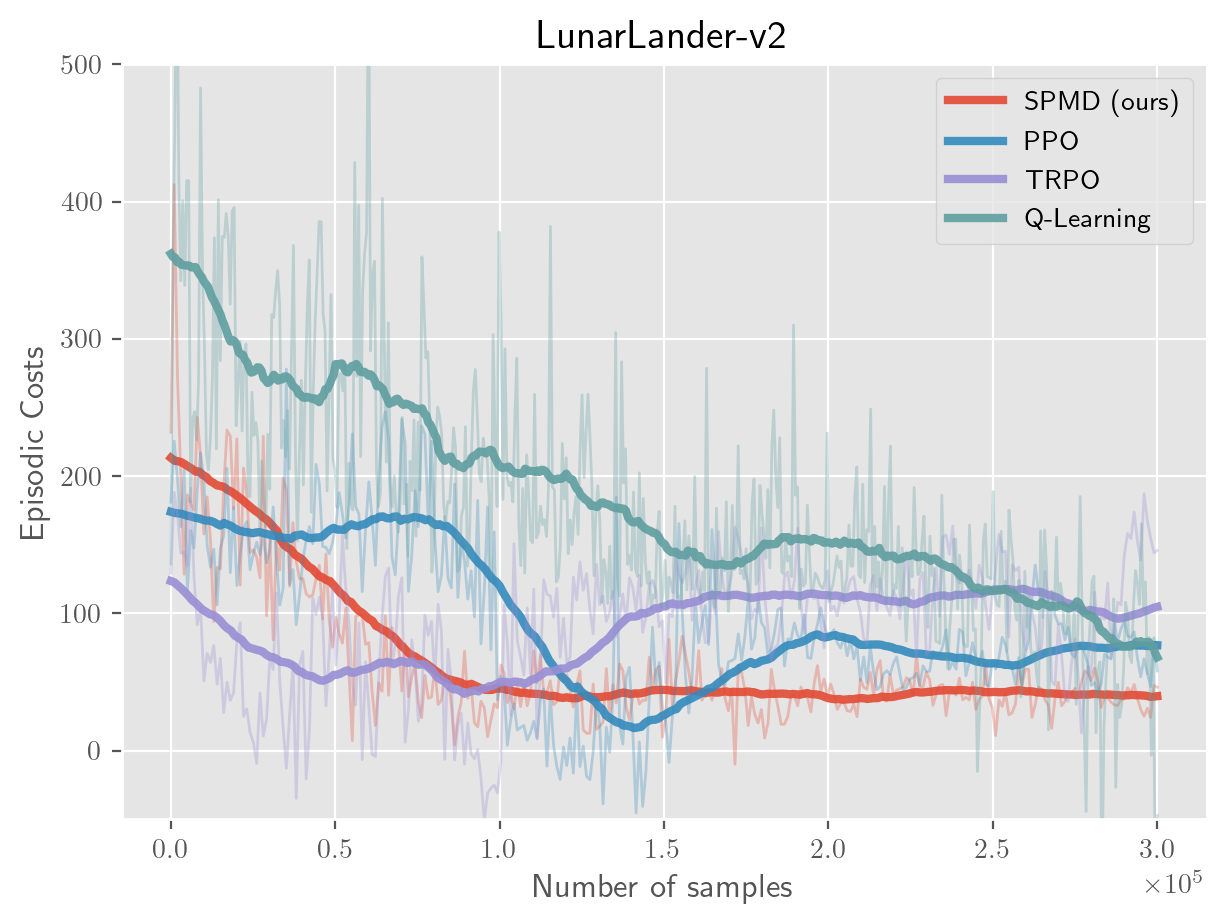}

	\endminipage
	\minipage{0.4\textwidth}%
  	\includegraphics[width=5cm]{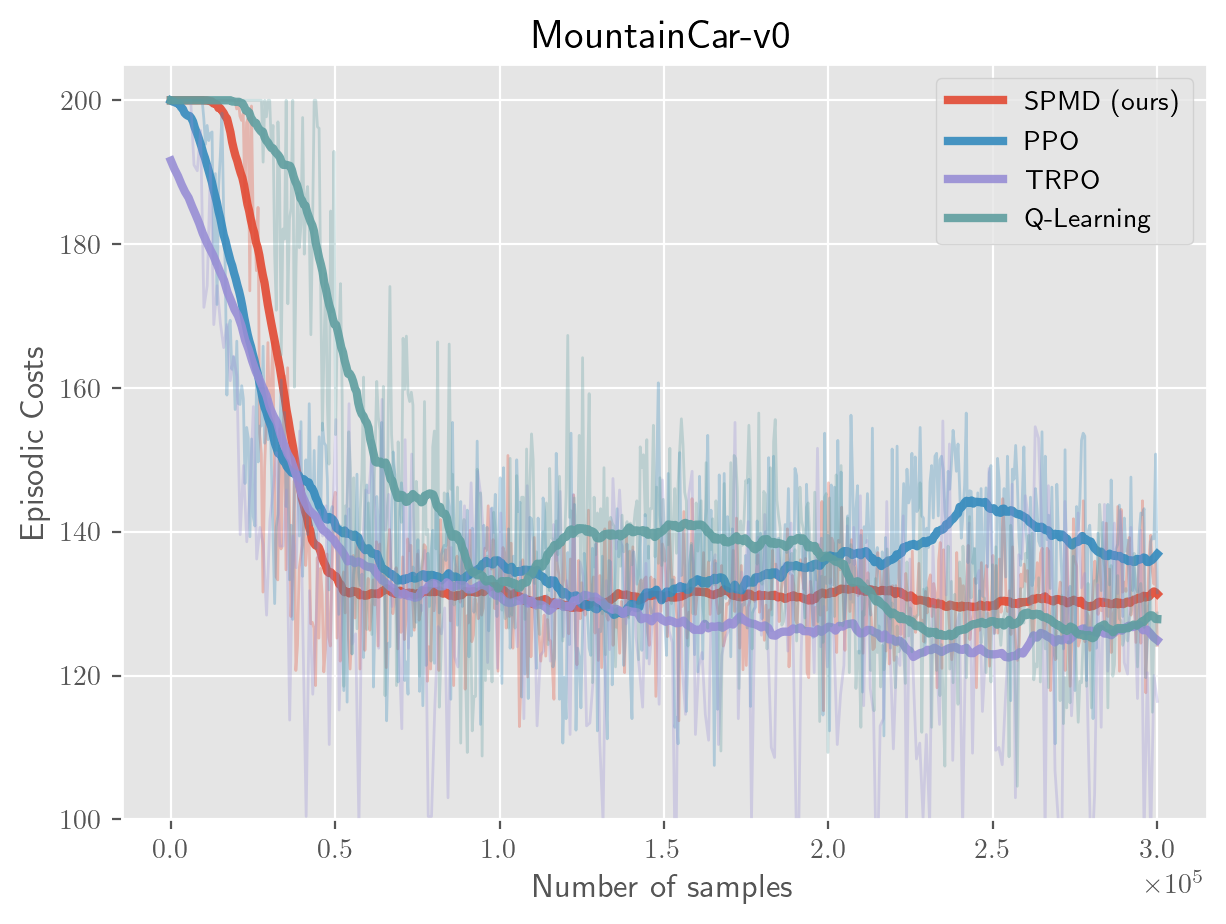}

	\endminipage
	\caption{Compare SPMD method with various methods. Policies are evaluated using 10 independent trajectories in each evaluation point. We highlight running average over 50 evaluation points. 
	}\label{fig-numerical-study}
	\end{figure}

    From Figure \ref{fig-numerical-study}, it is clear that our proposed SPMD method exhibits efficient and robust performance in both experiments. On the other hand, although TRPO and PPO perform well in MountainCar, they suffer from unstable performance in LunarLander. 
    The Q-learning approach achieves comparable policy improvement with SPMD in MoutainCar, but significantly underperforms in LunarLander.

\section{Conclusions}
\revision{In this paper, we focus on the design and analysis of policy evaluation and optimization algorithms for AMDPs. For policy evaluation, we develop the multiple trajectory method and the VRTD method for the generative model and Markovian noise model, respectively. Moreover, we develop the EVRTD method to handle insufficiently random policies in the on-policy evaluation setting. We further extend the design idea of EVRTD to DMDPs. For policy optimization, we develop an average-reward stochastic policy mirror descent (SPMD) method and establish $\widetilde{\order}(\epsilon^{-2})$ overall sample complexity under both generative model (with unichain assumption) and Markovian noise model (with ergodic assumption). Moreover, we improve the sampling complexity to $\widetilde{\order}(\epsilon^{-1})$ for solving regularized AMDPs.}
{In this paper, we focus on the design and analysis of policy evaluation and optimization algorithms for AMDPs. For policy optimization, we develop an average-reward stochastic policy mirror descent (SPMD) method along with convergence analysis for both unregularized and strongly convex regularized AMDPs. For policy evaluation, we develop the multiple trajectory method and the VRTD method for the generative model and Markovian noise model, respectively. Moreover, we develop a novel EVRTD method that resolves the issue of lack of exploration in the action space. Finally, we establish $\widetilde{\order}(\epsilon^{-2})$ overall sample complexity for unregularized AMDPs and $\widetilde{\order}(\epsilon^{-1})$ sample complexity for solving strongly convex regularized AMDPs.}

\subsection*{Acknowledgments}
	TL and GL were supported in part by the National Science Foundation grant CCF-1909298 and Office of Navel Research grant N00014-20-1-2089.
	
	\bibliographystyle{abbrvnat}
	\bibliography{arxiv_update}

\appendix
	\newpage
\section{Proofs of Section~\ref{sec:policy_evaluation}}\label{sec:proofs}
We now turn to proofs of the results in Section~\ref{sec:policy_evaluation}. 
\subsection{Mixing of Markov Chain}\label{proof_lemma_mixing}

To begin with, we establish the following supporting lemma which characterizes the mixing property of a Markov chain.
\begin{lemma}\label{lemma_mixing}
	Given a policy $\pi$, for any integer $k\geq \tmix$, we have
	\begin{align*}
		\|(P_\pi)^k - \mathbf{1} (\nu^\pi)^\top\|_\infty \leq (\tfrac{1}{2})^{\lfloor k/\tmix \rfloor}.
	\end{align*}
\end{lemma}

\begin{proof}
Fix a feasible policy $\pi$. By the fact of $(\nu^\pi)^\top P_\pi = (\nu^\pi)^\top$, we have that for any positive integer $z$,
\begin{align}\label{step_1_lemma_1}
	\left((P_\pi)^{z} - \ones (\nu^\pi)^\top\right)  \left(P_\pi - \ones (\nu^\pi)^\top\right) = (P_\pi)^{z+1} - \ones (\nu^\pi)^\top.
\end{align}
On the other hand, for any positive integer $z$,
\begin{align}\label{step_2_lemma_1}
	\|(P_\pi)^{z+1} - \ones (\nu^\pi)^\top\|_\infty =\left\|\left((P_\pi)^{z} - \ones (\nu^\pi)^\top\right)P_\pi \right\|_\infty\leq \|(P_\pi)^{z} - \ones (\nu^\pi)^\top\|_\infty.
\end{align}
Noting that $k\geq \tmix$ and utilizing the two relationships above, we obtain that
\begin{align*}
	\|(P_\pi)^k - \mathbf{1} (\nu^\pi)^\top\|_\infty&\overset{(i)}\leq\|(P_\pi)^{\tmix \cdot \lfloor k/ \tmix \rfloor} - \mathbf{1} (\nu^\pi)^\top\|_\infty\overset{(ii)}\leq \big\|\left((P_\pi)^{\tmix} - \ones (\nu^\pi)^\top\right)^{\lfloor k/ \tmix \rfloor} \big\|_\infty\\
	&\leq \big(\|(P_\pi)^{\tmix} - \ones (\nu^\pi)^\top\|_\infty\big)^{\lfloor k/ \tmix \rfloor},
\end{align*}
where step (i) follows from Ineq.~\eqref{step_2_lemma_1}, and step (ii) follows from Ineq.~\eqref{step_1_lemma_1}.
Next, we bound the term $\|(P_\pi)^{\tmix} - \ones (\nu^\pi)^\top\|_\infty$. We have 
\begin{align*}
	\|(P_\pi)^{\tmix} - \ones (\nu^\pi)^\top\|_\infty &= \max_{s\in\calS} \tsum_{s'\in\calS}|(P_\pi)^{\tmix}(s,s')-\nu^\pi(s')|\leq \max_{q\in \Delta_{|\calS|}}\|((P_\pi)^{\tmix})^\top q - \nu^\pi\|_1\overset{(i)}\leq \tfrac{1}{2},
\end{align*}
where (i) follows from Assumption~\ref{assump_mixing}. Combining two relationships above yields the desired result. 

\end{proof}

\subsection{Proof of Proposition~\ref{lemma_bias_bound_generative}}\label{proof_lemma_bias_bound_generative}


	For simplicity of notation, we denote $\Delta \rho := \widehat \rho^\pi - \rho^\pi$, $\Delta \rho^i := \rho^i - \rho^\pi$, thus we have $$\Delta \rho = \tfrac{1}{M}\tsum_{i=1}^M \Delta \rho^i.$$ 
	We begin by establishing a bias bound for the estimator of the average reward/cost $|\bbe[\Delta \rho^i]|$,
	\begin{align*}
		|\bbe[\Delta \rho^i]| = | (P_\pi)^{T}_{s'} \cdot c^\pi - \rho^\pi)|
		\leq \|(P_\pi)^{T} c^\pi - \mathbf{1} (\nu^\pi)^\top c^\pi \|_\infty
		\leq \|(P_\pi)^{T} - \mathbf{1} (\nu^\pi)^\top\|_\infty \|c^\pi \|_\infty
		\leq \bar  c\cdot (\tfrac{1}{2})^{\lfloor T/\tmix \rfloor}.
	\end{align*}
	Consequently, by triangle inequality,  we have 
	\begin{align}\label{bound_bias_rho}
		|\bbe[\Delta \rho]|\leq \bar c\cdot  (\tfrac{1}{2})^{\lfloor T/\tmix \rfloor}.
	\end{align}
	Now we are ready to bound the bias of the estimator of the basic differential Q-function. For simplicity of notation, we denote $\delta(s,a):=\tfrac{1}{N'}\tsum_{j=1}^{N'} \delta^j(s,a)$, where $$\delta^j(s,a):= \tsum_{t=0}^{T'} \left(c(s_t^j,a_t^j)-\bbe_\pi[c(s_t,a_t)|s_0=s,a_0=a] \right),\quad \forall s\in\calS, a\in \calA.$$ Let the error induced by truncation be
	\begin{align*}
		\trunc(s,a) := \tsum_{t=T'+1}^\infty \left(\rho^\pi - \bbe_\pi[c(s_t,a_t)|s_0=s,a_0=a] \right),\quad \forall s\in\calS, a\in \calA.
	\end{align*}
	Then by the definition of $\widehat Q$ and $\bar Q$, we have the following decomposition
	\begin{align}\label{Q_error_decomp}
		\widehat Q^\pi(s,a) - \bar Q^\pi(s,a) &= \tli{ \tfrac{1}{N'} \tsum_{j=1}^{N'}\tsum_{t=0}^{T'} \big(c(s_t^j,a_t^j)- \widehat \rho^{\pi}\big) - \bbe_\pi\left[\tsum_{t=0}^{\infty}\left(c\left(s_{t}, a_{t}\right)-\avgrwd\right) \big| s_{0}=s, a_{0}=a\right]}\nn\\
  &= \delta(s,a)  + (T'+1) \Delta\rho+ \trunc(s,a),\quad \forall s\in\calS, a\in \calA.
	\end{align}
	Noticing that $\bbe[\delta^j(s,a)] = 0$ because the trajectories are sampled following policy $\pi$, we have $\bbe[\delta(s,a)]=0$.
	It remains to upper bound $|\bbe[\trunc(s,a)]|$. For any $s\in\calS, a\in\calA$, we write
		\begin{align}\label{bound_trunc}
			|\trunc(s,a)|
			& = \left|\tsum_{s' \in \mathcal{S}} \mathsf{P}(s'|s,a) \tsum_{t=T'+1}^{\infty}[ (\nu^\pi)^\top c^\pi-(P_\pi)^{t-1}_{s'} c^\pi  ] \right|\nn\\
			&\leq  \left\|\tsum_{t=T'}^{\infty} [  \mathbf{1}(\nu^\pi)^\top-(P_\pi)^t ] c^\pi  \right\|_\infty \leq \tsum_{t=T'}^{\infty} \| \mathbf{1}(\nu^\pi)^\top-(P_\pi)^t \|_\infty\|c^\pi\|_\infty\nn\\
			&\overset{(i)}\leq \bar  c\cdot\tmix \cdot \tsum_{i=\lfloor T'/\tmix \rfloor}^\infty (\tfrac{1}{2})^{i}\leq 2\bar  c\cdot\tmix \cdot (\tfrac{1}{2})^{\lfloor T'/\tmix \rfloor},
	\end{align}
    where step (i) follows from Lemma~\ref{lemma_mixing}. Combining Ineqs. \eqref{bound_bias_rho}, \eqref{Q_error_decomp}, \eqref{bound_trunc}, and utilizing triangle inequality, we obtain
	\begin{align*}
		|\bbe[\widehat Q^\pi(s,a) - \bar Q^\pi(s,a)]| \leq \bar  c(T'+1) \cdot (\tfrac{1}{2})^{\lfloor T/\tmix \rfloor} + 2\bar  c\cdot \tmix \cdot (\tfrac{1}{2})^{\lfloor T'/\tmix \rfloor}.
	\end{align*}
	Noticing that the RHS of the above inequality does not depend on $(s,a)$, so this is a uniform bound for all the entries of $\bbe[\widehat Q^\pi - \bar Q^\pi]$ and we complete the proof.
\endproof

\subsection{Proof of Proposition~\ref{lemma_variance_bound_generative}}\label{proof_lemma_variance_bound_generative}
\proof{Proof.}
	First, focus our attention on providing an upper bound on $\bbe[(\Delta \rho)^2]$. We write
	\begin{align}\label{reward_variance_bound}
		\bbe[(\Delta \rho)^2] = \text{Var}(\Delta \rho) + (\bbe[\Delta \rho])^2\overset{(i)}\leq \tfrac{4\bar c^2}{M} + \bar c^2\cdot (\tfrac{1}{2})^{2\lfloor T/\tmix \rfloor},
	\end{align}
	where step (i) follows from that each $\rho^i$ is independent of each other and Ineq. \eqref{bias_bound_generative_0} of Proposition~\ref{lemma_bias_bound_generative}.
	Recalling the decomposition in Eq. \eqref{Q_error_decomp} and utilizing Young's inequality, we have
	\begin{align}\label{decomp_2}
		\bbe[\|\widehat Q^\pi - \bar Q^\pi\|_\infty^2] \leq 3\bbe[\|\delta\|_\infty^2] + 3(T'+1)^2\bbe[(\Delta\rho)^2] + 3\|\trunc\|_\infty^2,
	\end{align}
	where the last term on the RHS does not require an expectation since it does not depend on the samples generated when running Algorithm~\ref{alg:Mult}. Invoking an uniform upper bound on $\trunc$ in Ineq. \eqref{bound_trunc}, it remains to bound the term $\bbe[\|\delta\|_\infty^2]$. First, we focus on an $(s,a)$ pair, for $j=1,...,M'$,
	\begin{align*}
		\bbe[\delta^j(s,a)] = 0,\quad\text{and}\quad |\delta^j(s,a)|\leq 2(T'+1)\bar c. 
	\end{align*}
	Since each $\delta^j(s,a)$ is bounded and independent with each other, we have that $\delta(s,a)=\tfrac{1}{M'}\tsum_{j=1}^{M'}\delta^j(s,a)$ is a sub-gaussian variable, and there exists an absolute constant $u>0$ such that $\|\delta(s,a)\|_{\psi_2} \leq \tfrac{u\bar c (T'+1)}{\sqrt{M'}}$. As a result, $\delta(s,a)^2$ is sub-exponential and 
	\begin{align*}
		\|\delta(s,a)^2\|_{\psi_1}=\|\delta(s,a)\|^2_{\psi_2} \leq \tfrac{u^2\bar c^2 (T'+1)^2}{M'}.
	\end{align*}
	On the other hand, given that $\delta^j(s,a)$ are i.i.d. and zero-mean, we have
	\begin{align*}
		\bbe[\delta(s,a)^2] \leq \tfrac{4\bar c^2 (T'+1)^2}{M'}. 
	\end{align*}
	Then utilizing Lemma 25 of \cite{lan2021policy}, we have that there exists a universal constant $\upsilon>0$, such that 
	\begin{align}\label{delta_bound}
		\bbe[\|\delta\|_\infty^2] = \bbe[\max_{s\in\calS,a\in\calA}\delta(s,a)^2] \leq \tfrac{\upsilon\bar c^2 (T'+1)^2}{M'}(\log(|\calS||\calA|)+1) + \tfrac{4\bar c^2 (T'+1)^2}{M'}.
	\end{align}
	Substituting Ineqs. \eqref{delta_bound}, \eqref{reward_variance_bound} and \eqref{bound_trunc} into Ineq. \eqref{decomp_2} and invoking $M\geq M'$, we obtain
	\begin{align*}
		\bbe[\|\widehat Q^\pi - \bar Q^\pi\|_\infty^2] \leq \tfrac{3\upsilon \bar c^2 (T'+1)^2}{M'}(\log(|\calS||\calA|)+1) + \tfrac{24\bar c^2 (T'+1)^2}{M'} + \big(3(T'+1)^2+12\tmix^2)\bar c^2\cdot (\tfrac{1}{2})^{2\lfloor T'/\tmix \rfloor},
	\end{align*}
	which completes the proof.
\endproof

\subsection{Supporting lemmas for Theorem~\ref{thm_VRTD} and \ref{thm_VRTD_bias}}
We begin by establishing a few useful lemmas which will be  used in the proof of the main theorems in Section~\ref{sec:policy_evaluation}.

\revision{}{
\tli{Recall the definition of the deterministic operator in \eqref{deterministic_opt} and the definition of the stochastic operator in \eqref{stochastic_opt}.} To quantify how the sample skipping technique works in reducing the bias of the stochastic operator, we establish the following lemmas. 
\begin{lemma}\label{lemma_operator_bias_1}
	Let $\xi_t:=\{(s_t,a_t),(s_{t+1},a_{t+1}), c(s_t,a_t)\}$. For every $t,\tau\in \mathbb{Z}_{+}$, with probability 1,
	\begin{align}\label{bias_2}
		\|\bbe[\widetilde g( \thetasol, \rhosol, \xi_{t+\tau+1})|\mathcal{F}_t]-g( \thetasol,\rhosol)\|_2 \leq \mixcon\cdot \mixrate^\tau \|\Psi \theta^* -\Qsol\|_D.
	\end{align}
	where $\mixcon := \tfrac{\skipcon}{\sqrt{\min_{s\in\calS, a \in \calA} \nu(s)\pi(a|s)}}\|\Psi\|_2 \|I-P\|_2$ and $\mathcal{F}_t:= [\xi_1,...,\xi_t]$.
\end{lemma}
\proof{Proof.}
	Let $D_t^{\tau+1}:=\diag\left([\mathsf{P}(s_{t+\tau+1}=s, a_{t+\tau+1}=a|s_t, a_t)]\right)$ where $s\in\calS, a\in\calA$. 
	We have
	\begin{align*}
		\bbe[\widetilde g(\thetasol,\rhosol,\xi_{t+\tau+1})|\mathcal{F}_t] = \Psi D_t^{\tau+1}(\Psi \thetasol- P \Psi\thetasol -c + \rhosol\ones).
	\end{align*} 
	Combining the above equality with the definition of $g(\cdot)$ in~\eqref{deterministic_opt} yields
	\begin{align*}
		\|g(\thetasol,\rhosol)-\bbe[\widetilde g(\thetasol,\rhosol,\xi_{t+\tau+1})|\mathcal{F}_t]\|_2
		&=\|\Psi^\top  \big(D-D_t^{\tau+1}\big) (\Psi \thetasol- P \Psi\thetasol -c + \rhosol\ones)\|_2\\
		&= \|\Psi^\top  \big(D-D_t^{\tau+1}\big) (I-P)(\Psi\thetasol-\Qsol)\|_2\\
		&\leq \skipcon\mixrate^\tau \|\Psi\|_2\|I-P\|_2\|\Psi\thetasol-\Qsol\|_2 \leq  \mixcon \mixrate^\tau \|\Psi\thetasol-\Qsol\|_D,
	\end{align*}
	as claimed.
\endproof
In view of Lemma~\ref{lemma_operator_bias_1}, the bias of the stochastic operator at the optimal solution depends on the approximation error caused by linear function approximation, and the bound decays geometrically. 

\begin{lemma}\label{lemma_operator_bias_2}
	For every $t,\tau\in \mathbb{Z}_{+}$, $\theta, \theta'\in \bbr^d$ and $\rho, \rho' \in \bbr$, with probability 1,
	\begin{align}\label{bias_1}
		\|\bbe[\widetilde g(\theta, \rho,\xi_{t+\tau+1})|\mathcal{F}_t] -\bbe[\widetilde g(\theta', \rho',\xi_{t+\tau+1})|\mathcal{F}_t] &- [g(\theta, \rho)-g(\theta', \rho')] \|_2 \nn\\
		&\leq \mixcon\cdot \mixrate^\tau \|\Psi \theta - \Psi \theta'\|_D + \mixcons\cdot\mixrate^\tau |\rho - \rho'|,
	\end{align}
	where $\mixcons := {\skipcon}\|\Psi\|_2$.
\end{lemma}
\proof{Proof.}
	We write
	\begin{align*}
		&\quad\|g(\theta, \rho)-\bbe[\widetilde g(\theta,\rho,\xi_{t+\tau+1})|\mathcal{F}_t]-g(\theta', \rho')+\bbe[\widetilde g(\theta',\rho',\xi_{t+\tau+1})|\mathcal{F}_t]\|_2\\
		&=\|\Psi^\top \big(D-D_t^{\tau+1}\big)\big((I-P)\Psi (\theta-\theta') + (\rho-\rho')\big)\|_2\\
		&\leq \skipcon\mixrate^\tau \|\Psi\|_2   \|I-P\|_2  \|\Psi(\theta-\theta')\|_2 + \skipcon\mixrate^\tau \|\Psi\|_2 |\rho-\rho'| \leq  \mixcon\cdot \mixrate^\tau \|\Psi \theta - \Psi \theta'\|_D + \mixcons\cdot\mixrate^\tau |\rho - \rho'|,
	\end{align*}
	as claimed.
\endproof
In contrast to Lemma~\ref{lemma_operator_bias_1}, Lemma~\ref{lemma_operator_bias_2} captures the bias reduction of the difference between two stochastic operators calculated using the same samples, which also decays in a geometric rate and does not depend on the approximation error. 
For brevity of notation, we set $\Cmax:=\max\{\skipcon, \mixcon, \mixcons\}$. 
}

\revision{}{
Next, we establish a supporting lemma that is a consequence of Lemma~\ref{monotone_constant}.} \tli{Recall the definition of $M$ in \eqref{def_mu_M}.}

\begin{lemma}\label{lemma_inverse_operator}
	For any $\theta\in\bbr^d$, we have that the matrix $I_d - M$ is invertible and
	\begin{align}\label{inverse_operator}
		\|\big(I_d -M\big)^{-1} \theta\|_2 \leq \tfrac{1}{\monotonecon}\|\theta\|_2.
	\end{align}
\end{lemma}
\proof{Proof.}
	By Cauchy-Schwarz inequality, we have
	\begin{align*}
		\|(I_d-M)\theta\|_2\|\theta\|_2\geq \theta^\top \Phi^\top D(I-P)\Phi\theta
		\overset{(i)}\geq (\monotonecon)\|\Phi \theta\|_D^2\overset{(ii)}= (\monotonecon)\|\theta\|_2^2,
	\end{align*}
	where step (i) follows from Lemma~\ref{monotone_constant}, step (ii) follows from $\Phi^\top D \Phi = I_d$. As a result, we obtain that $\|\theta\|_2\leq \frac{1}{\monotonecon}\|(I_d-M)\theta\|_2$. 
    \tli{Meanwhile, by Lemma~\ref{monotone_constant}, it is clear that $\theta^\top (I_d - M) \theta = (\theta \Phi)^\top D(I-P)\Phi \theta > 0$ for any $\theta \neq 0$. Therefore, $I_d-M$ is invertible.}
\endproof

\vgap

The following lemma characterizes the Lipschitz continuity of the operator $g$ \tli{defined in \eqref{deterministic_opt}}.
\begin{lemma}\label{lemma_lipschitz}
	For any $\theta, \theta' \in \bbr^d$, $\rho\in\bbr$, we have
	\begin{align*}
		\|g(\theta,\rho) - g(\theta', \rho)\|_2 \leq 2 \|\Psi \theta - \Psi \theta'\|_D \leq 2 \|\theta -\theta'\|_2.
	\end{align*}
\end{lemma}
\proof{Proof.}
	Assume $s\sim \nu, a \sim \pi(\cdot|s), s' \sim \mathsf{P}(\cdot|s,a), a' \sim \pi(\cdot|s')$. Let $\Upsilon:= \langle \psi(s,a), \theta - \theta' \rangle$ and $\Upsilon':= \langle \psi(s',a'), \theta - \theta' \rangle$. We have
	\begin{align*}
		\|g(\theta,\rho) - g(\theta', \rho)\|_2 &= \|\Psi^\top D(I-P)\Psi (\theta-\theta')\|_2 = \|\bbe[\psi(s,a)(\Upsilon-\Upsilon')]\|_2 \\&\overset{(i)}\leq \bbe[|\Upsilon-\Upsilon'|] \leq \sqrt{\bbe[\Upsilon^2]}+\sqrt{\bbe[(\Upsilon')^2]}\overset{(ii)}= 2 \|\Psi \theta - \Psi \theta'\|_D\overset{(iii)}\leq 2 \|\theta - \theta'\|_2,
	\end{align*}
	where step (i) follows from Young's inequality and $\|\psi(s,a)\|_2\leq 1$, (ii) follows from the fact that $\Upsilon$ and $\Upsilon'$ have the same distribution and $\bbe[\Upsilon^2] = \|\Psi \theta - \Psi \theta'\|_D^2$, and step (iii) follows from Cauchy-Schwarz inequality and $\|\psi(s,a)\|_2\leq 1$.
 \endproof

Next we establish a few supporting lemmas that characterize the properties of the stochastic operator defined \tli{in \eqref{stochastic_opt}.} 
Lemma~\ref{assump_variance} and~\ref{assump_variance_2} upper bound the variance of the stochastic operator \tli{and will be used in proving Theorem~\ref{thm_VRTD}}. 
\begin{lemma} \label{assump_variance}
	For every $\theta, \theta' \in \bbr^d$, $\rho, \rho' \in \bbr$, and $\xi:=\{(s,a),(s',a'),c^\pi(s,a))$ where $s\sim \nu^\pi, a\sim \pi(\cdot|s),  s'\sim \mathsf{P}(\cdot|s,a)$, we have 
	\begin{align}\label{variance_3}
		\bbe\|\widetilde g(\theta,\rho, \xi) - \widetilde g(\theta', \rho', \xi)-\big(g(\theta, \rho)  - g(\theta', \rho')\big)\|^2_2 \leq 8\|\Psi \theta - \Psi\theta'\|_D^2 + 2|\rho-\rho'|^2.
	\end{align}
\end{lemma}
\proof{Proof.}
	First, it is easy to see that given $s \sim \nu^\pi$, we have $\bbe[\widetilde g(\theta,\rho, \xi) - \widetilde g(\theta', \rho', \xi)] = g(\theta, \rho)  - g(\theta', \rho')$. 
	For notation brevity, we let $\Upsilon:= \langle \psi(s,a), \theta - \theta' \rangle$ and $\Upsilon':= \langle \psi(s',a'), \theta - \theta' \rangle$. Since $s\in \nu^\pi$, $\Upsilon$ and $\Upsilon'$ have the same distribution, we can write
	\begin{align*}
		\bbe\|&\widetilde g(\theta,\rho, \xi) - \widetilde g(\theta', \rho', \xi)-\big(g(\theta, \rho)  - g(\theta', \rho')\big)\|^2_2\\ &\leq  \bbe[\|\widetilde g(\theta,\rho, \xi) - \widetilde g(\theta', \rho', \xi)\|_2^2] = \bbe[\|(\Upsilon - \Upsilon' + \rho - \rho')\psi(s,a) \|_2^2]\\
		& \overset{(i)}\leq 2\bbe[\|(\Upsilon - \Upsilon')\psi(s,a)\|_2^2] + 2\bbe[\|(\rho-\rho')\psi(s,a) \|_2^2]\\
		&  \overset{(ii)}\leq 4\bbe[\Upsilon^2] + 4\bbe[(\Upsilon')^2]+ 2 (\rho-\rho')^2 \overset{(iii)} \leq 8\|\Psi \theta - \Psi\theta'\|_D^2 + 2 (\rho-\rho')^2,
	\end{align*}
	where step (i) follows from Young's inequality, step (ii) follows from the fact that $\|\psi(s,a)\|_2\leq 1$ and Young's inequality, and step (iii) follows from the fact that $\Upsilon$ and $\Upsilon'$ have the same distribution and $\bbe[\Upsilon^2] = \|\Psi (\theta-\theta')\|_D^2$.
\endproof
\begin{lemma}
[see Lemma 12 of \cite{li2021accelerated}] 
\label{assump_variance_2}
	For $t\in \mathbb{Z}_+$, $\theta, \theta' \in \bbr^d$, $\rho, \rho' \in \bbr$, and $\tau \in \mathbb{Z}_+$ such that
	\begin{align}\label{cond_tau_1}
		\skipcon\cdot \mixrate^\tau \leq \min_{s\in \calS} \nu(s),
	\end{align}
	we have with probability 1,
	\begin{align}\label{variance_4}
		\bbe[\|\widetilde g(\theta,\rho, \xi_{t+\tau}) - \widetilde g(\theta', \rho', \xi_{t+\tau})-\big(g(\theta, \rho)  - g(\theta', \rho')\big)\|^2_2|\calF_t] \leq 16\|\Psi \theta - \Psi\theta'\|_D^2 + 4|\rho-\rho'|^2.
	\end{align}
\end{lemma}

Recall the definition of $\widehat g$, $N_k$, $N'_k$ and $\tau'$ in Algorithm~\ref{alg:VRTD} and $\Cmax:=\max\{\skipcon, \mixcon, \mixcons\}$.
\begin{lemma}\label{var_1}
	For every $\theta, \theta'\in \bbr^d$, $\rho, \rho' \in \bbr$, $\tau'\in \mathbb{Z}_+$ satisfied Ineq.~\eqref{cond_tau_1} and
	$\mixrate^{\tau'} \leq \tfrac{1}{4\Cmax +1},$ then
	\begin{align}
		\bbe\|\widehat g(\theta,\rho) - \widehat g(\theta',\rho')- g(\theta,\rho) + g(\theta',\rho')\|^2_2 \leq \tfrac{20}{N_k}\|\Psi \theta-\Psi \theta'\|_\Pi^2+\tfrac{8}{N_k}|\rho-\rho'|^2.
	\end{align}
\end{lemma}
\proof{Proof.}
We write
\begin{align*}
    \bbe\|\widehat g(\theta,\rho) - \widehat g(\theta',\rho')- g(\theta,\rho) + g(\theta',\rho')\|^2_2 = \tfrac{1}{N_k^2}\tsum_{i=1}^{N_k}\bbe\|\Lambda_i\|_2^2+ \tfrac{2}{N_k^2}\tsum_{i=1}^{N_k-1}\tsum_{j>i}\bbe\langle \Lambda_i,\Lambda_j \rangle,
\end{align*}
where $\Lambda_i = \widetilde g(\theta,\rho,\xi_i^k(\tau')) - \widetilde g(\theta',\rho',\xi_i^k(\tau'))- g(\theta,\rho) + g(\theta',\rho')$. Invoking Lemma~\ref{assump_variance_2} we can bound $\tfrac{1}{N_k^2}\tsum_{i=1}^{N_k}\bbe\|\Lambda_i\|_2^2\leq \tfrac{16}{N_k}\|\Psi \theta-\Psi \theta'\|_\Pi^2+\tfrac{4}{N_k}|\rho-\rho'|^2$. On the other hand, let $\calF_i:=\sigma(\xi_1^k(\tau'), ..., \xi_i^k(\tau'))$, we have
\begin{align*}
    \tfrac{2}{N_k^2}\tsum_{i=1}^{N_k-1}\tsum_{j>i}\bbe\langle \Lambda_i,\Lambda_j \rangle  &= \tfrac{2}{N_k^2}\tsum_{i=1}^{N_k-1}\tsum_{j>i} \bbe[\langle \Lambda_i, \bbe[\Lambda_j|\calF_i] \rangle]\\
    &\overset{(i)}\leq \tfrac{2}{N_k^2}\tsum_{i=1}^{N_k-1}\tsum_{j>i} \bbe[\Cmax \rho^{(j-i)\tau'}\|\Lambda_i\|_2(\|\Psi\theta-\Psi\theta'|_D+|\rho-\rho'|)]\\
    &\overset{(ii)}\leq \tfrac{2}{N_k^2}\tsum_{i=1}^{N_k-1}\tsum_{j>i} \bbe[\Cmax \rho^{(j-i)\tau'}(\tfrac{1}{8}\|\Lambda_i\|^2 + 4\|\Psi\theta-\Psi\theta'|_D^2+4|\rho-\rho'|^2)]\\
    &\overset{(iii)}\leq \tfrac{16}{N_k^2}\tsum_{i=1}^{N_k-1}\tsum_{j>i} \Cmax \rho^{(j-i)\tau'}\big(\|\Psi\theta-\Psi\theta'\|_D^2+|\rho-\rho'|^2\big)\\
    &\leq \tfrac{16\Cmax \rho^{\tau'}}{N_k(1-\rho^{\tau'})}\big(\|\Psi\theta-\Psi\theta'\|_D^2+|\rho-\rho'|^2\big) \overset{(iv)}\leq \tfrac{4}{N_k}\big(\|\Psi\theta-\Psi\theta'\|_D^2+|\rho-\rho'|^2\big),
\end{align*}
where (i) follows from Cauchy-Schwarz inequality and Lemma~\ref{lemma_operator_bias_2}, (ii) follows from Young's inequality, (iii) follows from Lemma~\ref{assump_variance_2}, (iv) follows from 	$\mixrate^{\tau'} \leq \tfrac{1}{4\Cmax +1}$. Combining the upper bounds yields the result.
\endproof

\begin{lemma}\label{var_2}
	For $\widetilde \rho$ defined in Algorithm~\ref{alg:VRTD}, if $\tau'$ satisfies $\mixrate^{\tau'} \leq \frac{1}{4 \Cmax+1}$, 
	then with probability~1,
	\begin{align}
		\bbe[(\widetilde \rho - \rhosol)^2] \leq \tfrac{5\cbound^2}{N_k'}. 
	\end{align}
\end{lemma}
\proof{Proof.}
	Recall the definition of $(\widetilde s_i,\widetilde a_i)$ in Algorithm~\ref{alg:VRTD}. 
	Define the vectors $\mathcal D, \mathcal D_j^{i} \in \bbr^{|\calS|\times |\calA|}$ as $\mathcal{D}:=\left[\nu^\pi(s)\cdot \pi(a|s)\right]$ and  $\mathcal D_j^{i}:=[\mathsf{P}(\widetilde s_j=s, \widetilde a_j=a|\widetilde s_i, \widetilde a_i)]$ where $s\in\calS, a\in\calA$. We have
	\begin{align*}
		\bbe[(\widetilde\rho - \rhosol)^2] 
  &= \tfrac{1}{N_k'^2}\tsum_{i=1}^{N_k'}\bbe|c(\widetilde s_i,\widetilde a_i) - \rho^*|^2 + \tfrac{2}{N_k'^2}\tsum_{i=1}^{N_k'-1}\tsum_{j=i+1}^{N_k'}\bbe\left[\big(c(\widetilde s_i, \widetilde a_i)-\rhosol\big)\bbe[c(\widetilde s_j,\widetilde a_j)-\rhosol|\calF_i]\right]\\
  &\overset{(i)}= \tfrac{1}{N_k'^2}\tsum_{i=1}^{N_k'}\bbe|c(\widetilde s_i,\widetilde a_i) - \rho^*|^2 + \tfrac{2}{N_k'^2}\tsum_{i=1}^{N_k'-1}\tsum_{j=i+1}^{N_k'}\bbe\left[\big(c(\widetilde s_i, \widetilde a_i)-\rhosol\big)\cdot (\textbf{c}^\top (\mathcal{D}_j^i- \mathcal{D}))\right]\\
		&\overset{(ii)}\leq\tfrac{4\cbound^2}{N_k'}+ \tfrac{2}{N_k'^2}\tsum_{i=1}^{N_k'-1}\tsum_{j=i+1}^{N_k'}\cdot\bbe\left[2 \bar c^2\|\mathcal{D}_j^i- \mathcal{D} \|_1\right]\\
		&\overset{(iii)}\leq \tfrac{4\cbound^2}{N_k'} + \tfrac{2}{N_k'^2}\tsum_{i=1}^{N_k'-1}\tsum_{j=i+1}^{N_k'} 2\skipcon\cdot \mixrate^{\tau'(j-i)}\cbound^2 \leq \tfrac{4\cbound^2}{N_k'}+ \tfrac{4\skipcon \cbound^2 \mixrate^{\tau'}}{N_k'(1-\mixrate^{\tau'})} \leq\tfrac{5\cbound^2}{N_k'},
	\end{align*}
	where step (i) follows from the definition of $\mathcal D, \mathcal D_j^{i}$, step (ii) follows from Cauchy-Schwarz inequality and $\|\mathbf{c}\|_\infty \leq \bar c$, and step (iii) follows from Lemma~\ref{lemma_assump_rho}.
\endproof

\subsection{Proof of Theorem~\ref{thm_VRTD}}\label{proof_thm_VRTD}
We first focus on the progress of a single epoch $k\in[K]$. 
Let $\underline\theta$ satisfy $g(\underline \theta, \widetilde \rho) - g(\widetilde \theta, \widetilde \rho) + \widehat g(\widetilde \theta, \widetilde \rho) = 0$ where $\widetilde \theta$ and $\widetilde \rho$ are defined in Algorithm~\ref{alg:VRTD}. The vector $\underline \theta$ can be interpreted as the fixed point that the VRTD algorithm is trying to solve inside the given epoch. The following lemma establish an upper bound on $\|\Psi \underline \theta - \Psi \thetasol\|_D^2$.
\begin{lemma}\label{bound_tilde_v}
	Consider a single epoch with index $k \in [K]$. We have 
	\begin{align}\label{VRTD_step_00}
		\bbe[\|\Psi\underline \theta - \Psi \theta^* \|^2_D]\leq& \tfrac{9}{N_k} \trace\left((I_d-M)^{-1}\bar  \Sigma(I_d-M)^{-\top}\right) + \tfrac{3}{2N_k(\monotonecon)^2\mu}\bbe\|\Psi\theta^* - \bar  Q^\pi\|_D^2 \nn\\
		&\quad+ \tfrac{60 }{N_k(\monotonecon)^2\mu}\bbe\|\Psi\widetilde \theta - \Psi \theta^*\|_D^2+ \tfrac{15\cbound^2}{N_k'(\monotonecon)^2\mu} + \tfrac{120\cbound^2}{N_k N_k'(\monotonecon)^2\mu},
	\end{align}
 \tli{where $\bar  \Sigma$ is defined in Theorem~\ref{thm_VRTD}, and $\mu$ and $M$ are defined in \eqref{def_mu_M}.}
\end{lemma}
Proof of the lemma is postponed to Section~\ref{proof_bound_tilde_v}. Given the upper bound on $\bbe[\|\Psi\underline \theta - \Psi \theta^* \|^2_D]$, we derive the progress in a single epoch $k\in[K]$, which is characterized in the following proposition.
\begin{proposition}\label{prop_VRTD}
	Consider a single epoch with index $k \in [K]$. Suppose that the parameters $\eta$, $N_k$ and $T$ satisfy 
	\beq\label{stepsize_1_1}
	\eta\leq \tfrac{1-\gamma}{845},~~T\geq \tfrac{64}{\mu(1-\gamma)\eta},~~\text{and}~~N_k\geq\tfrac{1160}{\mu(1-\gamma)^2}.
	\eeq
	Set the output of this epoch to be $\widehat \theta_k:=\frac{\sum_{t=1}^T  \theta_t}{T }$. If $\tau$ and $\tau'$ are chosen to satisfy 
	\begin{align}\label{tau_tau_p}
		\mixrate^{\tau}\leq \tfrac{\sqrt{\mu}\eta}{\mixcon}\quad \text{and} \quad \mixrate^{\tau'} \leq \min\{\tfrac{\min_{s\in \calS} \nu(s)}{\skipcon}, \tfrac{1}{4\Cmax+1}\},
	\end{align}
	then we have
	\begin{align}\label{VRTD_prop}
		\bbe[\|\Psi\widehat \theta_k - \Psi \theta^*\|_D^2] &\leq \tfrac{\bbe[\|\Psi \widetilde \theta - \Psi\theta^*\|_D^2]}{2}+ \tfrac{22\trace\{(I_d-M)^{-1}\bar{\Sigma} (I_d-M)^{-\top}\}}{N_k}\nn\\
		&\qquad + \tfrac{4\bbe[\|( {\Qsol} - \Psi\theta^*) \|_D^2])}{\mu(\monotonecon)^2N_k}+ \tfrac{36\cbound^2}{N_k'(\monotonecon)^2\mu}+ \tfrac{285\cbound^2}{N_kN_k'(\monotonecon)^2\mu}.
	\end{align}
\end{proposition}
\noindent See Section~\ref{proof_prop_VRTD} for the proof of this proposition. Now we take Proposition~\ref{prop_VRTD} as given and finish the proof of Theorem~\ref{thm_VRTD}.
First we suppose $\bar Q^* \neq \Psi \theta^*$. Recall the definition $W_1:= 22\trace\{(I_d-M)^{-1}\bar  \Sigma(I_d-M)^{-\top}\} +  \tfrac{4}{\mu(\monotonecon)^2}\bbe[\|(\Qsol - \Psi\theta^*) \|_D^2]$ in Theorem~\ref{thm_VRTD}. Recursively using Ineq.~\eqref{VRTD_prop} yields
\begin{align}\label{step_5_mark}
	\bbe\|\Psi \widehat \theta_K - \Psi \theta^*\|_D^2 &\leq \tfrac{1}{2^K}\|\Psi \theta^0 - \Psi \theta^*\|^2_D  + \tsum_{k=1}^K \tfrac{1}{2^{K-k}}\big(\tfrac{W_1}{N_k}+ \tfrac{36\cbound^2}{N_k'(\monotonecon)^2\mu}+ \tfrac{285\cbound^2}{N_kN_k'(\monotonecon)^2\mu} \big)\nn\\
	&\overset{(i)}\leq \tfrac{1}{2^K}\|\Psi \theta^0 - \Psi \theta^*\|^2_D  + \tsum_{k=1}^K(\tfrac{2}{3})^{K-k}\big(\tfrac{W_1}{N}+\tfrac{36\cbound^2}{N'(\monotonecon)^2\mu}+\tfrac{\cbound^2}{4N'}\big)\nn\\
	&\leq \tfrac{1}{2^K}\|\Psi \theta^0 - \Psi \theta^*\|^2_D  + \tfrac{3W_1}{N}+\tfrac{108\cbound^2}{N'(\monotonecon)^2\mu}+\tfrac{3\cbound^2}{4N'},
\end{align}
as desired, where step~(i) follows from $N_k \geq (\tfrac{3}{4})^{K-k}N$, $N_k' \geq (\tfrac{3}{4})^{K-k}N'$ and $N_k \geq \tfrac{1160}{\mu(1-\gamma)^2}$.

\subsubsection{Proof of Lemma~\ref{bound_tilde_v}}\label{proof_bound_tilde_v}For notational simplicity, we define $\gdiff(\theta, \rho) := \widehat g(\theta,\rho) - g(\theta,\rho)$ for $\theta\in \bbr^d, \rho\in \bbr$.
\tli{Recall the definition of $g$ in \eqref{deterministic_opt} and $g(\theta^*, \rho^*) = 0$, we have
\begin{align}\label{fixed_point_01}
\Psi^\top D \Psi \theta^* = \psi^\top D (c - \rho^* \ones) + \Psi^\top D P \Psi \theta^*.
\end{align}
Recall that 
$\underline\theta$ satisfies $g(\underline \theta, \widetilde \rho) - g(\widetilde \theta, \widetilde \rho) + \widehat g(\widetilde \theta, \widetilde \rho) = 0$, we have
\begin{align}\label{fixed_point_02}
\Psi^\top D \Psi\underline \theta &=\Psi^\top D(c - \widetilde \rho \ones)  + \Psi^\top D P \Psi \underline\theta + g(\widetilde \theta, \widetilde \rho) - \widehat g(\widetilde \theta,\widetilde \rho)\nn\\
&=\Psi^\top D(c - \widetilde \rho \ones)  + \Psi^\top D P \Psi \underline\theta -\widehat \delta_g(\widetilde \theta, \widetilde \rho).
\end{align}
By subtracting Eq.~\eqref{fixed_point_02} on both sides of Eq.~\eqref{fixed_point_01}, we obtain
\begin{align*}
	\Psi^\top D \Psi(\theta^* - \underline \theta) &= \Psi^\top D (\cost - \rhosol \mathbf{1}) + \Psi^\top D P \Psi \theta^* - \big(\Psi^\top D (\cost-\widetilde \rho \mathbf{1})+ \Psi^\top D P \Psi \underline \theta-\gdiff(\widetilde \theta, \widetilde \rho) \big)\\
	&= \Psi^\top D P \Psi(\theta^*- \underline \theta) + \Psi^\top D \mathbf{1}(\widetilde \rho - \rhosol) +  \gdiff(\widetilde \theta, \widetilde \rho).
\end{align*}
}
\tli{Rearrange the terms in the above relation,
we obtain
\[
\Psi^T D (I - P) \Psi (\theta^* - \underline \theta)= \Psi^\top D \mathbf{1}(\widetilde \rho - \rhosol) +  \gdiff(\widetilde \theta, \widetilde \rho)
\]
Using the fact that $\Phi = \Psi B^{-\frac{1}{2}}$ and the definition $M:=\Phi^T D P \Phi$, we have
\begin{align*}
\Psi^T D (I - P) \Psi 
&= B^{1/2} \Phi^\top D (I-P) \Phi B^{1/2}
= B^{1/2} (I -M) B^{1/2}
\end{align*}
Invoking that $I_d-M$ is invertible (proved in Lemma~\ref{lemma_inverse_operator}), we have that}
\begin{align}\label{bound_0}
	B^{\frac{1}{2}}(\theta^*-\underline \theta) 
	=(I_d-M)^{-1}B^{-\frac{1}{2}}\big(\Psi^\top D \mathbf{1}(\widetilde \rho - \rhosol) + \gdiff(\widetilde \theta, \widetilde \rho)\big).
\end{align}
Therefore, further utilizing the relationship $\Phi^\top D \Phi = I_d$ and $\Phi = \Psi B^{-\frac{1}{2}}$ yields
\begin{align}\label{bound_0_1}
	\bbe\|\Psi \theta^* - \Psi  \underline \theta \|_D^2 
 = \bbe \|B^{\frac{1}{2}}(\theta^*-\underline \theta)\|_2^2
 =\bbe\|(I_d-M)^{-1}B^{-\frac{1}{2}}\big(\Psi^\top D \mathbf{1}(\widetilde \rho - \rhosol) +\gdiff(\widetilde \theta, \widetilde \rho)\big)\|_2^2.
\end{align}
By applying Young's inequality we obtain that
\begin{align}\label{bound_inter_term_0}
	\bbe\|\Psi \theta^* - \Psi  \underline \theta \|_D^2 &\leq 3\bbe|\widetilde \rho - \rhosol|^2\cdot\|(I_d-M)^{-1}B^{-\frac{1}{2}}\Psi^\top D \mathbf{1}\|_2^2 \nn\\
	&\quad + 3\bbe\|(I_d-M)^{-1}B^{-\frac{1}{2}}\big(\gdiff(\widetilde \theta, \widetilde \rho)-\gdiff(\theta^*,\rhosol)\big)\|_2^2 + 3\bbe\|(I_d-M)^{-1}B^{-\frac{1}{2}}\gdiff(\theta^*,\rhosol)\|_2^2\nn\\
	& \overset{(i)}\leq \tfrac{15\cdot\cbound^2}{N_k'(\monotonecon)^2\mu} + 3\bbe\|(I_d-M)^{-1}B^{-\frac{1}{2}}\big(\gdiff(\widetilde \theta, \widetilde \rho)-\gdiff(\theta^*,\rhosol)\big)\|_2^2\nn\\
	&\quad+ 3\bbe\|(I_d-M)^{-1}B^{-\frac{1}{2}}\gdiff(\theta^*,\rhosol)\|_2^2 \nn\\
 & \overset{(ii)}\leq \tfrac{15\cdot\cbound^2}{N_k'(\monotonecon)^2\mu} + \tfrac{3}{(1-\beta)^2\mu}\bbe\|\gdiff(\widetilde \theta, \widetilde \rho)-\gdiff(\theta^*,\rhosol)\|_2^2 + 3\bbe\|(I_d-M)^{-1}B^{-\frac{1}{2}}\gdiff(\theta^*,\rhosol)\|_2^2 \nn\\
	& \overset{(iii)}\leq \tfrac{15 \cbound^2}{N_k'(\monotonecon)^2\mu}+\tfrac{24 \bbe|\widetilde\rho - \rho^*|^2}{N_k(\monotonecon)^2\mu}+ \tfrac{60\bbe\|\Psi \theta^* - \Psi\widetilde\theta\|_D^2}{N_k(\monotonecon)^2\mu}  + 3\bbe\|(I_d-M)^{-1}B^{-\frac{1}{2}}\gdiff(\theta^*,\rhosol)\|_2^2\nn\\
 & \overset{(iv)}\leq \tfrac{15 \cbound^2}{N_k'(\monotonecon)^2\mu}+\tfrac{120 \cbound^2}{N_kN_k'(\monotonecon)^2\mu}+ \tfrac{60\bbe\|\Psi \theta^* - \Psi\widetilde\theta\|_D^2}{N_k(\monotonecon)^2\mu}  + 3\bbe\|(I_d-M)^{-1}B^{-\frac{1}{2}}\gdiff(\theta^*,\rhosol)\|_2^2.
\end{align}
where $(i)$ follows from using Lemma \ref{var_2} to bound $\bbe|\widetilde \rho - \rhosol|^2$, and using Jensen's inequality $\|\Psi^\top D \mathbf{1}\|_2^2\leq \sum_{s\in\calS, a \in \calA} \nu(s)\pi(a|s) \|\psi(s,a)\|_2^2 \leq 1$, Lemma~\ref{lemma_inverse_operator}, and $\|B^{-\frac{1}{2}}\|_2 \leq \mu^{-\frac{1}{2}}$ to bound $\|(I_d-M)^{-1}B^{-\frac{1}{2}}\Psi^\top D \mathbf{1}\|_2^2$; step (ii) follows from Lemma~\ref{lemma_inverse_operator}, and $\|B^{-\frac{1}{2}}\|_2 \leq \mu^{-\frac{1}{2}}$;  step (iii) follows from Lemma \ref{var_1};  step (iv) follows from Lemma \ref{var_2}. 

\textcolor{black}{Next, we work on an upper bound on $\bbe\|(I_d-M)^{-1}B^{-\frac{1}{2}}\gdiff(\theta^*,\rhosol)\|_2^2$. We write
\begin{align}\label{step_first_term}
		&\quad\bbe\|(I_d-M)^{-1} B^{-\frac{1}{2}}\gdiff(\theta^*,\rhosol)\|_2^2\nn\\
		& =\tfrac{1}{N_k^2}\tsum_{i=1}^{N_k}\bbe\|(I_d-M)^{-1} B^{-\frac{1}{2}}\Delta_i^{\tau'}(\theta^*, \rho^*)\|_2^2
		+ R_1\nn\\
		&\overset{(i)} \leq \tfrac{2}{N_k}\trace\big((I_d-M)^{-1}\bar \Sigma(I_d-M)^{-\top}\big) + R_1,
	\end{align}
	where $\Delta_i^{\tau'}(\theta^*, \rho^*) := \widetilde g(\theta^*,\rho^*, \xi_i^k(\tau')) - g(\theta^*, \rho^*)$ with  $\widetilde g(\theta^*,\rho^*, \xi_i^k(\tau'))$ defined in Algorithm~\ref{alg:VRTD}, and
	$$
	R_1 := \tfrac{2}{N_k^2}\tsum_{i=1}^{N_k}\tsum_{j>i} \bbe[\langle (I_d-M)^{-1} B^{-\frac{1}{2}}\Delta_i^{\tau'}(\theta^*, \rho^*),(I_d-M)^{-1} B^{-\frac{1}{2}}\Delta_j^{\tau'}(\theta^*, \rho^*) \rangle].
	$$
	Step (i) of Ineq. \eqref{step_first_term} follows from condition $\skipcon\cdot \mixrate^\tau \leq \min_{s\in \calS} \nu(s)$ and the definition of $\bar \Sigma$ in  Theorem~\ref{thm_VRTD}, i.e.,
	\begin{align}\label{bound_D_0}
		&\quad \bbe\|(I_d-M)^{-1} B^{-\frac{1}{2}}\Delta_i^{\tau'}(\theta^*, \rho^*)\|_2^2\nn\\ 
		&= \tsum_{s\in \calS, a \in \calA} \mathbb{P}(s_i=s, a_i =a)\cdot \bbe[\|(I_d-M)^{-1}B^{-\frac{1}{2}}\Delta_i^{\tau'}(\theta^*, \rho^*)\|_2^2|s_i=s, a_i = a]\nn\\
		&\leq \tsum_{s\in \calS, a \in \calA} \nu(s) \pi(a|s)\cdot  \bbe[\|(I_d-M)^{-1}B^{-\frac{1}{2}}\Delta_i^{\tau'}(\theta^*, \rho^*)\|_2^2|s_i=s, a_i = a]\nn\\
		&\quad + \tsum_{s\in \calS, a \in \calA}\|\mathbb{P}(s_i=\cdot)- \nu \ \|_\infty \pi(a|s)\cdot \bbe\|(I_d-M)^{-1}B^{-\frac{1}{2}}\Delta_i^{\tau'}(\theta^*, \rho^*)\|_2^2\nn\\
		& \leq2 \cdot \trace\big((I_d-M)^{-1}\bar \Sigma(I_d-M)^{-\top}\big).
	\end{align}
	It suffices to bound the term $R_1$. Define $\calF^i := \sigma(\xi_1^k(\tau'), ..., \xi_i^k(\tau'))$. We write
	\begin{align}\label{bound_D}
		R_1
		& = \tfrac{2}{N_k^2}\tsum_{i=1}^{N_k}\tsum_{j>i}\bbe[\langle (I_d-M)^{-1} B^{-\frac{1}{2}}\Delta_i^{\tau'}(\theta^*, \rho^*),(I_d-M)^{-1} B^{-\frac{1}{2}}\bbe[\Delta_j^{\tau'}(\theta^*, \rho^*)|\mathcal{F}^i]\rangle]\nn\\
		& \overset{(i)}\leq \tfrac{2}{N_k^2}\tsum_{i=1}^{N_k}\tsum_{j>i}\bbe[\|(I_d-M)^{-1} B^{-\frac{1}{2}}\Delta_i^{\tau'}(\theta^*, \rho^*)\|_2\|(I_d-M)^{-1} B^{-\frac{1}{2}}\bbe[\Delta_j^{\tau'}(\theta^*, \rho^*)|\mathcal{F}^i]\|_2]\nn\\
		&\overset{(ii)}\leq\tfrac{2}{N_k^2}\tsum_{i=1}^{N_k}\tsum_{j>i} \mixcon\rho^{(j-i)\tau'}\bbe[\|(I_d-M)^{-1} B^{-\frac{1}{2}}\Delta_i^{\tau'}(\theta^*, \rho^*)\|_2\cdot\tfrac{1}{(1-\gamma)\sqrt{\mu}}\|\Psi \theta^* - \bar Q^*\|_D ]\nn\\
		&\overset{(iii)}\leq \tfrac{1}{N_k^2}\tsum_{i=1}^{N_k}\tsum_{j>i}\mixcon\rho^{(j-i)\tau'}
		\big(\bbe[\|(I_d-M)^{-1} B^{-\frac{1}{2}}\Delta_i^{\tau'}(\theta^*, \rho^*)\|_2^2] + \tfrac{1}{\mu (1-\gamma)^2}\|\Psi \theta^* - \bar Q^*\|_D^2 \big)\nn\\
		&\overset{(iv)}\leq \tfrac{1}{N_k^2}\tsum_{i=1}^{N_k}\tsum_{j>i}\mixcon\rho^{(j-i)\tau'}
		\left(2 \trace\big\{(I_d-M)^{-1}\bar \Sigma(I_d-M)^{-\top}\big\}+\tfrac{1}{\mu (1-\gamma)^2} \|\Psi \theta^* - \bar Q^*\|_D^2 \right),
	\end{align}
	where step (i) follows from Cauchy-Schwarz inequlity, step~(ii) follows from Lemma~\ref{lemma_inverse_operator}, $\|B^{\frac{1}{2}}\theta\|_2 \leq \tfrac{1}{\sqrt{\mu}}\|\theta\|_2$ and Lemma~\ref{lemma_operator_bias_1}, step (iii) follows from Young's inequality, and step~(iv) follows from Ineq.~\eqref{bound_D_0}. Substituting Ineq.~\eqref{bound_D} into Ineq.~\eqref{step_first_term} yields
	\begin{align*}
	&\bbe\|(I_d-M)^{-1}B^{-\tfrac{1}{2}}\gdiff(\theta^*,\rhosol)\|_2^2 \nn\\
	&\leq \big(\tfrac{2}{N_k}+\tfrac{2\mixcon \mixrate^{\tau'}}{(1-\mixrate^{\tau'})N_k}\big)\trace\big\{(I_d-M)^{-1}\bar  \Sigma(I_d-M)^{-\top}\big\} +  \tfrac{\mixcon \mixrate^{\tau'}}{(1-\mixrate^{\tau'})\mu(\monotonecon)^2N_k}\|\Psi \thetasol - \Qsol\|_D^2.
\end{align*}
}
Invoking the condition that $\mixrate^{\tau'} \leq \frac{1}{2\mixcon +1}$, we obtain
\begin{align}\label{bound_last_term}
	\bbe\|(I_d-M)^{-1}B^{-\tfrac{1}{2}}\gdiff(\theta^*,\rhosol)\|_2^2\leq \tfrac{3}{N_k}\trace\big\{(I_d-M)^{-1}\bar  \Sigma(I_d-M)^{-\top}\big\} +  \tfrac{1}{2\mu(\monotonecon)^2N_k}\|\Psi \thetasol - \Qsol\|_D^2.
\end{align}
Finally, putting together Ineq.~\eqref{bound_inter_term_0} and Ineq.~\eqref{bound_last_term} yields the desired result.

\subsubsection{Proof of Proposition~\ref{prop_VRTD}}\label{proof_prop_VRTD}
We first provide an upper bound on the term  $\bbe\|\Psi\widehat \theta_k - \Psi\underline \theta\|_{D}^2$. For simplicity, we define $\gdifft(\theta, \rho,\xi_t) := \widetilde g(\theta,\rho, \xi_t) - g(\theta,\rho)$ for $\theta\in\bbr^d$ and $\rho \in\bbr$. By the updates of the VRTD algorithm, we have
\begin{align*}
	\theta_{t+1}-\underline \theta&=\theta_{t} - \eta \big(\widetilde g(\theta_t,\rhoest,\xi_t)-\widetilde g(\widetilde \theta,\rhoest, \xi_t) + \widehat g(\widetilde \theta, \rhoest)\big)-\underline \theta\\
 &= \theta_{t} - \underline \theta
 -\eta\big(g(\theta_t, \widetilde \rho) 
 - g(\widetilde \theta, \widetilde \rho)
 + \widehat g(\widetilde \theta, \rhoest)\big)
 -\eta \big(\gdifft(\widetilde\theta,\widetilde \rho,\xi_t)-\gdifft(\theta_t,\widetilde \rho,\xi_t)\big).
\end{align*}
Invoking the fact that $  \widehat g(\widetilde \theta, \widetilde \rho) = g(\widetilde \theta, \widetilde \rho)-g(\underline \theta, \widetilde \rho) $, we obtain
\begin{align}\label{VRTD_step_0}
	\theta_{t+1}-\underline \theta = \theta_t - \underline \theta  - \eta\big(g(\theta_t, \rhoest) - g(\underline \theta, \rhoest)\big) + \eta \big(\gdifft(\widetilde\theta,\widetilde \rho,\xi_t)-\gdifft(\theta_t,\widetilde \rho,\xi_t)\big).
\end{align}
By taking $\|\cdot\|_2^2$ and expectation on both sides of the above equation, we have
\begin{align}\label{VRTD_1}
	&\bbe \|\theta_{t+1}-\underline \theta\|_2^2 \nn\\
	&\leq \bbe\|\theta_t - \underline \theta\|_2^2-2\eta\langle g(\theta_t, \rhoest) - g(\underline \theta, \rhoest), \theta_t-\underline \theta\rangle + 2\eta\bbe \langle \gdifft(\widetilde \theta,\widetilde \rho,\xi_t)-\gdifft( \theta_t,\widetilde \rho,\xi_t), \theta_t - \underline \theta  \rangle  \nn\\
	&\quad+2\eta^2\bbe\|g(\theta_t, \rhoest) - g(\underline \theta, \rhoest)\|_2^2 +2\eta^2\bbe\|\gdifft(\widetilde \theta,\widetilde \rho,\xi_t)-\gdifft( \theta_t,\widetilde \rho,\xi_t)\|_2^2\nn\\
	& \overset{(ii)}\leq \bbe\|\theta_t - \underline \theta\|_2^2 + \big( 8\eta^2-2\eta(\monotonecon) \big)\bbe\|\Psi \theta_t- \Psi\underline \theta\|_D^2 + 32\eta^2\bbe\|\Psi \widetilde \theta- \Psi \theta_t\|_D^2\nn\\
	&\quad + 2\eta \bbe\langle \gdifft(\widetilde \theta,\widetilde \rho,\xi_t)-\gdifft( \theta_t,\widetilde \rho,\xi_t), \theta_t - \underline \theta  \rangle\nn \\
	& \overset{(iii)}\leq \bbe\|\theta_t - \underline \theta\|_2^2 + \big( 8\eta^2-2\eta(\monotonecon) + \eta \mixcon \mixrate^{m_0}/\sqrt{\mu}\big)\bbe\|\Psi \theta_t- \Psi\underline \theta\|_D^2\nn\\
	&\quad+ (32\eta^2+ \eta \mixcon \mixrate^{m_0}/\sqrt{\mu})\bbe\|\Psi \widetilde \theta- \Psi \theta_t\|_D^2,
\end{align}
where 
step (ii) follows from using Lemma~\ref{monotone_constant} to bound the second term, Lemma~\ref{lemma_lipschitz} to bound the forth term and Lemma~\ref{assump_variance_2} to bound the last term. 
Step (iii) follows from 
	\begin{align*}
		&\langle \bbe[\gdifft(\widetilde \theta,\widetilde \rho,\xi_t)-\gdifft( \theta_t,\widetilde \rho,\xi_t)|\calF_t], \theta_t - \underline \theta  \rangle\\
		&\qquad \leq \|\bbe[\gdifft(\widetilde \theta,\widetilde \rho,\xi_t)-\gdifft( \theta_t,\widetilde \rho,\xi_t)|\calF_t]\|_2 \|\theta_t - \underline \theta\|_2\leq \mixcon\mixrate^{m_0} \|\Psi \widetilde \theta- \Psi \theta_t\|_D \cdot \tfrac{1}{\sqrt{\mu}} \|\Psi \theta_t - \Psi \underline \theta\|_D\\
		&\qquad\leq \tfrac{\mixcon\mixrate^{m_0} }{2\sqrt{\mu}}\big(\|\Psi \widetilde \theta -\Psi \theta_t\|_D^2 + \|\Psi \theta_t - \Psi \underline \theta\|_D^2\big),
	\end{align*}
	where the first inequality follows from Cauchy-Schwarz inequality, the second inequality follows from Lemma~\ref{lemma_operator_bias_2} and $\|\theta\|_2\leq (1/\sqrt{\mu})\|\Psi \theta\|_D$, and the third inequality follows from Young's inequality.
Further using Young's inequality $\|\Psi \widetilde \theta- \Psi \theta_t\|_D^2\leq 2\|\Psi \theta_t- \Psi\underline \theta\|_D^2+ 2\|\Psi\widetilde \theta- \Psi\underline \theta\|_D^2$ to Ineq.~\eqref{VRTD_1} yields
\begin{align}
	\bbe \|\theta_{t+1}-\underline \theta\|_2^2 
	&\leq \bbe\|\theta_t - \underline \theta\|_2^2+\big( 8\eta^2 + 64\eta^2-2\eta(\monotonecon)+ 3\eta \mixcon \mixrate^{m_0}/\sqrt{\mu}\big)\bbe\|\Psi \theta_t- \Psi\underline \theta\|_D^2 \nn\\
	&\quad+(64\eta^2+ 2\eta \mixcon \mixrate^{m_0}/\sqrt{\mu})\bbe\|\Psi \underline \theta - \Psi\widetilde \theta\|_D^2\nn\\
	&\overset{(i)}\leq \bbe\|\theta_t - \underline \theta\|_2^2 -\eta(\monotonecon)\bbe\|\Psi \theta_t- \Psi\underline \theta\|_D^2 + 66\eta^2\bbe\|\Psi \underline \theta - \Psi\widetilde \theta\|_D^2,  
\end{align}
where step (i) follows from the assumption $\eta \leq \tfrac{\monotonecon}{75}$ and $\mixrate^{m_0}\leq \frac{\sqrt{\mu}\eta}{\mixcon}$. By taking a telescopic sum of Ineq.~\eqref{VRTD_1} from $t=1$ to $T$, we obtain
\begin{align*}
	\tsum_{t=1}^T \eta(\monotonecon)\bbe[\|\Psi \theta_t - \Psi\underline \theta\|_D^2] + \bbe[\|\theta_{T+1}-\underline \theta\|_2^2]  \leq \bbe[\|\theta_1-\underline \theta\|_2^2]+ 66T\eta^2\bbe[\|\Psi \underline \theta - \Psi \widetilde \theta\|_D^2].
\end{align*}
Utilizing the inequality 
$\|\theta\|_2^2 \leq \tfrac{1}{\mu}\|B^{\frac{1}{2}}\theta\|_2^2 = \tfrac{1}{\mu}\|\Psi \theta\|_D^2$, we then have
\begin{align*}
	\tsum_{t=1}^T \eta(\monotonecon)\bbe[\|\Psi \theta_t &- \Psi\underline \theta\|_D^2] 
	\leq \tfrac{1}{\mu}\bbe[\|\Psi \theta_1- \Psi\underline \theta\|_D^2]+ 66T\eta^2\bbe[\|\Psi \underline \theta - \Psi \widetilde \theta\|_D^2].
\end{align*}
Noting in addition that $\theta_1=\widetilde \theta$ then yields
\begin{align}\label{VRTD_step_2}
	\tsum_{t=1}^T \bbe[\|\Psi \theta_t - \Psi\underline \theta\|_D^2]/T &\leq \tfrac{1}{T \eta (\monotonecon) }( \tfrac{1}{\mu}+ 66T\eta^2)\bbe[\|\Psi \underline \theta - \Psi \widetilde \theta\|_D^2]\nn\\
	&\leq \tfrac{2/\mu+132T\eta^2}{T \eta (\monotonecon) }(\bbe[\|\Psi \widetilde \theta - \Psi\theta^*\|_D^2] + \bbe[\|\Psi \underline \theta - \Psi\theta^*\|_D^2]).
\end{align}
Combining the bound on $\|\Psi\underline \theta - \Psi \theta^*\|_D^2$ shown in Lemma~\ref{bound_tilde_v} and the bound on $\tfrac{\bbe[\|\Psi \theta_t - \Psi\underline \theta\|_D^2]}{T }$ from Ineq.~\eqref{VRTD_step_2}, we obtain 
\begin{align}\label{VRTD_step_3}
	\tsum_{t=1}^{T} \tfrac{\bbe[\|\Psi \theta_t - \Psi 	\thetasol\|_D^2]}{T}&\leq  \tsum_{t=1}^T 2\bbe[\|\Psi \theta_t - \Psi\underline \theta\|_D^2]/T + 2 \bbe[\|\Psi \theta^* - \Psi \underline \theta\|_D^2]\nn\\
	& \leq \tfrac{4/\mu+264T\eta^2}{T \eta (\monotonecon) }\bbe[\|\Psi \widetilde \theta - \Psi\theta^*\|_D^2]+ \big(\tfrac{4/\mu+264T\eta^2}{T \eta (\monotonecon) }+2\big)\bbe[\|\Psi \underline \theta - \Psi\theta^*\|_D^2]\nn\\
	& \leq \tfrac{4/\mu+264T\eta^2}{T \eta (\monotonecon) }\bbe[\|\Psi \widetilde \theta - \Psi\theta^*\|_D^2] + \big(\tfrac{4/\mu+264T\eta^2}{T \eta (\monotonecon) }+2\big)\big(\tfrac{120\bbe\|(\Psi \widetilde \theta - \Psi\theta^*) \|_D^2+3\bbe\|(\Qsol - \Psi\theta^*) \|_D^2}{2\mu(\monotonecon)^2N_k} \nn\\
	&\quad\quad\qquad+ \tfrac{15\cbound^2}{N_k'(\monotonecon)^2\mu}+ \tfrac{120\cbound^2}{N_kN_k'(\monotonecon)^2\mu}
	+\tfrac{9}{N_k}\trace\{(I_d-M)^{-1}\bar  \Sigma(I_d-M)^{-\top}\}\big).
\end{align}
Then the final term takes the desired form; we are left with bounding the first two terms in the display above.
First, recall by assumption that $\eta \leq \tfrac{\monotonecon}{845}$, and $T\geq\tfrac{64}{\mu(\monotonecon)\eta}$,  so that
\begin{align}\label{cond_1}
	\tfrac{4/\mu+264T\eta^2}{T\eta (\monotonecon) }\leq \tfrac{4/\mu + 64\cdot 2 64\eta/\mu(\monotonecon)}{64/\mu}\leq \tfrac{3}{8}.
\end{align}
Second, utilizing the condition $N_k\geq \tfrac{1160}{\mu(\monotonecon)^2}$, we have 
$
	\big(\tfrac{4/\mu+264T\eta^2}{T\eta (\monotonecon) }+2\big)\tfrac{60}{\mu(\monotonecon)^2N_k} \leq \tfrac{1}{8}.
$
Substituting the bounds above into inequality \eqnok{VRTD_step_3}, we obtain
\begin{align}\label{cond_3}
	\tfrac{\sum_{t=1}^T\bbe\|\Psi \theta_t - \Psi \thetasol\|_D^2}{T} \leq \tfrac{\bbe[\|\Psi \widetilde \theta - \Psi\theta^*\|_D^2]}{2} + \tfrac{22\trace\{(I_d-M)^{-1}\bar  \Sigma(I_d-M)^{-\top}\}}{N_k} + \tfrac{4\bbe[\|(\bar  Q^\pi - \Psi\theta^*) \|_D^2])}{\mu(\monotonecon)^2N_k}+ \tfrac{36\cbound^2}{N_k'(\monotonecon)^2\mu}+ \tfrac{\cbound^2}{4N_k'}.
\end{align}
By the definition of $\widehat \theta_k:=\frac{\sum_{t=1}^T\theta_t}{T}$ and Jensen's inequality, we have 
$
	\bbe[\|\Psi\widehat \theta_k - \Psi\theta^*\|_D^2] \leq \tfrac{\sum_{t=1}^T\bbe\|\Psi \theta_t - \Psi \thetasol\|_D^2}{T  },
$
which completes the proof of Proposition~\ref{prop_VRTD}.

\subsection{Proof of Theorem~\ref{thm_VRTD_bias}}\label{proof_thm_VRTD_bias}
Similar to the proof of Theorem~\ref{thm_VRTD}, we focus on a single epoch. For simplicity, we let $\bar  \theta_t$ denote $\bbe[\theta_t]$, let $\bar {\underline \theta}$ denote $\bbe[\underline \theta]$. 
We first establish a lemma to bound the term $\|\Psi\theta^* - \Psi \bar {\underline \theta}\|_D^2$.
\begin{lemma}\label{bias_lemma_0}
	We have
	\begin{align*}
		\|\Psi \theta^* - \Psi \bar {\underline \theta}\|_D^2 \leq \tfrac{3 \mixrate^{2\tau'}}{\mu(\monotonecon)^2 (1-\mixrate^{\tau'})^2} \left( \tfrac{\skipcon^2\cdot \cbound^2}{N'^2} + \tfrac{2\mixcon^2\bbe[\|\Psi \widetilde\theta - \Psi \theta^*\|_D^2]+\mixcon^2\|\Psi \theta^* -\Qsol\|_D^2 + 8\mixcons^2\cdot\cbound^2}{N^2}\right).
	\end{align*}
\end{lemma}
\noindent The proof of this lemma is postponed to Section~\ref{proof_bias_lemma_0}.

Now we switch our attention to bound the iterates. Recall the definition $\gdifft(\theta, \rho,\xi_t) := \widetilde g(\theta,\rho, \xi_t) - g(\theta,\rho)$.
By taking expectation on both sides of Eq. \eqref{VRTD_step_0} and using the linearity of $g$, we obtain
\begin{align*}
	\bar  \theta_{t+1} - \bar {\underline \theta} = \bar  \theta_t - \bar {\underline \theta} - \eta\big(g(\bar  \theta_t, \bbe[\rhoest]) - g(\bar {\underline \theta}, \bbe[\rhoest]) \big) + \eta \bbe\big[\gdifft(\widetilde \theta,\widetilde \rho, \xi_t)- \gdiff(\theta_t,\widetilde \rho,\xi_t)\big].
\end{align*}
Taking squared $\ell_2$-norm on both sides of the above equality yields
\begin{align}\label{bias_step_0}
	\|\bar  \theta_{t+1} - \bar {\underline \theta}\|_2^2 & = \|\bar  \theta_{t} - \bar {\underline \theta}\|_2^2 - 2\eta\langle g(\bar  \theta_t, \bbe[\rhoest]) - g(\bar {\underline \theta}, \bbe[\rhoest]) ,  \bar  \theta_{t} - \bar {\underline \theta} \rangle + 2\eta \big\langle \bbe\big[\gdifft(\widetilde \theta,\widetilde \rho, \xi_t)- \gdiff(\theta_t,\widetilde \rho,\xi_t)\big] ,  \bar  \theta_{t} - \bar {\underline \theta}\big\rangle\nn\\
	&\quad + \eta^2 \big\|\big(g(\bar  \theta_t, \bbe[\rhoest]) - g(\bar {\underline \theta}, \bbe[\rhoest]) \big) - \bbe\big[\gdifft(\widetilde \theta,\widetilde \rho, \xi_t)- \gdiff(\theta_t,\widetilde \rho,\xi_t)\big]\big\|_2^2\nn\\
	&\leq \|\bar  \theta_{t} - \bar {\underline \theta}\|_2^2 - 2\eta\langle g(\bar  \theta_t, \bbe[\rhoest]) - g(\bar {\underline \theta}, \bbe[\rhoest]) ,  \bar  \theta_{t} - \bar {\underline \theta} \rangle + 2\eta \big\langle \bbe\big[\gdifft(\widetilde \theta,\widetilde \rho, \xi_t)- \gdiff(\theta_t,\widetilde \rho,\xi_t)\big] ,  \bar  \theta_{t} - \bar {\underline \theta}\big\rangle\nn\\
	&\quad + 2\eta^2 \left\|g(\bar  \theta_t, \bbe[\rhoest]) - g(\bar {\underline \theta}, \bbe[\rhoest]) \right\|_2^2 +2\eta^2\big\| \bbe\big[\gdifft(\widetilde \theta,\widetilde \rho, \xi_t)- \gdiff(\theta_t,\widetilde \rho,\xi_t)\big]\big\|_2^2,
\end{align}
where the last inequality follows from Young's inequality. By Jensen's inequality, we have that 
\begin{align}\label{bias_step_1}
	\big\| \bbe\big[\gdifft(\widetilde \theta,\widetilde \rho, \xi_t)- \gdiff(\theta_t,\widetilde \rho,\xi_t)\big]\big\|_2^2&\leq \bbe\big\|\bbe[\gdifft(\widetilde \theta,\widetilde \rho, \xi_t)- \gdiff(\theta_t,\widetilde \rho,\xi_t)|\mathcal{F}_{t-1}]\big\|_2^2\leq \mixcon^2 \mixrate^{2\tau} \bbe[\|\Psi \theta_t - \Psi \widetilde \theta\|_D^2].
\end{align}
Using similar idea, we obtain another upper bound
\begin{align}\label{bias_step_2}
	&\big\langle \bbe\big[\gdifft(\widetilde \theta,\widetilde \rho, \xi_t)- \gdiff(\theta_t,\widetilde \rho,\xi_t)\big] ,  \bar  \theta_{t} - \bar {\underline \theta}\big\rangle\nn\\
	&\qquad\overset{(i)}\leq \bbe\big\|\bbe[\gdifft(\widetilde \theta,\widetilde \rho, \xi_t)- \gdiff(\theta_t,\widetilde \rho,\xi_t)|\calF_{t-1}]\big\|_2\|\bar  \theta_t - \underline{\bar  \theta}\|_2\nn\\
	& \qquad\overset{(ii)}\leq \tfrac{\mixcon \mixrate^{\tau}}{\sqrt{\mu}} \bbe[\|\Psi \theta_t - \Psi \widetilde \theta\|_D]\|\Psi \bar  \theta_t - \Psi\underline{\bar  \theta}\|_D \overset{(iii)}\leq \tfrac{\mixcon \mixrate^{\tau}}{2\sqrt{\mu}}\big( \bbe[\|\Psi \theta_t - \Psi \widetilde \theta\|_D^2] + \|\Psi \bar  \theta_t - \Psi\underline{\bar  \theta}\|_D^2\big),
\end{align}
where step (i) follows from Cauchy-Schwarz inequality and Jensen's inequality, step (ii) follows from $\|\theta\|_2\leq \frac{1}{\sqrt{\mu}}\|\Psi \theta\|_D$ and Lemma~\ref{lemma_operator_bias_1}, step (iii) follows from Jensen's inequality and Young's inequality. 
Combining Ineqs. \eqref{bias_step_0}, \eqref{bias_step_1}, \eqref{bias_step_2} we obtain
\begin{align*}
	&\|\bar  \theta_{t+1} - \bar {\underline \theta}\|_2^2\\ & \leq \|\bar  \theta_{t} - \bar {\underline \theta}\|_2^2 - 2\eta\langle g(\bar  \theta_t, \bbe[\rhoest]) - g(\bar {\underline \theta}, \bbe[\rhoest]) ,  \bar  \theta_{t} - \bar {\underline \theta} \rangle+ \tfrac{\eta\mixcon \mixrate^\tau}{\sqrt{\mu}} \big(\bbe[\|\Psi \theta_t - \Psi \widetilde \theta\|_D^2] + \| \Psi\bar  \theta_t - \Psi\bar {\underline \theta}\|_2^2)\\
	&\quad + 2\eta^2 \left\|g(\bar  \theta_t, \bbe[\rhoest]) - g(\bar {\underline \theta}, \bbe[\rhoest]) \right\|_2^2 +2\eta^2\mixcon^2\mixrate^{2\tau} \bbe[\|\Psi \theta_t - \Psi \widetilde \theta\|_D^2]\\
	&\overset{(i)}\leq \|\bar  \theta_{t} - \bar {\underline \theta}\|_2^2  + \big(8\eta^2 - 2\eta(\monotonecon) + \tfrac{\eta \mixcon \mixrate^\tau}{\sqrt{\mu}} \big)\|\Psi \bar  \theta_t - \Psi \bar {\underline \theta}\|_D^2 + (\tfrac{\eta \mixcon \mixrate^\tau}{\sqrt{\mu}} + 2\eta^2\mixcon^2\mixrate^{2\tau} )\bbe[\|\Psi \theta_t - \Psi \widetilde \theta\|_D^2]\\
	&\overset{(ii)}\leq \|\bar  \theta_{t} - \bar {\underline \theta}\|_2^2   - \eta(\monotonecon) \|\Psi \bar  \theta_t - \Psi \bar {\underline \theta}\|_D^2 + (\tfrac{2\eta \mixcon \mixrate^\tau}{\sqrt{\mu}} + 4\eta^2\mixcon^2\mixrate^{2\tau} )\big(\bbe[\|\Psi \theta_t - \Psi  \theta^*\|_D^2]+\bbe[\|\Psi \theta^* - \Psi \widetilde \theta\|_D^2]\big).
\end{align*}
where step (i) follows from Lemma~\ref{monotone_constant} and  \ref{lemma_lipschitz}, step (ii) follows from Young's inequality, $\eta \leq \tfrac{\monotonecon}{9}$ and $\mixrate^{\tau}\leq \frac{\sqrt{\mu}\eta}{\mixcon}$. 
By taking a telescope sum of the above inequality from $t=1$ to $T$, we have 
\begin{align*}
	\tsum_{t=1}^T &\eta(\monotonecon)\|\Psi \bar  \theta_t - \Psi \bar {\underline \theta}\|_D^2 \\ &\leq \|\bar  \theta_1 - \bar {\underline \theta}\|_2^2
	+ (\tfrac{2\eta \mixcon \mixrate^\tau}{\sqrt{\mu}} + 4\eta^2\mixcon^2\mixrate^{2\tau} )\big(T\bbe[\|\Psi \theta^* - \Psi \widetilde \theta\|_D^2] + \tsum_{t=1}^T\bbe[\|\Psi \theta_t - \Psi  \theta^*\|_D^2]\big).
\end{align*}
By definition, we have $\widehat \theta_k = \frac{\sum_{t=1}^T\theta_t }{T }$; noting in addition that $\theta_1=\widetilde \theta$, we have
\begin{align}
	&T \eta(\monotonecon) \|\Psi\bbe[\widehat \theta_k]-\Psi \bar {\underline \theta}\|_D^2 \nn\\
	&\leq \tfrac{1}{\mu} \|\Psi\bbe[\widetilde \theta] - \Psi\bar {\underline\theta}\|_D^2+ (\tfrac{2\eta \mixcon \mixrate^\tau}{\sqrt{\mu}} + 4\eta^2\mixcon^2\mixrate^{2\tau} )\big(T\bbe[\|\Psi \theta^* - \Psi \widetilde \theta\|_D^2] + \tsum_{t=1}^T\bbe[\|\Psi \theta_t - \Psi  \theta^*\|_D^2]\big)\nn\\
	&\leq \tfrac{2}{\mu} \big(\|\Psi\bbe[\widetilde \theta] - \Psi\theta^*\|_D^2+\|\Psi \theta^* - \Psi\bar {\underline\theta}\|_D^2\big)\nn\\
	&\quad\quad\qquad+ (\tfrac{2\eta \mixcon \mixrate^\tau}{\sqrt{\mu}} + 4\eta^2\mixcon^2\mixrate^{2\tau} )\big(T\bbe[\|\Psi \theta^* - \Psi \widetilde \theta\|_D^2] + \tsum_{t=1}^T\bbe[\|\Psi \theta_t - \Psi  \theta^*\|_D^2]\big).
\end{align}
Recalling the upper bound of $\tsum_{t=1}^T\bbe[\|\Psi \theta_t - \Psi  \theta^*\|_D^2]/T$ in Ineq.~\eqref{cond_3} and Young's inequality $\|\Psi\bbe[\widehat \theta_k]-\Psi \theta^*\|_D^2\leq 2\|\Psi\bbe[\widehat \theta_k]-\Psi \bar {\underline \theta}\|_D^2 + 2\|\Psi\theta^*-\Psi \bar {\underline \theta}\|_D^2$, we have 
\begin{align*}
	\|\Psi\bbe[\widehat \theta_k]-\Psi \theta^*\|_D^2 &\leq \tfrac{4}{\mu T \eta(\monotonecon)}\|\Psi\bbe[\widetilde \theta] - \Psi\theta^*\|_D^2 +(2+\tfrac{4}{\mu T \eta(\monotonecon)}) \|\Psi \theta^* - \Psi\bar {\underline\theta}\|_D^2\nn\\
	&\quad+ \tfrac{4\eta \mixcon \mixrate^\tau/\sqrt{\mu} + 8\eta^2\mixcon^2\mixrate^{2\tau} }{ \eta(\monotonecon)}\big(\tfrac{3}{2}\bbe[\|\Psi \theta^* - \Psi \widetilde \theta\|_D^2]+\tfrac{W_1}{N_k}+ \tfrac{36\cbound^2}{N_k'(\monotonecon)^2\mu}+ \tfrac{\cbound^2}{4N_k'}\big),
\end{align*}
where $W_1:=22\trace\{(I_d-M)^{-1}\bar  \Sigma(I_d-M)^{-\top}\} +  \tfrac{4}{\mu(\monotonecon)^2}\bbe[\|(\Qsol - \Psi\theta^*) \|_D^2]$. Noticing that if the conditions \eqref{stepsize_1} and \eqref{stepsize_thm_2} are satisfied, we have that all assumptions of Theorem~\ref{thm_VRTD} are satisfied. Therefore, by utilizing Theorem~\ref{thm_VRTD} and $\widetilde \theta = \widehat \theta_{k-1}$, $N_k = N$ and  $N_k' = N'$, we obtain
\begin{align*}
	\|\Psi\bbe[\widehat \theta_k]-\Psi \theta^*\|_D^2\leq &\tfrac{4}{\mu T \eta(\monotonecon)}\|\Psi\bbe[\widetilde \theta] - \Psi\theta^*\|_D^2 +(2+\tfrac{4}{\mu T \eta(\monotonecon)}) \|\Psi \theta^* - \Psi\bar{\underline\theta}\|_D^2\nn\\
	&\quad+ \tfrac{4\eta \mixcon \mixrate^\tau/\sqrt{\mu} + 8\eta^2\mixcon^2\mixrate^{2\tau} }{ \eta(\monotonecon)}\big(\tfrac{3}{2^k}\|\Psi \theta^0 - \Psi \thetasol\|_D^2+\tfrac{11W_1}{2N}+ \tfrac{198\cbound^2}{N'(\monotonecon)^2\mu} +\tfrac{11\cbound^2}{8N'}\big),
\end{align*}
Finally, invoking the bound of $\|\Psi \theta^* - \Psi\bar {\underline\theta}\|_D^2$ in Lemma~\ref{bias_lemma_0} and the fact that $\tfrac{4}{\mu T \eta(\monotonecon)}\leq \tfrac{1}{2}$, we obtain
\begin{align*}
	\|\Psi&\bbe[\widehat \theta_k]-\Psi \theta^*\|_D^2\\ &\leq \tfrac{1}{2}\|\Psi\bbe[\widetilde \theta] - \Psi\theta^*\|_D^2 +\tfrac{15 \mixrate^{2\tau'}}{2\mu(\monotonecon)^2 (1-\mixrate^{\tau'})^2} \big( \tfrac{\skipcon^2\cdot \cbound^2}{N'^2} + \tfrac{2\mixcon^2\bbe[\|\Psi \widetilde\theta - \Psi \theta^*\|_D^2]+\mixcon^2\|\Psi \theta^* -\bar  Q^\pi\|_D^2 + 8\mixcons^2\cdot\cbound^2}{N^2}\big)\nn\\
	&\quad+ \tfrac{4\eta \mixcon \mixrate^\tau/\sqrt{\mu} + 8\eta^2\mixcon^2\mixrate^{2\tau} }{ \eta(\monotonecon)}\big(\tfrac{3}{2^k}\|\Psi \theta^0 - \Psi \thetasol\|_D^2+\tfrac{11W_1}{2N}+ \tfrac{198\cbound^2}{N'(\monotonecon)^2\mu} +\tfrac{11\cbound^2}{8N'}\big).
\end{align*}
Invoking that $N\geq \tfrac{1}{\mu(\monotonecon)^2}$ and $N'\geq \tfrac{1}{\mu(\monotonecon)^2}$, we have with constants $b_1$ and $b_2$,
\begin{align*}
	\|\Psi&\bbe[\widehat \theta_k]-\Psi \theta^*\|_D^2 \leq\tfrac{1}{2}\|\Psi\bbe[\widehat\theta_{k-1}] - \Psi\theta^*\|_D^2 + \big( b_1 \Cmax^2 \mixrate^{2\tau'} + \tfrac{b_2\Cmax\mixrate^\tau}{(\monotonecon)\sqrt{\mu}}\big) \left(\|\Psi \theta^0 - \Psi \thetasol\|_D^2 +  \tfrac{W_1}{N} + \cbound^2\right).
\end{align*}
By recursively using this inequality for epochs, we achieve the desired result.

\subsubsection{Proof of Lemma~\ref{bias_lemma_0}}\label{proof_bias_lemma_0}
\tli{First, by using Assumption~\ref{lemma_assump_rho}, we have
\begin{align}\label{bound_rho_bias_lemma21}
\bbe[\widetilde \rho] - \rho^* \leq \tfrac{\bar c}{N'} \tsum_{i=1}^{N'} \skipcon \mixrate^{i\cdot \tau'} \leq \tfrac{\bar c \skipcon \mixrate^{\tau'}}{N'(1-\mixrate^{\tau'})}.
\end{align}
Similarly, by using Lemma \ref{lemma_operator_bias_1} and Lemma~\ref{lemma_operator_bias_2}, we have
\begin{align}
\|\bbe[\widehat g(\theta^*, \rhosol) -  g(\theta^*, \rhosol)]\|_2 \leq \tfrac{1}{N'} \tsum_{i=1}^{N'} \mixcon \mixrate^{i\cdot \tau'}\|\Psi\theta^* - \bar Q^\pi\|_D\leq \tfrac{\mixcon \mixrate^{\tau'}}{N'(1-\mixrate^{\tau'})}\|\Psi\theta^* - \bar Q^\pi\|_D, \label{bias_at_optimal}
\end{align}
and
\begin{align}
\|\bbe[\widehat g(\widetilde\theta, \rhoest) - \widehat g(\theta^*, \rhosol) -  g(\widetilde\theta, \rhoest) + g(\theta^*, \rhosol)]\|_2 &\leq \tfrac{1}{N'} \tsum_{i=1}^{N'}[\mixcon \mixrate^{i\cdot \tau'}\bbe\|\Psi\widetilde\theta - \Psi\theta^*\|_D + \mixcons \mixrate^{i\cdot \tau'} \bbe|\rhoest - \rhosol|]\nn\\
& \leq \tfrac{\mixrate^{\tau'}}{N'(1-\mixrate^{\tau'})}[\mixcon \bbe\|\Psi\widetilde\theta - \Psi\theta^*\|_D + 4 \mixcons \cdot \bar c^2]
\label{bias_of_diff}
\end{align}
}
Next, we start deriving the bound for $\|\Psi \theta^* - \Psi \bar {\underline \theta}\|_D^2$. By taking expectation on both sides of Eq. \eqref{bound_0}, we obtain
\begin{align}\label{bound_0_bias}
	B^{\frac{1}{2}}(\theta^*-\bar {\underline \theta}) =(I_d-M)^{-1}B^{-\frac{1}{2}}\big(\Psi^\top D \mathbf{1}(\bbe[\rhoest] - \rhosol) + \bbe[\widehat g(\widetilde\theta, \rhoest) - g(\widetilde\theta, \rhoest)]\big).
\end{align}
Invoking the fact that $\Phi^\top D \Phi = I_d$ and $\Phi^\top = B^{-\frac{1}{2}}\Psi^\top$ yields
\begin{align*}
	\|\Psi \theta^* - \Psi \bar {\underline \theta}\|_D^2 &= \|(I_d-M)^{-1}B^{-\frac{1}{2}}\big(\Psi^\top D \mathbf{1}(\bbe[\rhoest] - \rhosol) + \bbe[\widehat g(\widetilde\theta, \rhoest) - g(\widetilde\theta, \rhoest)]\big)\|_2^2\\
	&\overset{(i)}\leq  3(\bbe[\rhoest] - \rhosol)^2\cdot\|(I_d-M)^{-1}B^{-\frac{1}{2}}\Psi^\top D \mathbf{1}\|_2^2+ 3\|(I_d-M)^{-1}B^{-\frac{1}{2}}\bbe[\widehat g(\theta^*, \rhosol) -  g(\theta^*, \rhosol)]\big\|_2^2\\
	&\qquad + 3\|(I_d-M)^{-1}B^{-\frac{1}{2}}\bbe[\widehat g(\widetilde\theta, \rhoest) - \widehat g(\theta^*, \rhosol) -  g(\widetilde\theta, \rhoest) + g(\theta^*, \rhosol)]\big\|_2^2\\
 &\overset{(ii)}\leq \tfrac{3}{(\monotonecon)^2\mu}\cdot (\bbe[\rhoest] - \rhosol)^2 + \tfrac{3}{(\monotonecon)^2\mu}\|\bbe[\widehat g(\theta^*, \rhosol) -  g(\theta^*, \rhosol)]\big\|_2^2\\
	&\qquad+ \tfrac{3}{(\monotonecon)^2\mu}\|\bbe[\widehat g(\widetilde\theta, \rhoest) - \widehat g(\theta^*, \rhosol) -  g(\widetilde\theta, \rhoest) + g(\theta^*, \rhosol)]\big\|_2^2\\
	&\overset{(iii)}\leq \tfrac{3}{(\monotonecon)^2\mu}\cdot\big(\tfrac{\cbound\skipcon\mixrate^{\tau'}}{N' (1-\mixrate^{\tau'})}\big)^2  + \tfrac{3}{(\monotonecon)^2\mu} \big(\tfrac{ \mixrate^{\tau'}}{N (1-\mixrate^{\tau'})} \big)^2 \mixcon^2\|\Psi \theta^* -\bar  Q^\pi\|_D^2\\
	&\qquad+ \tfrac{3}{(\monotonecon)^2\mu}\cdot \big(\tfrac{ \mixrate^{\tau'}}{N (1-\mixrate^{\tau'})} \big)^2\cdot(2\mixcon^2\bbe[\|\Psi \widetilde\theta - \Psi \theta^*\|_D^2]+8\mixcons^2\cdot\cbound^2)
\end{align*}
\tli{where step (i) follows from Young's inequality; step (ii) follows from Lemma~\ref{lemma_inverse_operator}, $\|B^{-\frac{1}{2}}\|_2 \leq \mu^{-\frac{1}{2}}$ and using Jensen's inequality $\|\Psi^\top D \mathbf{1}\|_2^2\leq \sum_{s\in\calS, a \in \calA} \nu(s)\pi(a|s) \|\psi(s,a)\|_2^2 \leq 1$; step (iii) follows from Ineqs.~\eqref{bound_rho_bias_lemma21}, \eqref{bias_at_optimal}, and \eqref{bias_of_diff} with Young's inequality. Then the desired result immediately follows from the above inequality.} 

\subsection{Proof of Lemma \ref{monotone_constant_perturb}}\label{proof_of_lemma_perturb}
\proof{Proof.}
	We first prove Ineq. \eqref{lemma_monotone_before_perturb}. It should be noted that $\|\cdot\|_D$ is a semi-norm when there exists $(s,a)\in \calS\times \calA$ such that $\pi(a|s)=0$. However, after removing the state-action pairs where $\pi(a|s)=0$, we got an irreducible Markov chain with the state-action pairs left.
	Recalling the fact that $\mathcal{I}\cap \mathbb{S}=\phi$, then the proof of Ineq. \eqref{lemma_monotone_before_perturb} follows from the same argument as the proof of Lemma~\ref{monotone_constant}. 
	
	Next we prove Ineq. \eqref{lemma_monotone_perturb}. The basic idea of this proof is to utilize Ineq.~\eqref{lemma_monotone_before_perturb} and the construction of the perturbed policy $\widetilde \pi$ to prove the desired lower bound (Ineq.~\eqref{lemma_monotone_perturb}) in the weighted $\ell_2$-norm $\|\cdot\|_{\widetilde D}$.
	Towards this end, we first define a vector $\widehat v \in \bbr^{|\calS|\times |\calA|}$ as $\widehat v(s,a)=\nu(s)\underline \pi$ if $a \notin \calA_s$ and $\widehat v(s,a)=0$ otherwise. Let $\widehat D := \diag(\widehat v)$. Then it is easy to see that $D + \widehat D \succeq \widetilde D$. Recalling the definition of $\alpha(\underline \pi)$, we have for all $s\in \calS$, $a\in \calA$,
	\begin{align}\label{monotone_perturb_step_1}
		\widetilde \pi(a|s) \geq \alpha(\underline \pi)\pi(a|s)+ \mathbb{I}_{\{a\notin \calA_s\}}(\underline \pi - \alpha(\underline \pi) \pi(a|s)).
	\end{align}
	Now we write
	\begin{align}\label{step_0_monotone}
		x^\top \widetilde D (I-P)x &= \bbe\left[x(s,a)^2-x(s,a)x(s',a')\big|s\sim \nu, a\sim\widetilde \pi(\cdot|s), s'\sim\mathsf{P}(\cdot|s,a),a'\sim \pi(\cdot|s')\right]\nn\\
		&=\tfrac{1}{2}\underbrace{\bbe\big[\big(x(s,a)-x(s',a')\big)^2|s\sim \nu, a\sim\widetilde \pi(\cdot|s), s'\sim\mathsf{P}(\cdot|s,a),a'\sim \pi(\cdot|s')\big]}_{R_1}\nn\\
		&\quad+ \tfrac{1}{2}\|x\|^2_{\widetilde D} - \tfrac{1}{2} \underbrace{\bbe\big[x(s',a')^2\big|s\sim \nu, a\sim\widetilde \pi(\cdot|s), s'\sim\mathsf{P}(\cdot|s,a),a'\sim \pi(\cdot|s')\big]}_{R_2}.
	\end{align}
	Next we lower bound $R_1$ and upper bound $R_2$ separately. Invoking Ineq. \eqref{monotone_perturb_step_1}, we can write
	\begin{align}\label{step_1_monotone}
		R_1 &\geq \alpha(\underline \pi)\bbe\big[\big(x(s,a)-x(s',a')\big)^2|s\sim \nu, a\sim\pi(\cdot|s), s'\sim\mathsf{P}(\cdot|s,a),a'\sim \pi(\cdot|s')\big]\nn\\
		&\quad+ \tsum_{s\in\calS}\nu(s) \tsum_{a \notin \calA_s}  (\underline \pi - \alpha(\underline \pi)\pi(a|s)) \tsum_{s'\in \calS}\mathsf{P}(s'|s,a)\tsum_{a'\in \calA}\pi(a'|s') \big(x(s,a)-x(s',a')\big)^2\nn\\
		&\overset{(i)}\geq 2(\monotonecon)\alpha(\underline \pi)\|x\|_D^2 + \tfrac{\underline \pi}{2}\tsum_{s\in\calS}\nu(s) \tsum_{a \notin \calA_s}    \tsum_{s'\in \calS}\mathsf{P}(s'|s,a)\tsum_{a'\in \calA}\pi(a'|s') \big(x(s,a)-x(s',a')\big)^2\nn\\
		&\overset{(ii)}\geq 2(\monotonecon)\alpha(\underline \pi)\|x\|_D^2 + \tfrac{\underline \pi}{2}\tsum_{s\in\calS}\nu(s) \tsum_{a \notin \calA_s}   \tsum_{s'\in \calS}\mathsf{P}(s'|s,a)\tsum_{a'\in \calA}\pi(a'|s') \big(\tfrac{1}{2}x(s,a)^2-x(s',a')^2\big)\nn\\
		&\overset{(iii)}\geq 2(\monotonecon)\alpha(\underline \pi)\|x\|_D^2 + \tfrac{1}{4}\|x\|_{\widehat D}^2- \tfrac{\underline \pi}{2} \big(|\calA|-\min_{s\in \calS}|\calA_s|\big)\tsum_{s'\in \calS}\tsum_{a'\in \calA}\pi(a'|s') x(s',a')^2\nn\\
		& \overset{(iv)}\geq \tfrac{31(\monotonecon)\alpha(\underline \pi)}{16}\|x\|_D^2+ \tfrac{1}{4}\|x\|_{\widehat D}^2\overset{(iv)}\geq \min\big\{\tfrac{31(\monotonecon)\alpha(\underline \pi)}{16}, \tfrac{1}{4}\big\} \|x\|_{\widetilde D}^2,
	\end{align}
	where step (i) follows from Eq. \eqref{proof_monotone_step_1} and \eqref{lemma_monotone_before_perturb}, step (ii) follows from the fact that $(x-y)^2=x^2-2xy+y^2\geq \tfrac{x^2}{2}-y^2$, step (iii) follows from $\mathsf{P}(s'|s,a)\leq 1$ and $\sum_{s\in\calS}\nu(s) = 1$, and step (iii) follows from that the definition of $\underline \pi$ ensures $\underline \pi\leq\tfrac{\alpha(\underline \pi)(\monotonecon)\min_{s\in\calS}\nu(s)}{8(|\calA|- \min_{s\in\calS}|\calA_s|)}$, and step (iv) follows from $D + \widehat D \succeq \widetilde D$. 
	
	Next we upper bound $R_2$. Write
	\begin{align}\label{per_step_1}
		R_2 
		& = \tsum_{s'\in \calS}\tsum_{a'\in\calA}\pi(a'|s')x(s',a')^2 ~\left(\tsum_{s\in\calS}\tsum_{a\in \calA} \nu(s) \pi(a|s)\sfP(s'|s,a)\right)\nn\\
		&\quad + \tsum_{s'\in \calS}\tsum_{a'\in\calA}\pi(a'|s')x(s',a')^2 ~\left(\tsum_{s\in\calS}\tsum_{a\in \calA} \nu(s)\big(\widetilde \pi(a|s)-\pi(a|s)\big)\sfP(s'|s,a)\right)\nn\\
		&\overset{(i)}= \tsum_{s'\in \calS}\tsum_{a'\in\calA}\nu(s')\pi(a'|s')x(s',a')^2\nn\\ &\quad+\underbrace{ \tsum_{s'\in \calS}\tsum_{a'\in\calA}\pi(a'|s')x(s',a')^2 ~\left(\tsum_{s\in\calS}\tsum_{a\in \calA} \nu(s)\big(\widetilde \pi(a|s)-\pi(a|s)\big)\sfP(s'|s,a)\right)}_{R_3},
	\end{align}
	where step (i) follows from the fact that $\nu$ is the stationary distribution under policy $\pi$. Next we bound $R_3$ on the RHS of the above equality. 
	\begin{align*}
		R_3 &\overset{(i)}\leq  \tsum_{s'\in \calS}\tsum_{a'\in\calA}\pi(a'|s')x(s',a')^2 ~\left(\tsum_{s\in\calS}\tsum_{a\notin \calA_s} \nu(s)\big(\widetilde \pi(a|s)-\pi(a|s)\big)\sfP(s'|s,a)\right)\\
		& \overset{(ii)}\leq \tsum_{s'\in \calS}\tsum_{a'\in\calA}\pi(a'|s')x(s',a')^2 ~\left(\tsum_{s\in\calS}\nu(s) (|\calA| - |\calA_s|)\underline \pi\right)\\
		& \overset{(iii)}\leq \tfrac{\alpha(\underline \pi)(\monotonecon)}{8}\tsum_{s'\in \calS}\tsum_{a'\in\calA}\nu(s')\pi(a'|s')x(s',a')^2,
	\end{align*}
	where step (i) follows from that $\widetilde \pi(a|s)-\pi(a|s)>0$ only when $a\notin \calA_s$, step (ii) follows from $\mathsf{P}(s'|s,a)\leq 1$, and step (iii) follows from $\underline \pi \leq \tfrac{\alpha(\underline \pi)(\monotonecon)\min_{s\in\calS}\nu(s)}{8(|\calA|- \min_{s\in\calS}|\calA_s|)}$. Combining the bound of $R_3$ with Eq. \eqref{per_step_1}, we obtain
	\begin{align} \label{per_step_2}
		R_2&\leq\big(1+\tfrac{\alpha(\underline \pi)(\monotonecon)}{8} \big)\tsum_{s\in \calS}\tsum_{a\in\calA}\nu(s)\pi(a|s)x(s,a)^2\nn\\
		&\overset{(i)} \leq\big(1+\tfrac{\alpha(\underline \pi)(\monotonecon)}{8} \big)\tsum_{s\in \calS}\tsum_{a\in\calA}\nu(s)\widetilde \pi(a|s)x(s,a)^2\nn\\
		&\quad+ \big(1+\tfrac{\alpha(\underline \pi)(\monotonecon)}{8} \big)\tsum_{s\in \calS}\tsum_{a\in\calA_s}\nu(s)\big(\pi(a|s)- \widetilde \pi(a|s)\big)x(s,a)^2\nn\\
		&\overset{(ii)} \leq\big(1+\tfrac{\alpha(\underline \pi)(\monotonecon)}{8} \big)\|x\|_{\widetilde D}^2+ \big(1+\tfrac{\alpha(\underline \pi)(\monotonecon)}{8} \big)\big(\tfrac{1}{\alpha(\underline \pi)}-1\big)\|x\|_{\widetilde D}^2\nn\\
		&\overset{(iii)} \leq\big(1+\tfrac{\alpha(\underline \pi)(\monotonecon)}{4} \big)\|x\|_{\widetilde D}^2.
	\end{align}
	where step (i) follows from $\pi(a|s)-\widetilde \pi(a|s) \geq 0$ only when $a\in \calA_s$, step (ii) follows from Ineq. \eqref{monotone_perturb_step_1}, and step (iii) follows from that the definition of $\underline \pi$ and $\alpha(\underline \pi)$ ensures $\big(1+\tfrac{\alpha(\underline \pi)(\monotonecon)}{8} \big)\big(\tfrac{1}{\alpha(\underline \pi)}-1\big)\leq \tfrac{\alpha(\underline \pi)(\monotonecon)}{8}$. 
	Combining the Ineqs. \eqref{step_0_monotone}, \eqref{step_1_monotone} and \eqref{per_step_2} yields 
	\begin{align*}
		x^\top \widetilde D (I-P)x
		\geq \big(\min\{\tfrac{31(\monotonecon)\alpha(\underline \pi)}{32}, \tfrac{1}{8}\} + \tfrac{1}{2}-\big(\tfrac{1}{2}+\tfrac{\alpha(\underline \pi)(\monotonecon)}{8} \big)\big)\|x\|_{\widetilde D}^2= \min\big\{\tfrac{27(\monotonecon)\alpha(\underline \pi)}{32}, \tfrac{1-\alpha(\underline \pi)(\monotonecon)}{8}\big\}\|x\|_{\widetilde D}^2,
	\end{align*}
	which completes the proof.
\endproof

\section{Extension of Lemma~\ref{monotone_constant_perturb} to DMDPs}\label{extension_to_DMDP}
As mentioned in Section~\ref{sec:insufficient_random}, the EVRTD method naturally extends to the DMDP setting. To illustrate the extension, we proof the analog of Lemma~\ref{monotone_constant_perturb} in the DMDP setting.
\begin{lemma}\label{monotone_constant_perturb_DMDP}
	Let $\gamma\in(0,1)$ denote the discount factor for the underlying discounted Markov decision process. 
	Assume $|\calS|\geq 2$. For $s\in\calS$, if we take 
	\begin{align}\label{underline_pi_DMDP}
		\underline \pi \leq  \tfrac{(1-\gamma)\min_{s\in\calS}\nu(s)}{(|\calA|- \min_{s\in\calS}|\calA_s|)(8+(1-\gamma)\min_{s\in\calS}\nu(s))},  
	\end{align}
	and construct the perturbed policy $\widetilde \pi$ as \eqref{construct_tilde_pi}, then
	\begin{align}\label{lemma_monotone_perturb_DMDP}
		\inf_{x \in \mathbb{S}, \|x\|_{\widetilde D}=1} x^\top \widetilde D (I-\gamma P)x \geq \tfrac{3}{4}(1-\gamma).
	\end{align}
\end{lemma}
\proof{Proof.}
	First, we utilize the Cauchy-Schwarz inequality to obtain
	\begin{align}\label{cs_bound}
		x^\top \widetilde D (I-\gamma P) x = \|x\|_{\widetilde D}^2 -\gamma x^\top \widetilde D P x \geq  \|x\|_{\widetilde D}^2 -\gamma  \|P x\|_{\widetilde D}  \|x\|_{\widetilde D}.
	\end{align}
	Towards this end, the rest of this proof focus on establishing an upper bound on $ \|P x\|_{\widetilde D}$. We write
	\begin{align*}
		\|P x\|_{\widetilde D}^2 \leq \tsum_{s\in\calS} \tsum_{a\in \calA} \nu(s) \widetilde \pi(a|s) \left(\tsum_{s'\in \calS}\tsum_{a'\in \calA} \mathsf{P}(s'|s,a)\pi(a'|s') x(s',a')^2 \right),
	\end{align*}
	which follows from Jensen's inequality. Notice that the RHS of the above inequality is $R_2$ in the proof of Lemma~\ref{monotone_constant_perturb}. Applying the upper bound of $R_2$ (with $\beta$ replaced by $\gamma$) in \eqref{per_step_2} gives us
	\begin{align*}
		\|P x\|_{\widetilde D}^2 \leq\big(1+\tfrac{\alpha(\underline \pi)(1-\gamma)}{4} \big)\|x\|_{\widetilde D}^2,
	\end{align*}
	where $\alpha(\underline \pi) := 1- (|\calA| - \min_{s\in \calS}|\calA_s|)\underline \pi\leq 1$. Therefore, we have $\|P x\|_{\widetilde D} \leq \big(1+\tfrac{\alpha(\underline \pi)(1-\gamma)}{4} \big)\|x\|_{\widetilde D}$. Substituting this inequality into Ineq.~\eqref{cs_bound} yields
	\begin{align*}
		x^\top \widetilde D (I-\gamma P) x \geq \|x\|_{\widetilde D}^2 -\gamma  \big(1+\tfrac{\alpha(\underline \pi)(1-\gamma)}{4} \big)\|x\|_{\widetilde D}^2 = \big(1-\tfrac{\gamma\alpha(\underline \pi)}{4} \big)(1-\gamma)\|x\|_{\widetilde D}^2\geq \tfrac{3}{4}(1-\gamma)\|x\|_{\widetilde D}^2,
	\end{align*}
	which completes the proof.
\endproof

\end{document}